\documentclass[a4paper, twocolumn]{article}
\usepackage{arxiv}

\usepackage{amsmath,amsfonts,bm}
\usepackage{amsthm}

\def\Figref#1{Figure~\ref{#1}}

\def\eqref#1{equation~\ref{#1}}

\def\floor#1{\lfloor #1 \rfloor}
\def\1{\bm{1}}

\def\vm{{\bm{m}}}

\def\vt{{\bm{t}}}
\def\vu{{\bm{u}}}
\def\vv{{\bm{v}}}
\def\vw{{\bm{w}}}
\def\vx{{\bm{x}}}
\def\vy{{\bm{y}}}
\def\vz{{\bm{z}}}

\def\mS{{\bm{S}}}

\def\mX{{\bm{X}}}

\DeclareMathAlphabet{\mathsfit}{\encodingdefault}{\sfdefault}{m}{sl}
\SetMathAlphabet{\mathsfit}{bold}{\encodingdefault}{\sfdefault}{bx}{n}

\def\gG{{\mathcal{G}}}

\def\sC{{\mathbb{C}}}

\def\sN{{\mathbb{N}}}

\def\sZ{{\mathbb{Z}}}

\newcommand{\normltwo}{L^2}

\DeclareMathOperator*{\argmin}{arg\,min}

\newtheorem{proposition}{Proposition}
\newtheorem{definition}{Definition}
\newtheorem{corollary}{Corollary}
\usepackage[utf8]{inputenc} 
\usepackage[T1]{fontenc}    
\usepackage{hyperref}
\usepackage{url}            
\usepackage{booktabs}       
\usepackage{amsfonts}       
\usepackage{nicefrac}       
\usepackage{microtype}      
\usepackage[dvipsnames]{xcolor}

\usepackage{graphicx}
\usepackage{subcaption}

\usepackage{nomencl}
\makenomenclature

\usepackage{amssymb}
\usepackage{pifont}
\newcommand{\cmark}{\ding{51}}%
\newcommand{\xmark}{\ding{55}}%

\usepackage{mathtools}
\usepackage{tikz}

\usepackage{makecell}
\usepackage{multirow}

\usepackage{wrapfig} 

\newcommand{\ours}{T-Phenotype}

\newcommand{\overbar}[1]{\mkern 1.5mu\overline{\mkern-1.5mu#1\mkern-1.5mu}\mkern 1.5mu}

\usepackage{enumitem}
\usepackage{algorithm}
\usepackage{algpseudocode}
\newcommand{\var}[1]{{#1}}

\algrenewcommand\algorithmicrequire{\textbf{Input:}}
\algrenewcommand\algorithmicensure{\textbf{Output:}}

\newcommand{\capstr}[1]{{\texorpdfstring{\uppercase{#1}}{}}}

\RequirePackage[round]{natbib}

\setcitestyle{authoryear,round,citesep={;},aysep={,},yysep={;}}

\title{T-Phenotype: Discovering Phenotypes of \\Predictive Temporal Patterns in Disease Progression}

\author{
    Yuchao Qin\\
	University of Cambridge, UK\\
	\And
    Mihaela van der Schaar \\
    University of Cambridge, UK \\
    The Alan Turing Institute, UK\\
    \And 
    Changhee Lee\\
    Chung-Ang University, South Korea\\
}
\date{}

\renewcommand{\headeright}{}
\renewcommand{\undertitle}{}
\renewcommand{\shorttitle}{T-Phenotype}

\begin{document}

\twocolumn[
\maketitle
]

\begin{abstract}
Clustering time-series data in healthcare is crucial for clinical phenotyping to understand patients’ disease progression patterns and to design treatment guidelines tailored to homogeneous patient subgroups.  While rich temporal dynamics enable the discovery of potential clusters beyond static correlations, two major challenges remain outstanding: 
i) discovery of predictive patterns from many potential temporal correlations in the multi-variate time-series data and ii) association of individual temporal patterns to the target label distribution that best characterizes the underlying clinical progression. 
To address such challenges, we develop a novel temporal clustering method, \textit{T-Phenotype}, to discover phenotypes of predictive temporal patterns from labeled time-series data. 
We introduce an efficient representation learning approach in frequency domain that can encode variable-length, irregularly-sampled time-series into a unified representation space, which is then applied to identify various temporal patterns that potentially contribute to the target label using a new notion of path-based similarity. 
Throughout the experiments on synthetic and real-world datasets, we show that T-Phenotype achieves the best phenotype discovery performance over all the evaluated baselines. We further demonstrate the utility of T-Phenotype by uncovering clinically meaningful patient subgroups characterized by unique temporal patterns.

\end{abstract}

\section{INTRODUCTION}
Discovering predictive patterns of disease progression has been a long pursuit in healthcare. 
Clinicians have considered specific clinical (disease) status and the associated patterns as a \textit{phenotype} to uncover the heterogeneity of diseases and to design therapeutic guidelines tailored to homogeneous subgroups \citep{hripcsak2013phenotype,richesson2016clinical}. 
While rule-based phenotypes identified by domain experts have been widely used \citep{denny2013systematic,richesson2016clinical}, designing and validating such rules require tremendous effort. 
Unfortunately, disease progression can manifest through a broad spectrum of clinical factors, collected as a sequence of measurements in electronic health records (EHRs), that may vary greatly across individual patients. This makes it even more daunting for domain experts to transform such raw and complex clinical observations into clinically relevant and interpretable patterns.

Temporal clustering has been recently used as a data-driven framework for phenotyping to partition patients with sequences of observations into homogeneous subgroups.  To discover different temporal patterns, traditional notions of similarity focus on either adjusting similarity measures \citep{xzhang2019clustering, baytas2017patient} or finding low-dimensional representations \citep{jho2014clustering, agiannoula2018clustering} for longitudinal observations. 
These approaches are purely unsupervised and discard valuable information about the disease status that is often available in the clinical data. More recently, predictive clustering methods \citep{lee2020temporal,lee2020outcome,lee2022developing,aguiar2022learning} have introduced a new notion of similarity such that each cluster shares similar disease status to provide a better prognostic value.  
Despite the effort to understand temporal dynamics in their mutual context, 
these clustering methods fail to capture the full picture of disease progression as reflected by covariate trajectories of prognostic characteristics, i.e., temporal patterns associated with specific disease status.
Figure~\ref{fig:clusters} illustrates a pictorial depiction of the notion of phenotypes behind different temporal clustering methods.

\begin{figure*}[!tb]
	\begin{center}
	 \includegraphics[clip, trim=0cm 0.8cm 0cm 0.8cm, width=0.9\textwidth]{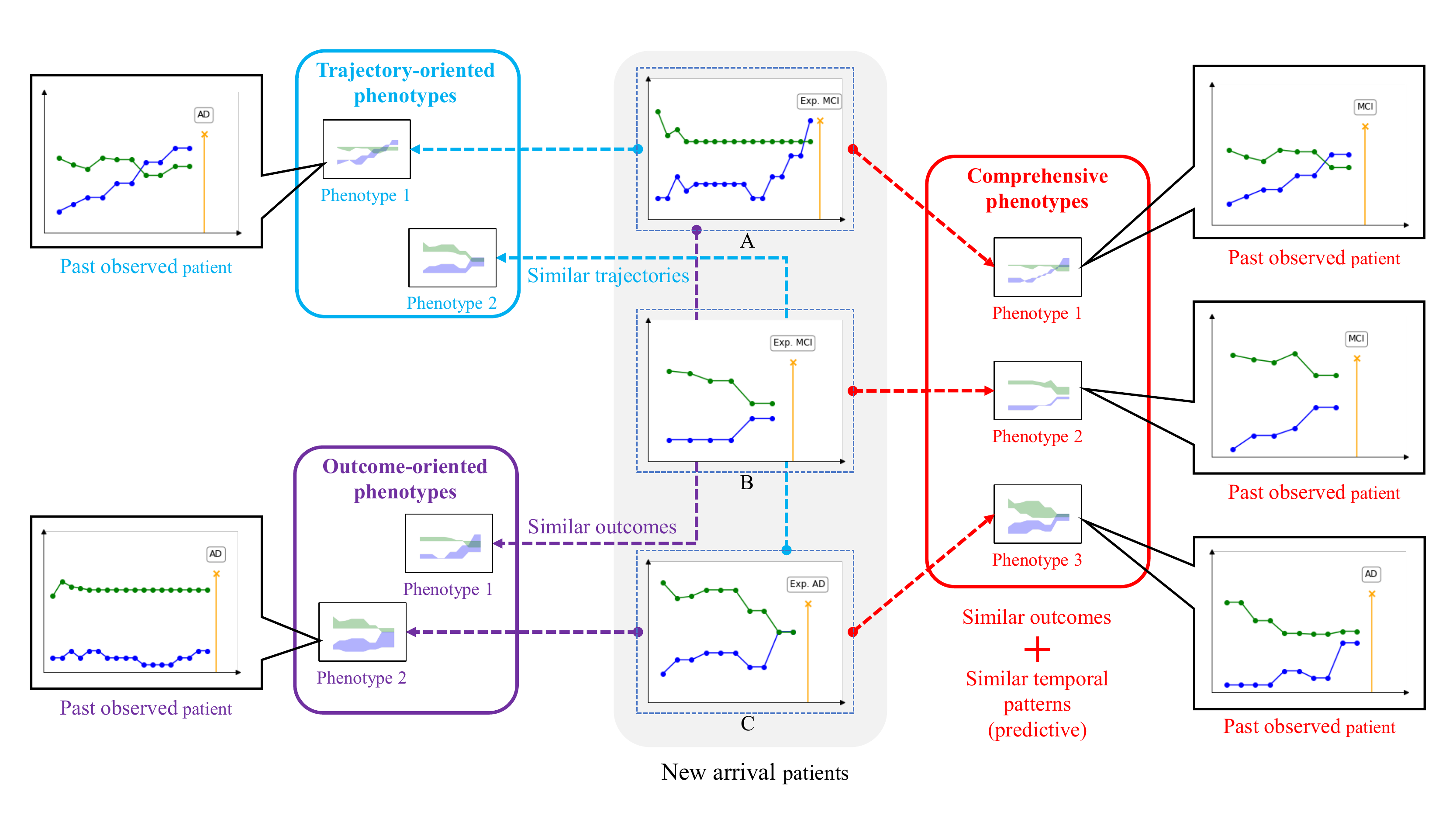}
	\caption{{Different Notions of Temporal Phenotypes.}
		\small
		Purely unsupervised clustering approaches focus on trajectory-oriented phenotypes (\textcolor{Cerulean}{blue}) and disregard the valuable information in patient outcomes. Predictive clustering methods aim at discovering outcome-oriented phenotypes (\textcolor{Purple}{purple}) which may not reflect the heterogeneity in patient trajectories despite the same diagnosis outcome. A desirable phenotyping method shall address both types of similarity and discover comprehensive phenotypes (\textcolor{Red}{red}).
		} 
		\label{fig:clusters}
	\end{center}
\end{figure*}

\textbf{Contribution.}\quad
In this paper, we propose a novel temporal clustering method to correctly uncover predictive temporal patterns descriptive of the underlying disease progression from the labeled time-series data.  
First, we formally define the notion of temporal phenotypes as predictive temporal patterns. 
Then, the association of individual temporal patterns with the target disease status is assessed by proposing a novel path-based similarity score. For effective evaluation of the path-based similarity, we introduce a representation learning approach based on the Laplace transform to convert variable-length, irregularly sampled time-series data into unified embeddings.  
Finally, based on the resulting path-based similarity graph, we formulate the task of temporal phenotyping as a temporal predictive clustering problem that can be efficiently solved by adopting the graph-constrained $K$-means clustering. 

We validate our approach through experiments on synthetic and real-world time-series datasets. Our method discovers temporal phenotypes that provide superior prediction performance compared to state-of-the-art benchmarks, and we corroborate the interpretability of our discovered phenotypes with supporting medical and scientific literature.

\section{TEMPORAL PHENOTYPING}
Suppose disease progression manifests through a multi-variate continuous-time trajectory $\vx(t)\in \mathcal{X}$ defined on $t\in [0,1]$, where $\mathcal{X}$ is the functional space of all possible patient trajectories.\footnote{
Trajectories defined within the interval $\mathbb{R}_{+}$ can be simply scaled to the unit interval $[0,1]$.
} Each trajectory consists of $\mathrm{dim}_x$-dimensional time-varying covariates, i.e., $\vx(t) = [x_1(t),\ldots,x_{\mathrm{dim}_x}(t)]^\top$, each of which can be described by a continuous-time function $x_i$ in $\normltwo_{[0,1]}$ (i.e., $\normltwo$-space under the interval $[0,1]$).\footnote{In many practical scenarios, the continuous-time functions for time-varying covariates are bounded and fall into the $\normltwo$-space which has a natural extension of Euclidean distance.
}  
Thus, the considered trajectory space can be given as $\mathcal{X} = \bigotimes_{\mathrm{dim}_x} \normltwo_{[0,1]}$. 
Each trajectory $\vx$ is correlated with a target label vector $\vy = [y_1, \dots, y_{\mathrm{dim}_y}]^\top \in \mathcal{Y}$ that describes the clinical status of the underlying disease progression (e.g., clinical endpoints). 
Throughout the paper, we focus our description on the case where the outcome of interest $\vy$ is categorical and represented by a one-hot vector, i.e., $\mathcal{Y} = \{0,1\}^{\mathrm{dim}_y}$.

Let $p(\vx, \vy)$ be the joint distribution of the continuous-time trajectory and the label vector. 
To discover temporal patterns that are predictive of the clinical status of patients, we first define a vector-valued function $g(\vx) = [p(y_1|\vx),\ldots, p(y_{\mathrm{dim}_y}|\vx)]^\top$ which implies the categorical conditional distribution $p(\vy|\vx)$. 
We assume the clinical status conditioned on a patient trajectory can be represented by one of the \textit{$\delta$-separable modes} in $g(\vx)$.
These modes are $\delta$-separable such that they can be separated based on a proper distance metric $\mathrm{d}_y$ with some threshold $\delta>0$.
Here, we choose the Jensen–Shannon (JS) divergence as our distance metric, i.e., $\mathrm{d}_y(\vv, \vu) = \frac{1}{2}KL( g(\vv) || \vm) + \frac{1}{2}KL( g(\vu) || \vm)$, where $KL$ is the Kullback-Leibler divergence, $\vm=\frac{g(\vv) + g(\vu)}{2}$.

\subsection{Phenotypes: Predictive Temporal Patterns}
In this subsection, we introduce the formal definition of \textit{phenotypes} as temporal patterns that are predictive of disease progression.
To this goal, we start by describing how the temporal patterns in continuous-time trajectories can be discovered and how the specific disease progression can be associated with each individual pattern.

\textbf{Temporal Patterns.}\quad
A temporal pattern characterizes some temporal dynamics that are shared by a subset of trajectories in $\mathcal{X}$.
Here, we introduce a novel definition to describe temporal patterns in the general form based on connectivity in trajectory space $\mathcal{X}$.
Given two trajectories $\vx^1, \vx^2\in \mathcal{X}$, we define a translation from $\vx^1$ to $\vx^2$, denoted as $\Gamma(\vx^1 \rightarrow \vx^2)$, as a continuous path $\Gamma$ connecting the two trajectories in space $\mathcal{X}$. Typically, $\Gamma(\vx^1 \rightarrow \vx^2)$ can continuously morph the shape of $\vx^1$ into that of $\vx^2$. 
Then, we formally define a \textit{temporal pattern} as a connected set $\Phi \subset \mathcal{X}$ such that all the trajectories in $\Phi$ can be inter-connected by translations within $\Phi$. That is, there exists a series of translations from any trajectory to any other trajectory in $\Phi$.

\textbf{Phenotypes.}\quad
Considering multivariate continuous-time trajectories, a variety of temporal patterns may exist in $\mathcal{X}$ while only a few of them are relevant to the target label. 
In the meantime, the clinical status marked by the same target label may manifest in patient trajectories through different temporal characteristics. 
For instance, in lung transplant referral of cystic fibrosis patients, (i) low lung function score, (ii) rapid declining lung function score, and (iii) multiple exacerbations requiring intravenous antibiotics are identified as distinct predictive temporal patterns \citep{ramos2019lung} among various temporal dynamics.

To provide insights on disease progression, desirable phenotypes shall be defined based on distinct predictive temporal patterns.
In line with such notion of phenotypes, we propose a new path-based similarity score that measures the variation of conditional label distribution (described by function $g(\vx)$) along a translation between two trajectories.
Specifically, consider two continuous-time trajectories $\vx^1, \vx^2$ and a translation $\Gamma(\vx^1 \rightarrow \vx^2)$, the score function evaluates the similarity between $\vx^1$ and $\vx^2$ via their impact on label $\vy$ through path $\Gamma$ as follows:
\begin{equation}
    \mathrm{d}_{\Gamma} (\vx^1,\vx^2) = \max_{\substack{\vx \in \Gamma(\vx^1\rightarrow\vx^2) \\ i \in \{1,2\}}} {\mathrm{d}_y (g(\vx), g(\vx^i))}.\label{eq:path-based-test}
\end{equation}
Small value of $\mathrm{d}_{\Gamma} (\vx^1,\vx^2) $ indicates that trajectories $\vx^1$ and $\vx^2$ share similar clinical status $\vy$ and contain similar temporal patterns that are predictive of their associated label.

Finally, we provide a formal definition of \textit{phenotype} as a predictive temporal pattern associated with a distinct clinical status as follows:
\begin{definition}\label{def:phenotype}
	\textnormal{\textbf{(Phenotype)}} 
	Let $\vv$ be the centroid of a $\delta$-separable mode in $g(\vx)$. Then, there exists a unique phenotype, denoted as a tuple $(\vv, \Phi)$ with $\Phi$ as a set of trajectories, that satisfies the following two properties:
	\begin{equation} \nonumber
	    \begin{aligned}
	    &\text{(Similar clinical status)} &\max_{\vx \in \Phi}~\mathrm{d}_y(g(\vx), \vv)\leq \frac{\delta}{2}, \\
	    &\text{(Similar predictive pattern)} &\max_{\substack{\vx^1,\vx^2\in \Phi\\\Gamma\subseteq \Phi}} \!\!\! \mathrm{d}_{\Gamma}(\vx^1, \vx^2) \leq \delta,
	    \end{aligned} 
	\end{equation}
	and any trajectory $\vx \in \mathcal{X} \setminus \Phi$ is either not connected to $\Phi$ or has a different mode.	
\end{definition}
Intuitively, the homogeneity of each phenotype $(\vv, \Phi)$ guarantees that the continuous-time trajectories exhibiting a similar temporal pattern will lead to a similar clinical status, which in turn provides a prognostic value on the underlying disease progression.

\subsection{Predictive Temporal Clustering}
In practice, the continuous-time trajectories of a patient are systematically collected in EHRs as discrete observations with irregular intervals during his/her regular follow-ups or stay at hospital. 
Hence, we focus this subsection on formulating the task of discovering phenotypes given discrete observations of trajectories as a novel clustering problem.

Suppose we have a dataset $\mathcal{D}=\{(\vt^i, \mX^i,\vy^i)\}_{i=1}^N$ comprising discrete observations on the underlying continuous-time trajectories and target labels. Here, we denote discrete observations as time-series $\mX=[\vx(t_1),\vx(t_2),\ldots,\vx(t_T)]$ which contains sequential observations of a trajectory $\vx$ at observation time stamps $\vt= [t_1, t_2 ,\ldots, t_T]^\top$ with $0\leq t_1 \leq \ldots \leq t_T \leq 1$.
The label vector $\vy\in \mathcal{Y}$ describes the clinical status sampled from the conditional distribution $p(\vy|\vx)$.
From this point forward, we will slightly abuse the notation and interchangeably write $\mX$ to denote the discrete time-series and the associated time stamps.

\textbf{Path-Based Connectivity.}\quad
Note that the property of a phenotype in Definition \ref{def:phenotype} requires all trajectories in that phenotype share a similar predictive pattern. 
Consider two time-series $\mX^1, \mX^2$ with underlying continuous-time trajectories $\vx^1, \vx^2$ from the same phenotype $(\vv, \Phi)$. 
There must exist a translation $\Gamma$ from trajectory $\vx^1$ to $\vx^2$ such that the condition in $\mathrm{d}_{\Gamma} (\vx^1,\vx^2)\leq \delta$ holds.
Violating such a condition implies a significant difference between the two trajectories suggesting they are from different phenotypes. 
Therefore, we utilize the \textit{path-based connectivity test}, i.e.,  $\exists\, \Gamma(\vx^1\rightarrow\vx^2),\, \mathrm{d}_{\Gamma} (\vx^1,\vx^2)\leq \delta$, to assesses the phenotype similarity between two given trajectories $\mX^1$ and $\mX^2$. 
This enables discovery of predictive temporal patterns without access to the ground-truth phenotypes. 
Evaluation of the path-based connectivity on all possible pairs of time-series in dataset $\mathcal{D}$ generates a distance matrix $\mS$.
Element-wise comparison of $\mS$ and threshold $\delta$ yields
a similarity graph $\gG_\delta$ with edges between similar samples. 
We will discuss how we can approximately achieve the path-based connectivity test based on the discrete observations in the next section.

\begin{figure*}[tb!]
	\begin{center}
	\centering
    \includegraphics[clip, trim=1.5cm 3cm 1.5cm 2.5cm, width=0.75\textwidth]{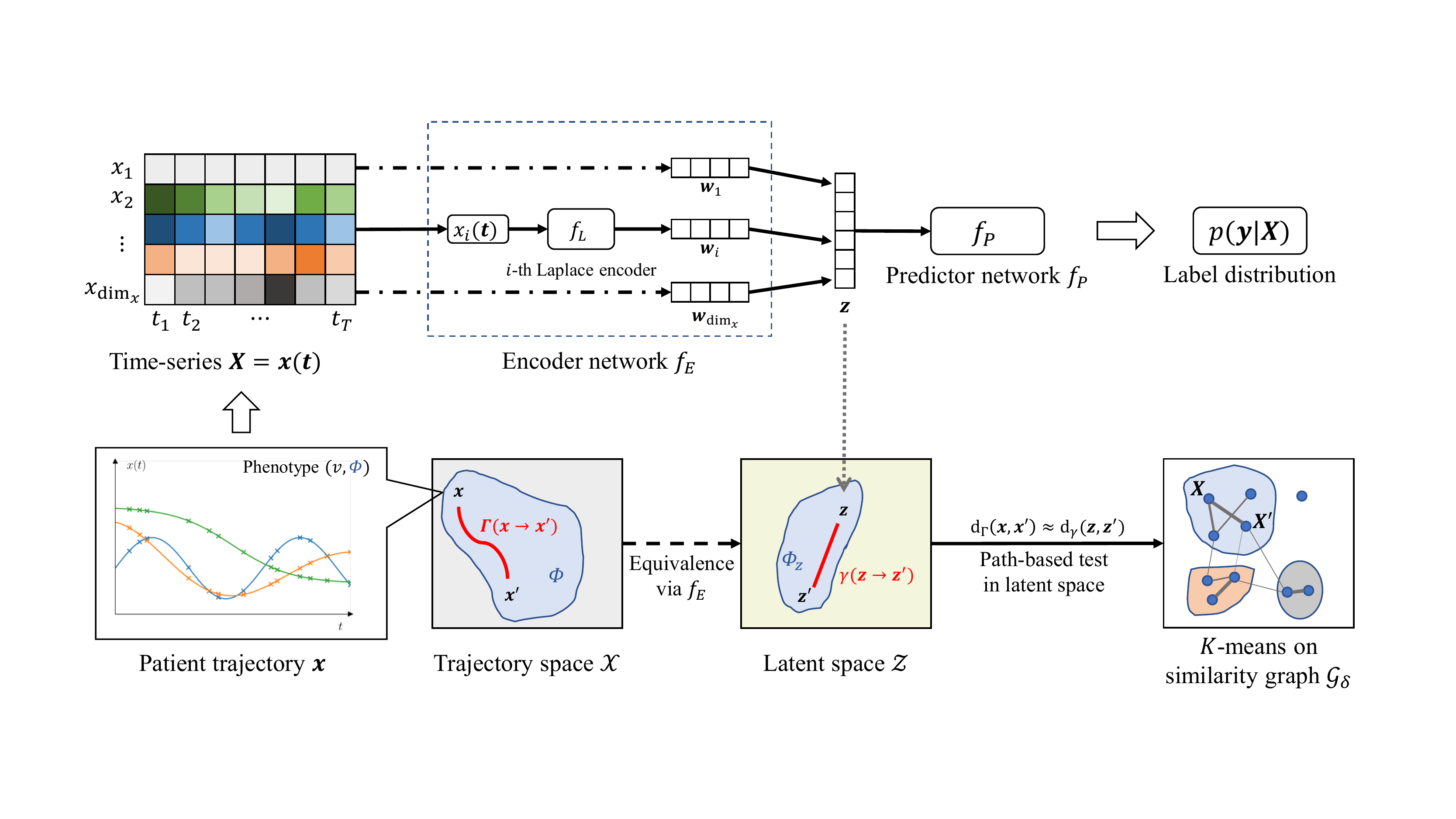}
	\caption{{Overview of {\ours}.}
		\small} 
		\label{fig:overview}
	\end{center}
\end{figure*}

\textbf{Temporal Phenotyping.}\quad
To discover phenotypes from dataset $\mathcal{D}$, we assume that we have a proper approximator $f(\mX)$ of the conditional label distribution $g(\vx)$ from discrete observations in $\mX$.
Thus, similarity graph $\gG_\delta$ can be constructed based on the path-based connectivity test with approximator $f(\mX)$. 
Now, we formulate the task of temporal phenotyping as a predictive clustering problem \citep{lee2020temporal} to group time-series into different clusters on top of $\gG_\delta$. 
More specifically, the clusters (with distinct phenotypes) are discovered by solving the following constrained optimization problem: 
\begin{equation}
\begin{aligned}
\min_{\mathcal{C}} \quad & \sum_{C_k \in \mathcal{C}}\sum_{\mX\in C_k} {\mathrm{d}_y(f(\mX), \vv_k)},\\
    s.t. \quad & \forall \mX^1, \mX^2 \in C_k,~~\mX^1 \xleftrightarrow{\gG_\delta}\mX^2,
\end{aligned}
\label{eq:problem}
\end{equation}
where $\mathcal{C} = \{C_1, C_2, \dots, C_K\}$ is a feasible set of $K\in \sN$ clusters each of which has a centroid $\vv_k$ as the average density $f(\mX)$, 
Since threshold $\delta$ is usually unknown in advance, we set its value according to $\delta = 2 \max_{C_k\in \mathcal{C},\mX\in C_k} \mathrm{d}_y(f(\mX), \vv_k)$ for consistency with Definition \ref{def:phenotype}.
Here, $\mX^1 \xleftrightarrow{\gG_\delta}\mX^2$ implies that there exists a path over graph $\gG_\delta$ such that $\mX^1$ and $\mX^2$ are interconnected. 
In (\ref{eq:problem}), the objective function  encourages the cluster centroids to be clearly distinguished in approximated label distribution $f(\mX)$ while the constraint on similarity graph $\gG_\delta$ ensures that samples in the same cluster are of similar phenotypes.
Each discovered cluster $C_k$ represents a unique phenotype with centroid $\vv_k$ describing the associated clinical status and allows us to explain the predictive temporal pattern in terms of the collection of time-series in $C_k$.

Unfortunately, the optimization problem in (\ref{eq:problem}) is highly non-trivial due to the following two challenges: First, it requires to learn a proper approximation of the conditional label distribution from irregularly-sampled discrete time-series. Second, an efficient evaluation of the path-based connectivity test is required to construct similarity graph $\gG_\delta$ given discrete time-series in $\mathcal{D}$.

\section{METHOD: T-PHENOTYPE}
In this section, we propose a novel temporal clustering framework, \textit{\ours}, that effectively discovers phenotypes from discrete time-series data. 
To estimate the conditional label distribution from discrete time-series, we introduce two networks, an encoder and a predictor. 
The encoder, $f_{E}$, comprises $\mathrm{dim}_x$ feature-wise Laplace encoders, each of which transforms a single feature dimension of discrete time-series $\mX$ into a fixed-length latent embedding.
The predictor, $f_{P}$, takes embeddings from $\mathrm{dim}_x$ Laplace encoders as the input $\vz$ in the latent space and estimates the conditional label distribution. 
The proposed Laplace encoders, $f_{{L}}$, allow us to 
establish (approximately) equivalence translation in the latent space and thereby to efficiently evaluate the path-based connectivity test between discrete time-series in dataset $\mathcal{D}$. 
Then, given an approximate similarity graph $\gG_\delta$ constructed from the result of pair-wise connectivity test, we propose a graph-constrained $K$-means algorithm to discover distinct phenotypes.  The overview of steps involved in \ours~ is illustrated in Figure~\ref{fig:overview}.

\subsection{Time-Series Embedding via Laplace Encoder}
\label{sec:encoder}
Now, we introduce a novel time-series encoder which encodes each dimension of a given discrete time-series into a unified parametric function in the frequency domain as an approximation of the Laplace transform.

\textbf{Laplace Encoder.} \quad
Let $x(\vt) = [x(t_1), \dots, x(t_T)]^\top \in \mathbb{R}^{T}$ be a time-series of discrete observations on a univariate trajectory $x(t)$ at time stamps $\vt = [t_1, \ldots, t_T]^\top$ in the unit interval. 
The Laplace encoder (parameterized by $\theta_L$), $f_{L}: \mathbb{R}^{T} \rightarrow \mathbb{C}^{n(d+1)}$, encodes discrete time-series $x(\vt)$ into a rational function on the complex plane with $n\in \sN$ poles of maximum degree of $d\in \sN$ as follows:
\begin{equation} \label{eq:laplace-transform}
    F_{w}(s) = \sum_{m=1}^n \sum_{l=1}^d \frac{c_{m,l}}{(s-p_m)^l}, ~~~~ c_{m,l}, p_m\in \sC. 
\end{equation}
Here, $\vw \triangleq f_{L}(x(\vt))= [p_1, \ldots, p_n, c_{1,1}, \ldots, c_{n, d}]^\top$ is the Laplace embedding comprising the poles and the corresponding coefficients. 
Note that the poles in (\ref{eq:laplace-transform}) are distinct and are in a lexical order, i.e., $p_m \leq p_{m+1}$ for $m=1,\dots, n-1$ where $p_{m} \leq p_{n}$ if and only if $\mathrm{Re}(p_m) < \mathrm{Re}(p_{n})$ or $\mathrm{Re}(p_m) = \mathrm{Re}(p_{n}) \wedge  \mathrm{Im}(p_m)\leq \mathrm{Im}(p_{n})$ holds. 
Then, the time-domain function can be efficiently reconstructed via the inverse Laplace transform: 
\begin{equation} 
    \hat{x}(t)
    = \frac{1}{2\pi j} \lim_{T\rightarrow \infty} \int_{\sigma-jT}^{\sigma+jT} e^{st}F_{w}(s)\mathrm{d}s,\label{eq:inverse-laplace-transform}
\end{equation}
where $j^2 = -1$ and $\sigma$ is some suitable complex number such that $\mathrm{Re}(\sigma) > \max_{p_m\in \vw} \mathrm{Re}(p_m)$.
With a sufficient number of poles, the Laplace embedding $\vw$ becomes an equivalent description of the underlying trajectory $x(t)$.
That is, the orthonormal basis $\{e^{2\pi j m t}, m\in \sZ\}$ of $\normltwo_{[0,1]}$ is covered by the reconstruction $\hat{x}(t)$ when $n\rightarrow \infty$.

Given a dataset of $N$ discrete univariate time-series, i.e., $\{x^{i}(\vt)\}_{i=1}^{N}$, we train the Laplace encoder utilizing the following loss function that consists of the time-series reconstruction error and the  regularization term specifically designed to encourage unique Laplace embeddings:
\begin{equation}
\mathcal{L}_\mathrm{laplace}(\theta_L) = \mathcal{L}_\mathrm{mse}(\theta_L) + \alpha \mathcal{L}_\mathrm{unique}(\theta_L)
 \label{eq:loss-encoder}
\end{equation}
where $\alpha$ is a balancing coefficient. The former term, i.e., $\mathcal{L}_\mathrm{mse}(\theta_L)= \frac{1}{N} \sum_{i=1}^{N} \|x^{i}(\vt) - \hat{x}^{i}(\vt)\|^2_2$, is the reconstruction error from our Laplace embeddings, and the latter term, i.e., $\mathcal{L}_\mathrm{unique}(\theta_L)=\frac{1}{N(N-1)} \sum_{i\neq j} \ell_\mathrm{unique}(\hat{x}^{i}(\vt), \hat{x}^{j}(\vt))$, encourages the uniqueness of the Laplace embedding. More specifically, $\ell_\mathrm{unique}$ focuses on three aspects -- (i) the obtained poles are distinct, (ii) the reconstructed trajectories are real-valued, and (iii) no two distinct Laplace embeddings generate the same trajectory. We further elaborate the uniqueness regularization in the Appendix.

\textbf{From Trajectory Space to Latent Space.} \quad
Utilizing $\mathrm{dim}_x$ feature-wise Laplace encoders as our encoder, $f_{E}$, any discrete observations of a continuous-time trajectory $\vx \in \mathcal{X}$ can be transformed into a fixed-length embedding $\vz \in \mathcal{Z}$ in the latent space as a composition of $\mathrm{dim}_x$ Laplace embeddings, i.e., $\vz \triangleq [f_{L}(x_{1}(\vt)), \dots, f_{L}(x_{\mathrm{dim}_x}(\vt))]^{\top}$. 
The following proposition builds a strong connection between the trajectory space $\mathcal{X}$ and the latent space $\mathcal{Z}$:
\begin{proposition} \label{proposition1}
    Without loss of generality, consider univariate continuous-time trajectories $\hat{x}^1, \hat{x}^2 \in \mathcal{X}$ and their corresponding latent embeddings $\vz^1, \vz^2 \in \mathcal{Z}$, respectively. Then, the distance between two trajectories can be bounded by $\| \hat{x}^1 - \hat{x}^2 \|_{\normltwo_{[0,1]}}^{2} \leq \psi \| \vz^1 - \vz^2\|_2^2$, 
    where $\psi>0$ is a constant and $\| x(t) \|_{\normltwo_{[0,1]}}^2 \!\!=\! \int_{0}^1 \! {x(t) \overbar{x(t)} \mathrm{d}t}$.
\end{proposition}
The detailed proof can be found in the Appendix.
Consider a subset of latent variables $\Phi_z$ and the corresponding trajectory set $\Phi$ of their time-domain representations.
The upper bound in Proposition \ref{proposition1} implies that continuity of $\Phi_z$ in the latent space leads to the continuity of $\Phi$ in the trajectory space.
This property allows efficient evaluation of the path-based connectivity test in the latent space as illustrated in the following subsection.

\subsection{Efficient Evaluation of Path-based Similarity}
Construction of similarity graph $\gG_\delta$ involves iterative evaluation of the path-based similarity score $\mathrm{d}_\Gamma$ in (\ref{eq:path-based-test}) for all possible pairs of time-series samples in  $\mathcal{D}$. This requires a substantial number of computations in both constructing translation $\Gamma$ and calculating conditional $g(\vx)$ on all available continuous-time trajectories $\vx\in \Gamma$. 
Instead, we efficiently approximate the similarity graph $\gG_\delta$ via path-based connectivity test in the latent space and estimate the conditional $g(\vx)$ via neural networks.

\textbf{Translation in Latent Space.} \quad
Consider two trajectories $\hat{\vx}^1, \hat{\vx}^2\in \mathcal{X}$ with the corresponding latent embedding $\vz^1, \vz^2 \in \mathcal{Z}$.
For any translation $\Gamma(\hat{\vx}^1\rightarrow \hat{\vx}^2) \subseteq \mathcal{X}$ in trajectory space, we can always find a continuous path in the latent space, i.e., ${\gamma}(\vz^1 \rightarrow \vz^2) \subseteq \mathcal{Z}$, such that the distance between its time-domain reconstruction and $\Gamma$ is minimized.
We consider $\gamma$ to be an (approximately) equivalent translation of $\Gamma$.\footnote{The equivalence is strict when all trajectories along translation $\Gamma$ have rational Laplace transform as described in (\ref{eq:laplace-transform}).} This enables us to capitalize on the translation in the latent space without constructing intermediate trajectories along path $\Gamma$, which significantly reduces computations in obtaining the path-based similarity in (\ref{eq:path-based-test}).

\textbf{Predictor.}\quad
To estimate the function $g(\vx)$, we utilize the time-series encoder $f_{E}$, which consists of $\mathrm{dim}_{x}$ Laplace encoders, and a predictor $f_{P}$ (an MLP parameterized by $\theta_{P}$) to construct the approximator as $f(\mX) \triangleq f_{P}\circ f_{E}(\mX) \approx g(\vx)$ where $\mX$ is the discrete observation of trajectory $\vx$. 
The predictor $f_{P}$ is trained based on the cross-entropy loss:
\begin{equation}\label{eq:loss-predictor}
    \mathcal{L}_\mathrm{predictor}(\theta_{P}) = - \frac{1}{N} \sum_{i=1}^{N}  \sum_{c=1}^{\mathrm{dim}_y} \vy^{i}_{c} \log f_{P}(\vz^{i})_{c},
\end{equation}
where $\vz = f_{E}(\mX)$ and subscript $c$ indicates the $c$-th element in the output space. To maintain the property of the Laplace encoders, we only update the predictor via the signal from the label during training.

Consider a trajectory translation $\Gamma(\hat{\vx}^1\rightarrow\hat{\vx}^2)$ and its equivalent translation $\gamma(\vz^1\rightarrow\vz^2)$ in latent space, the path-based similarity can be approximately calculated as
\begin{equation}
        \mathrm{d}_\Gamma(\hat{\vx}^1, \hat{\vx}^2) \approx  \,\mathrm{d}_\gamma(\vz^1, \vz^2) = \!\! \max_{\vz\in \gamma, i=1,2} \mathrm{d}_y(f_{P}(\vz), f_{P}(\vz^i)).
\label{eq:latent-path-based-test}
\end{equation} 
Hence, given two discrete time-series $\mX^1$ and $\mX^2$, the path-based connectivity test can be efficiently performed along translation $\gamma$ in the latent space without assessing the corresponding translation in the trajectory space ${\mathcal{X}}$.

\textbf{Approximate Similarity Graph.}\quad
Consider a phenotype $(\vv, \Phi)$ where centroid $\vv$ represents a specific clinical status and $\Phi$ is the associated predictive temporal pattern. 
The encoder $f_{E}$ is trained to map time-series $\mX$ sampled from trajectories in $\Phi$ into a connected area $\Phi_z$ in latent space $\mathcal{Z}$ via Laplace encoders. 
Given time-series $\mX$ that is observed from trajectory $\vx\in \Phi$, Definition \ref{def:phenotype} implies that we have $\mathrm{d}_y(f(\mX), \vv)\leq \frac{\delta}{2}$ where $f(\mX) = f_P(\vz)$ and $\vz=f_E(\mX)\in \Phi_z$. 
Hence, for two embeddings $\vz^1, \vz^2 \in \Phi_z$, there always exist a translation $\gamma(\vz^1\rightarrow\vz^2)\subseteq \Phi_z$ such that $\mathrm{d}_\gamma(\vz^1, \vz^2) \leq \delta$ due to the connectivity of $\Phi_z$ in the latent space. 
If two latent embeddings $\vz^1, \vz^2$ are located in the same convex subset of $\Phi_z$, linear path $\Bar{\gamma}(\vz^1\rightarrow\vz^2) = \{\vz| (1-a) \vz^1 + a \vz^2 , a\in [0,1]\}$ suffices the connectivity test. 
When $\vz^1$ and $\vz^2$ are in different convex subsets, the connectivity of $\Phi_z$ guarantees that there exists a series of intermediate points $\vz^{m_1}, \vz^{m_2}, \ldots, \vz^{m_l}$ such that composite path $\gamma(\vz^1\rightarrow\vz^2) = \Bar{\gamma}(\vz^1\rightarrow\vz^{m_1}) \cup \ldots \cup\Bar{\gamma}(\vz^{m_l}\rightarrow\vz^2) $ is inside $\Phi_z$ and can be used for connectivity test.
Therefore, in this work, we simplify the path-based connectivity test to the linear paths between latent variables
as the similarity between two time-series can be inferred based on these linear paths. 
Overall, given two time-series $\mX^i$ and $\mX^j$, we calculate the approximate distance $\mathrm{d}_{\Bar{\gamma}}(f_E(\mX^i),f_E(\mX^j))$ via discrete points along path $\Bar{\gamma}$, which is stored in element $S_{ij}$ of \textit{path-based distance matrix} $\mS$. 
The approximate similarity graph $\gG_\delta$ is then constructed with edges between samples $\mX^i$ and $\mX^j$ if and only if $S_{i,j}\leq \delta$.

\begin{table*}[!htb]
\caption{{Comparison of Temporal Clustering Methods.}
\small
The difference in the notion of phenotypes and similarity measure are highlighted together with two desiderata:
(i) clusters are outcomes associated; and (ii) with interpretable insights on cluster assignment.} 
\label{tab:comparison}
\small
\centering
\begin{tabular}{lllcc}
\toprule
\capstr{\textbf{Method}} & \capstr{\textbf{Phenotype}} & \capstr{\textbf{Similarity}} \capstr{\textbf{Measure}} & \capstr{\textbf{(i)}} &\capstr{\textbf{(ii)}} \\
\midrule
Deep temporal $K$-means & Distance-based & Euclidean distance& \cmark&\xmark\\
\citet{bahadori2015functional} & Affinity-based &  Self-expression&\xmark&\xmark\\
\citet{chen2022clustering}& Pattern-oriented& Latent distance & \xmark & \cmark\\
\citet{aguiar2022learning}&Attention\&outcome-oriented& KL-divergence & \cmark& \cmark\\
\citet{lee2020temporal} & Outcome-oriented & KL-divergence & \cmark& \xmark\\
{\ours} (Ours) & Predictive pattern-oriented & Path-based connectivity & \cmark&\cmark\\
\bottomrule
\end{tabular}
\end{table*}

\subsection{Predictive Clustering on Similarity Graph}
Unfortunately, solving the clustering objective in (\ref{eq:problem}) is a NP-hard combinatorial problem. Thus, we introduce a greedy approach to discover the temporal clusters from the path-based distance matrix $\mS$ defined in the previous subsection.

The objective function in (\ref{eq:problem}) has the following upper bound:
\begin{equation} \label{eq:warm_start}
    \begin{aligned}
    J &\triangleq \sum_{C_k \in \mathcal{C}}\sum_{\mX\in C_k} {\mathrm{d}_y(f(\mX), \vv_k)},\\
    &\leq \sum_{C_k \in \mathcal{C}} \frac{1}{|C_k|} \sum_{\mX^i\!,\mX^j\in C_k} {\mathrm{d}_y(f(\mX^i), f(\mX^j))},\\
    &\leq  \sum_{C_k \in \mathcal{C}} \sum_{\mX^i\!,\mX^j\in C_k} {\mathrm{d}_{\Bar{\gamma}}(\vz^i, \vz^j)},\\
    &=  \sum_{C_k \in \mathcal{C}} \sum_{\mX^i\!,\mX^j\in C_k} {S_{ij}} \triangleq \Bar{J}(\mS) ,
    \end{aligned}
\end{equation}
where $\vz^i=f_E(\mX^i)$,
latent translation $\Bar{\gamma}$ is a linear path connecting two embeddings $\vz^i$ and $\vz^j$.
The first inequality comes from the convexity of the JS divergence, and the second inequality establishes from equation (\ref{eq:latent-path-based-test}) and the fact that $|C_k|\geq 1$.
Local minimum of the upper bound $\Bar{J}(\mS)$ can be achieved via a greedy $K$-partitioning algorithm based on pair-wise sample distances in matrix $\mS$.

Utilizing the approximate solution in (\ref{eq:warm_start}) as warm-start, we propose a graph-constrained $K$-means clustering approach to solve problem (\ref{eq:problem}) via a greedy breadth-first search algorithm \var{GK-means} (details in Appendix). The overview of our predictive clustering method, {\ours}, is given in Algorithm \ref{alg:clustering}. More details about the algorithm are provided in the Appendix.

\begin{algorithm}
\caption{~~\ours}\label{alg:clustering}
    \small
  \begin{algorithmic}
  \Require{dataset $\mathcal{D}$, number of clusters $K$}
    \Ensure{$\mathcal{C}=\{C_1, C_2, \dots, C_K \}$}
    \State calculate distance matrix $\mS$ based on (\ref{eq:latent-path-based-test})
    \State $\mathcal{C} \gets \argmin_{\mathcal{C}} \Bar{J}(\mS)$ \Comment{warm-start}
    \State $\delta \gets \log(2)$ \Comment{upper bound of $\mathrm{d}_\mathrm{JS}$}
    \While{not converged}
        \For{$k = 1,2,\ldots, K$}
        \State update cluster seed $e_k$ via (\ref{eq:centroid})
        \EndFor
        \State $\delta' \gets 2 \max_{C_k\in \mathcal{C},\mX\in C_k} \mathrm{d}_y(f(\mX), \vv_k)$
        \State $\delta \gets \min (\delta, \delta')$ \Comment{upper bound $J\leq N \delta$}
        \State create similarity graph $\gG_\delta$ from $S_{i,j}\leq \delta$
        \State $\mathcal{C} \gets \text{\var{GK-means}}(J| e_1,e_2,\ldots, e_K, \gG_\delta)$
    \EndWhile
  \end{algorithmic}
\end{algorithm}

The cluster seeds in Algorithm \ref{alg:clustering} are used to perform greedy cluster expansion over similarity graph $\gG_\delta$.
For the $k$-th cluster, the cluster seed $e_k = (\vv_k, \mX^{(k)})$
can be given as 
\begin{equation}
     \vv_k = \frac{1}{|C_k|} \!\sum_{\mX\in C_k}\!  f(\mX),~~\mX^{(k)} \!= \argmin_{\mX\in C_k}  \mathrm{d}_y(f(\mX),\vv_k),
    \label{eq:centroid}
\end{equation}
where $\vv_k$ is the cluster centroid and $\mX^{(k)}$ is the representative time-series in cluster $C_k$ with closest conditional to that of the centroid.

\section{RELATED WORK}

Different strands of clustering methods have been increasingly investigated for knowledge discovery from time-series data with various similarity notions accustomed to specific application scenarios. 
One strand is unsupervised clustering methods that adopt the traditional notion of similarity into the time-series setting. 
To flexibly incorporate with variable-length irregularly-sampled time-series observations, the traditional methods applied $K$-means clustering by either finding fixed-length and low-dimensional representations using deep learning-based sequence-to-sequence model \citep{ma2019learning, xzhang2019clustering} or on modifying the similarity measure such as dynamic time warping (DTW) \citep{agiannoula2018clustering} and the associated graph Laplacian \citep{lei2019similarity,hayashi2005embedding}. 
Alternatively, \citet{bahadori2015functional} focused on sample affinities to conduct spectral clustering, and 
\citet{chen2022clustering} proposed a deep generative model whose parametric space is then used for clustering.
Further, advanced hidden Markov models \citep{ceritli2022mixture} and Gaussian processes \citep{schulam2015clustering} have also been utilized together with hierarchical graph models in disease subtype discovery.
In general, these methods are limited by some model specifications such as the linear subspace assumptions and graphical models for the underlying data generation process.

Clusters identified through these methods are purely unsupervised -- they do not account for patients' clinical outcomes that are often available in EHRs -- which may lead to heterogeneous outcomes even for patients in the same cluster. To overcome this issue, another strand of clustering methods combine predictions on the future outcomes with clustering. 
\citet{lee2020temporal} proposed an actor-critic approach to divide time-series of patient trajectories into subgroups based on their associated clinical status. 
The discovered patient subgroups allow clinicians to investigate the temporal patterns related to the transition of disease stages. 
\citet{aguiar2022learning} extended it to capture phenotype-related feature contributions by employing an attention mechanism. 
Given predicted clusters, visualizing the associated attention map provides additional interpretability about the underlying disease progression. 

Unfortunately, actionable information that can be inferred from the aforementioned temporal predictive clusters is still limited. These methods primarily focus on finding the discrete representations that can best describe the outcome labels rather without properly associating with temporal patterns that can be found among time-series samples. 
In this paper, we propose a novel temporal clustering method to correctly uncover predictive temporal
patterns descriptive of the underlying disease progression
from the labeled time-series data. Therefore, our method  not only can provide clusters that have a prognostic value but also can offer interpretable information about the disease progression patterns.

\section{EXPERIMENTS}
In this section, we evaluate the clustering performance and the prognostic value of {\ours} with one synthetic dataset and two real-world datasets (detailed statistics are provided in the Appendix).

\textbf{Synthetic Dataset.} \quad
We construct a synthetic dataset of $N=1200$ samples with ground truth cluster  labels. Each sample comprises discrete observations of a 2-dimensional trajectory $\vx(t)$ and the target binary outcome. 
We design the two elements $x_1(t)$ and $x_2(t)$ to model trend and periodicity of a trajectory, respectively: we set $x_1(t)=\iota \cdot \mathrm{sigmoid}(a\cdot (t- b-\varphi ))$ with sign $\iota\in \{-1,1\}$, $a=10$, $b=0.5$, and $\varphi \sim \mathrm{exp}(\frac{3}{10})$ and set $x_2(t)=\sin(c\cdot(t-\varphi))$ with $c\in \{4,6,8\}$ and $\varphi$ identical to that of  $x_1$.
The trajectory $\vx=[x_1,x_2]^\top$ is irregularly observed over $20$ time stamps in $t\in[0,2]$ with a white noise $\mathcal{N}(0,0.1^2)$ for each variable. 
We set $c$ as the ground truth phenotype label representing different periodicity and set the target outcome label $y$ as $y=0$ when $c=6$ and $y=1$ otherwise.

\textbf{ADNI Dataset.}\quad
The Alzheimer’s Disease Neuroimaging Initiative\footnote{\scriptsize \url{https://adni.loni.usc.edu}} (ADNI) dataset includes records on the progression of Alzheimer’s disease (AD) of $N=1346$ patients with regular follow-ups every six months.
Each patient is associated with various biomarkers, evaluation of MRI and PET images, and cognitive tests results. We set the target outcome at each time stamp as the three diagnostic groups -- i.e., normal brain functioning (NL), mild cognitive impairment (MCI), and AD -- which is used to indicate different stages of AD progression. 
We focus on three important temporal variables -- i.e., the genetic biomarker of apolipoprotein (APOE) $\varepsilon4$ gene, the hippocampus evaluation from MRI, and the cognitive test result of CDRSB -- to predict the AD progression. 
 
\textbf{ICU Dataset.}\quad
The PhysioNet ICU\footnote{\scriptsize \url{https://physionet.org/content/challenge-2012/}} \citep{goldberger2000physiobank} dataset contains temporal observations on 42 covariates of adult patients over the first 48 hours of ICU stay. 
We extract $N=1554$ records of adult patients admitted to the medical or surgical ICU. Temporal covariates used in the experiments are age, gender, Glasgow Coma Scale (GCS), and partial pressure of arterial CO2 (PaCO2) with a time resolution of 1 hour, and we set patient mortality as the target binary outcome of interest.

\textbf{Baselines.}\quad
We compare the performance of {\ours} with the following benchmarks ranging from traditional method to recently developed deep learning-based methods, where each clustering method reflects a different notion of temporal phenotypes: 1) $K$-means with warping-based distance (KM-DTW); 2) deep temporal $K$-means with the encoder-predictor (E2P) structure introduced in \citep{lee2020temporal}, i.e., KM-E2P(z) and KM-E2P(y); 3) $K$-means on top of our proposed Laplace encoder (KM-$\mathcal{L}$); 4) sequence-to-sequence with $K$-means friendly representation space (SEQ2SEQ); and 5) the state-of-the-art temporal clustering approach AC-TPC \citep{lee2020temporal}. Detailed description can be found in Appendix.
In addition, we consider the ablation study of {\ours} with joint optimization for the Laplace encoders and predictor $f_P$ and denote such model with {\ours} (J).

\begin{figure*}[!htb]
	\begin{center}
	\centering
	\addtolength{\tabcolsep}{-0.6em}
  \begin{tabular}{ c }
    \begin{subfigure}{0.4\linewidth}
      \includegraphics[width=\linewidth]{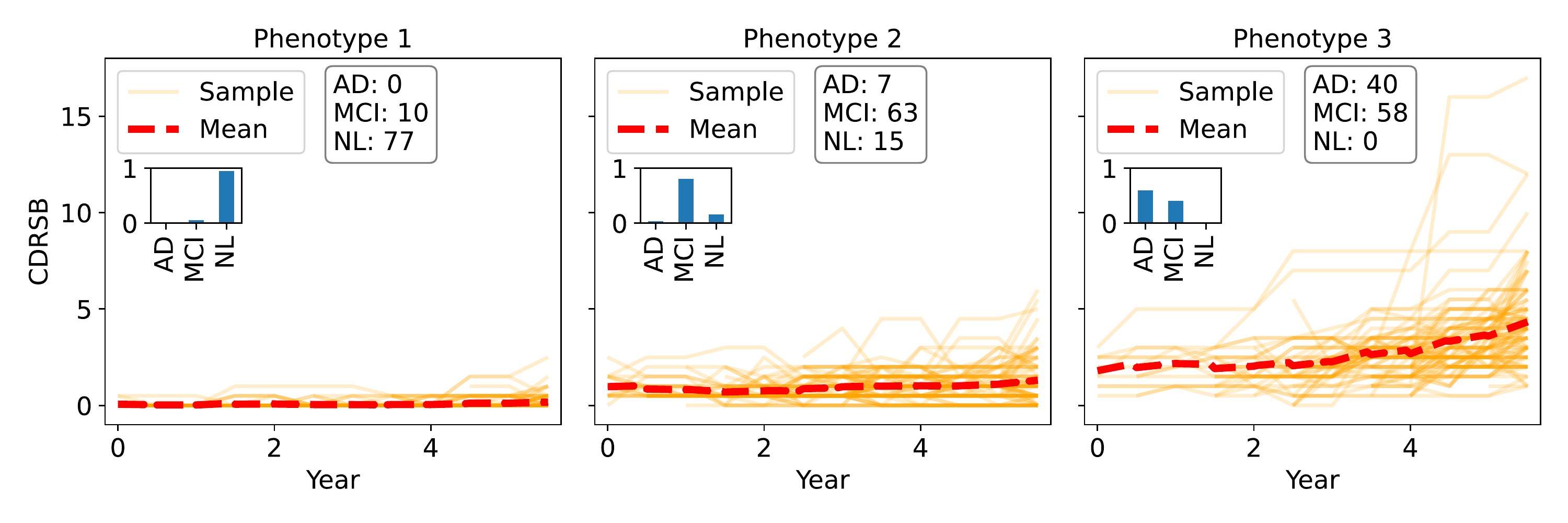}
      \caption{Three phenotypes from AC-TPC.}
        \label{fig:ADNI-AC-TPC}
    \end{subfigure}
        \\
        \begin{subfigure}{0.54\linewidth}
        \includegraphics[width=\linewidth]{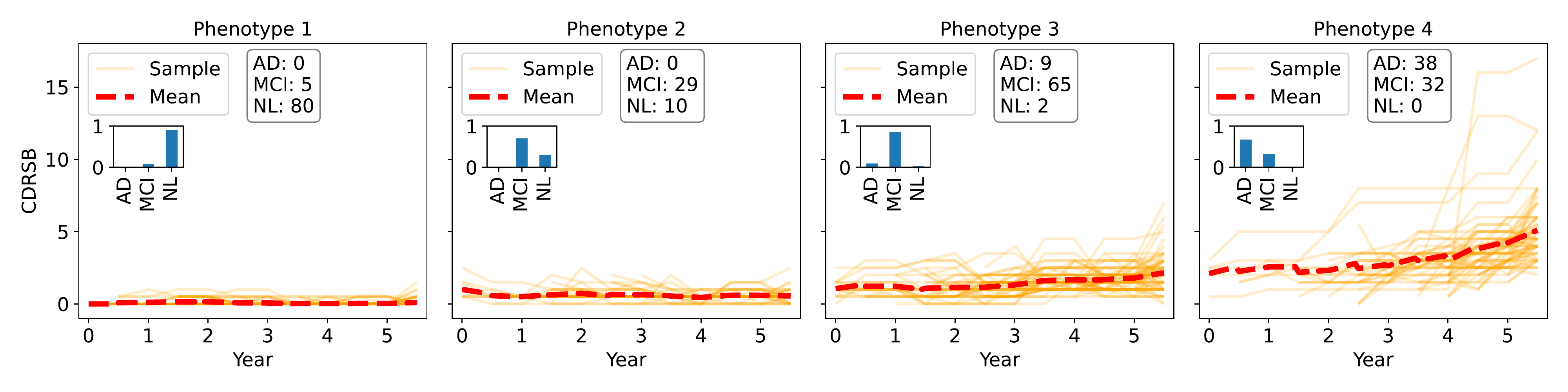}
      \caption{Four phenotypes from {\ours}.}
    \label{fig:ADNI-ours}
    \end{subfigure}
  \end{tabular}
  \begin{tabular}{ c }
  \begin{subfigure}{0.44\linewidth}
    \includegraphics[clip, trim=1.8cm 0cm 1.8cm 0cm, width=\linewidth]{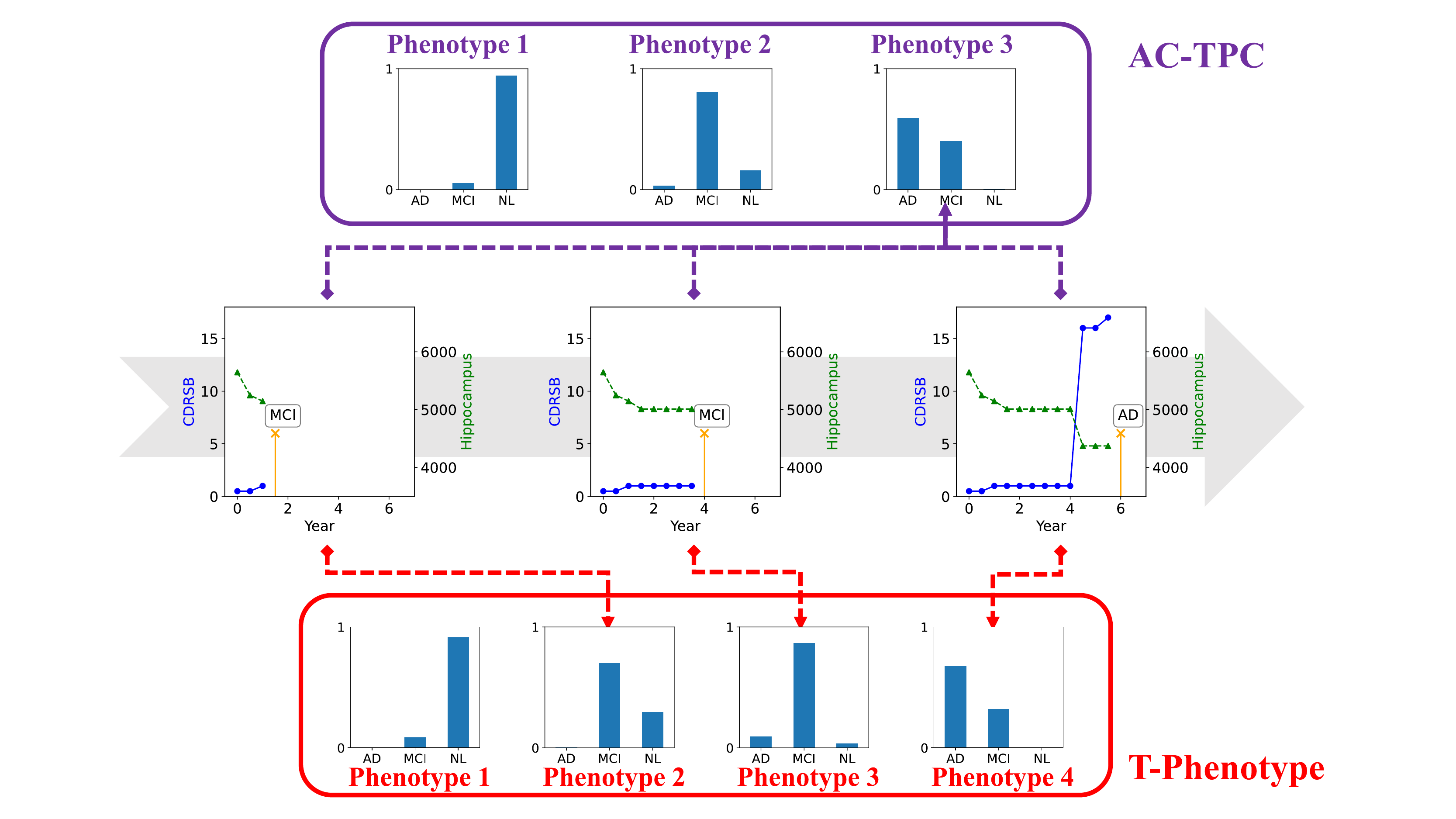}
    \caption{{Prognostic values of {\ours} and AC-TPC.}\small} 
		\label{fig:temporal-cluster}
    \end{subfigure}
  \end{tabular}%
      \caption{{Comparison of Phenotypes Discovered by {\ours} and AC-TPC on the ADNI Dataset.}\small} 
		\label{fig:ADNI-phenotypes}
	\end{center}
\end{figure*}

Throughout the experiments, time stamps of discrete time-series are scaled into $t\in [0,1]$.
For the synthetic and ADNI datasets, we use 64/16/20 train/validation/test splits in experiments.
To get reliable clustering performance measurement on the ICU dataset, we use 48/12/40 train/validation/test splits for experiments.
Hyperparameters of {\ours} and baselines are optimized through 3-fold cross-validation. 
For comparison of clustering performance, the number of clusters $K$ for each dataset is shared by all methods.
We select $K$ as a hyperparameter of {\ours}, and the optimal cluster numbers are determined to be $K=3$ (ground truth), $K=4$ and $K=3$ for the synthetic, ADNI and ICU dataset, respectively. Details can be found in the Appendix.

\begin{table}[!h]
\caption{Clustering Performance on the Synthetic Dataset.
}
\label{tab:benchmark-synth}
\begin{center}
\tiny
\begin{tabular}{lccc}
\toprule
\textbf{METHOD}   & \textbf{PURITY} & \textbf{RAND} & \textbf{NMI} \\
\midrule
KM-E2P(y)  &0.663$\pm$0.019 	&0.477$\pm$0.033 	&0.569$\pm$0.045 	 \\
KM-E2P(z) &	0.677$\pm$0.029 	&0.418$\pm$0.024 	&0.485$\pm$0.047 		 \\
KM-DTW & 0.469$\pm$0.017 &	0.068$\pm$0.021 &	0.077$\pm$0.022	 \\
KM-$\mathcal{L}$& 0.687$\pm$0.033 	&0.395$\pm$0.058 	&0.447$\pm$0.059\\ 
SEQ2SEQ &	0.378$\pm$0.008 	&-0.003$\pm$0.003 	&0.005$\pm$0.003 	\\ 
AC-TPC &	0.659$\pm$0.020 	&0.487$\pm$0.035 	&0.596$\pm$0.043	\\ \midrule
{\ours} (J) & 0.655$\pm$0.021 	&0.440$\pm$0.051 	&0.543$\pm$0.064\\
{\ours} & \textbf{0.965$\pm$0.018}$^\ddagger$  	&\textbf{0.902$\pm$0.048}$^\ddagger$  	&\textbf{0.875$\pm$0.050}$^\ddagger$ \\
\bottomrule
\end{tabular}
\end{center}
\tiny
Purity score, Rand index and normalized mutual information (NMI) are used to evaluate the clustering performance with ground truth phenotype labels. Best performance is highlighted in \textbf{bold}, and $^\ddagger$ indicates $p$-value $<0.01$.
\end{table}
\begin{table}[!h]
\caption{Benchmark Result on Two Real-world Datasets.}
\label{tab:benchmark-medical}
\begin{center}
\tiny
\addtolength{\tabcolsep}{-0.4em}
\begin{tabular}{llcccc}
\toprule
{} & \textbf{\capstr{Method}}   & \textbf{AUROC} & \textbf{AUPRC} & $H_\mathrm{ROC}$&$H_\mathrm{PRC}$\\
\midrule
\multirow{8}{*}{\rotatebox[origin=c]{90}{\textbf{ADNI}}}&
KM-E2P(y) & \textbf{0.893$\pm$0.005} 	&\textbf{0.728$\pm$0.017} 	&	0.770$\pm$0.013 &	0.701$\pm$0.012 \\
&KM-E2P(z) &0.884$\pm$0.012 &	0.711$\pm$0.020 & 	0.763$\pm$0.018 	&0.690$\pm$0.013\\
&KM-DTW & 0.743$\pm$0.013 	&0.522$\pm$0.020 & 	0.752$\pm$0.027 	&0.618$\pm$0.021\\
&KM-$\mathcal{L}$& 0.697$\pm$0.029 &	0.465$\pm$0.019  	&0.753$\pm$0.019 	&0.593$\pm$0.018\\
&SEQ2SEQ &	 	0.775$\pm$0.023 	&0.550$\pm$0.030 	&	0.773$\pm$0.012 	&0.642$\pm$0.022\\
&AC-TPC 	&0.861$\pm$0.012 &	0.665$\pm$0.020 &	0.788$\pm$0.014 	&0.694$\pm$0.013\\ \cmidrule{2-6}
&{\ours} (J) & 0.867$\pm$0.020 	&0.679$\pm$0.040 &	0.768$\pm$0.011 	&0.684$\pm$0.021\\
&{\ours} & 0.891$\pm$0.005 &	0.716$\pm$0.015 & 	\textbf{0.791$\pm$0.013} 	&\textbf{0.713$\pm$0.009}$^\ddagger$\\
\midrule
\multirow{8}{*}{\rotatebox[origin=c]{90}{\textbf{ICU}}}&
KM-E2P(y) &\textbf{0.697$\pm$0.014} 	&0.593$\pm$0.012 &	0.682$\pm$0.029 	&0.628$\pm$0.025 \\
&KM-E2P(z) &0.677$\pm$0.030 	&0.579$\pm$0.018 & 	0.686$\pm$0.031 &	0.633$\pm$0.024 \\
&KM-DTW & 0.539$\pm$0.030 	&0.515$\pm$0.011 & 	0.636$\pm$0.023 	&0.621$\pm$0.021\\
&KM-$\mathcal{L}$& 0.577$\pm$0.019 &	0.532$\pm$0.009 	&0.682$\pm$0.009 &	0.649$\pm$0.004\\
&SEQ2SEQ &0.592$\pm$0.024 &	0.539$\pm$0.012 & 	0.690$\pm$0.011 	&\textbf{0.653$\pm$0.004}\\
&AC-TPC 	&0.660$\pm$0.008 	&0.573$\pm$0.003 & 	0.695$\pm$0.014 &	0.644$\pm$0.011\\ \cmidrule{2-6}
&{\ours} (J)& \textbf{0.697$\pm$0.025} 	&\textbf{0.595$\pm$0.017} 	& 	0.691$\pm$0.056 &	0.636$\pm$0.048\\
&{\ours} & 0.681$\pm$0.017 &	0.585$\pm$0.015 &	\textbf{0.703$\pm$0.007} 	&0.648$\pm$0.008\\
\bottomrule 
\end{tabular}
\end{center}
\tiny
The area under the curve of receiving-operator characteristic (AUROC) and area under the curve of precision-recall (AUPRC) are used to assess the prognostic value of the discovered clusters on predicting target outcomes. 
Two composite metrics $H_\mathrm{ROC}$ and $H_\mathrm{PRC}$, calculated as harmonic means between predictive accuracy (AUROC or AUPRC) and a cluster consistency metric AUSIL, are used to measure the phenotype discovery performance.
Please refer to the Appendix for details. 
Best performance is highlighted in \textbf{bold}, and $^\ddagger$ indicates $p$-value $<0.01$.
\end{table}

\paragraph{Benchmark.}
The clustering performance of {\ours} is compared with six baselines, 
with all results reported using 5 random train/validation/test splits of the corresponding dataset.
Benchmark results on synthetic dataset and two real-world datasets are provided in Table \ref{tab:benchmark-synth} and Table \ref{tab:benchmark-medical}, respectively.
Complete benchmark tables are available in the Appendix. 
On the synthetic dataset, {\ours} outperforms all baselines with significant gaps in considered clustering accuracy metrics.
Similarly, {\ours} has the best (or very close to best) outcome prediction performance on both ADNI and ICU datasets and outperforms AC-TPC and most other baselines in phenotype discovery on the two datasets.
The baseline of KM-E2P(y) directly discovers clusters over predicted outcome distributions and achieves the best prediction performance on the ADNI dataset, which is within expectation.
However, its clustering performance, particularly $H_\mathrm{ROC}$, is inferior to that of {\ours} due to the negligence of similarity in temporal patterns. 
On the ICU dataset, while {\ours} has close phenotype discovery performance $H_\mathrm{PRC}$ to baseline SEQ2SEQ, the clusters discovered by our method provide greater prognostic values as reflected in the outcome prediction accuracy.

\textbf{Phenotypes of AD Progression.}\quad
The CDRSB score measures the impairment on both cognitive abilities and brain function \citep{coley2011suitability} and is widely used in AD progression assessment and staging \citep{kim2020disease,o2008staging}.
The temporal patterns in CDRSB trajectory vary in different disease stages and show stable prognostic power on patient outcomes \citep{delor2013modeling}.
On the ADNI dataset, four phenotypes are discovered by {\ours}.
We examine these phenotypes by plotting the CDRSB scores of $N_{test}=270$ test samples separately in corresponding clusters. 
As shown in \Figref{fig:ADNI-ours}, 
normal and high-risk patients with divergent cognitive test trajectories are correctly identified in phenotype 1 and 4 by {\ours}.
In the meantime, for the predicted outcome of MCI, two subtypes of patients are clearly separated into two phenotypes (2 and 3) with different growth rates in CDRSB score.
In comparison, AC-TPC fails to distinguish between these two subtypes as illustrated in \Figref{fig:ADNI-AC-TPC}, which impedes the prognostic value of clusters discovered by AC-TPC.

\textbf{Prognostic Value of {\ours}.}\quad
We further demonstrate the prognostic value of {\ours} with the temporal phenotyping results obtained on a typical patient from the ADNI dataset.
The studied patient had a positive biomarker of APOE $\varepsilon4$ gene which contributes to an increased risk of AD \citep{yamazaki2019apolipoprotein}.
Consecutive observations of patient covariates at three time stamps are plotted in \Figref{fig:temporal-cluster}.
Hippocampus volume (\textcolor{OliveGreen}{green triangle}) and CDRSB score (\textcolor{Blue}{blue dot}) are displayed together with diagnosis obtained at the next follow-up (\textcolor{YellowOrange}{yellow bar}).
The temporal phenotype assignment via {\ours} is shown at the bottom.
As a predictive factor of early-stage AD \citep{rao2022hippocampus}, fast decrease in hippocampus volume leads to the initial diagnosis of phenotype 2 (MCI) in \Figref{fig:ADNI-ours} by {\ours} despite a low CDRSB score from cognitive test.
Then, with a clear trend of increase appearing in CDRSB trajectory, the studied patient is classified into phenotypes ($2\rightarrow3\rightarrow4$) that reflect the growing risk in developing AD. 
In contrast, as shown on the top of \Figref{fig:temporal-cluster}, AC-TPC simply assigns the same phenotype to the patient throughout the considered time period and is unable to provide comparable insights on AD progression from the patient trajectory.

\section{CONCLUSION}

In this paper, we propose a novel phenotype discovery approach {\ours} to uncover predictive patterns from labeled time-series data.
A representation learning method in frequency-domain is developed to efficiently embed the variable-length, irregularly sampled time-series into a unified latent space that provides insights on their temporal patterns.
With our new notion of path-based phenotype similarity, a graph-constrained $K$-means approach is utilized to discover clusters representing distinct phenotypes.
Throughout experiments on synthetic and real-world datasets, we show that {\ours} outperforms all baselines in phenotype discovery. 
The utility of {\ours} to discover clinically meaningful phenotypes is further demonstrated via comparison with the the state-of-the-art temporal phenotyping method AC-TPC on real-world healthcare datasets.

\section{LIMITATIONS}
Our proposed method, {\ours}, leverages Laplace encoders as a general approach to capture temporal patterns from time-series data as distinct Laplace embeddings.
However, there may exist some complex temporal patterns, e.g., interactions between patient covariates at two specific time points, that cannot be encoded in this manner.
To address this issue, additional representations (e.g., representation via attention mechanism) from the input time-series can be introduced to augment the Laplace embedding, which we leave as a future work.
In the meantime, the phenotype discovery performance of {\ours} is highly dependent on the quality of predictor $f_P$.
Unstable predictions from $f_P$ will directly lead to inaccuracies in phenotype assignment.
Thus, effective regularization of the predictor network would be another important future direction.

\section{SOCIETAL IMPACT}
Discovery of phenotypes from disease trajectories is a long pursuit in healthcare.
In line with the target of precision medicine, the phenotype connects temporal patterns in patient trajectory and clinical outcomes is of great prognostic value since it allows clinicians to make more accurate diagnosis and issue the most appropriate treatment to their patients.
By combining notions of similarity in both patient trajectories and clinical outcomes, our method, {\ours}, can effectively identify phenotypes of desired property.
The discovered patient subgroups can be used to improve current clinical guidelines and help clinicians to better understand the disease progression of their patients.
Nevertheless, the association between temporal patterns and clinical outcome in a phenotype cannot be interpreted as causal relationship without careful tests and examinations.
Application of {\ours} without audits from human experts may lead to undesirable outcome of patients in certain edge cases.

\subsubsection*{Acknowledgements}
Yuchao Qin was supported by the Cystic Fibrosis Trust.
Changhee Lee was supported through the IITP grant funded by the Korea government (MSIT) (No. 2021-0-01341, AI Graduate School Program, CAU).
We thank all reviewers at AISTATS 2023 for their time in helping us to evaluate our work.
Their insightful comments are greatly appreciated.

\bibliography{reference.bib}

\begin{thebibliography}{36}
\providecommand{\natexlab}[1]{#1}
\providecommand{\url}[1]{\texttt{#1}}
\expandafter\ifx\csname urlstyle\endcsname\relax
  \providecommand{\doi}[1]{doi: #1}\else
  \providecommand{\doi}{doi: \begingroup \urlstyle{rm}\Url}\fi

\bibitem[Aguiar et~al.(2022)Aguiar, Santos, Watkinson, and
  Zhu]{aguiar2022learning}
H.~Aguiar, M.~Santos, P.~Watkinson, and T.~Zhu.
\newblock Learning of cluster-based feature importance for electronic health
  record time-series.
\newblock In \emph{International Conference on Machine Learning}, pages
  161--179. PMLR, 2022.

\bibitem[Bahadori et~al.(2015)Bahadori, Kale, Fan, and
  Liu]{bahadori2015functional}
M.~T. Bahadori, D.~Kale, Y.~Fan, and Y.~Liu.
\newblock Functional subspace clustering with application to time series.
\newblock In \emph{International Conference on Machine Learning}, pages
  228--237. PMLR, 2015.

\bibitem[Bastos et~al.(1993)Bastos, Sun, Wagner, Wu, and
  Knaus]{bastos1993glasgow}
P.~G. Bastos, X.~Sun, D.~P. Wagner, A.~W. Wu, and W.~A. Knaus.
\newblock Glasgow coma scale score in the evaluation of outcome in the
  intensive care unit: findings from the acute physiology and chronic health
  evaluation iii study.
\newblock \emph{Critical Care Medicine}, 21\penalty0 (10):\penalty0 1459--1465,
  1993.

\bibitem[Baytas et~al.(2017)Baytas, Xiao, Zhang, Wang, Jain, and
  Zhou]{baytas2017patient}
I.~M. Baytas, C.~Xiao, X.~Zhang, F.~Wang, A.~K. Jain, and J.~Zhou.
\newblock Patient subtyping via time-aware lstm networks.
\newblock In \emph{Proceedings of the 23rd ACM SIGKDD International Conference
  on Knowledge Discovery and Data Mining}, pages 65--74, 2017.

\bibitem[Blot et~al.(2009)Blot, Cankurtaran, Petrovic, Vandijck, Lizy,
  Decruyenaere, Danneels, Vandewoude, Piette, Vershraegen,
  et~al.]{blot2009epidemiology}
S.~Blot, M.~Cankurtaran, M.~Petrovic, D.~Vandijck, C.~Lizy, J.~Decruyenaere,
  C.~Danneels, K.~Vandewoude, A.~Piette, G.~Vershraegen, et~al.
\newblock Epidemiology and outcome of nosocomial bloodstream infection in
  elderly critically ill patients: a comparison between middle-aged, old, and
  very old patients.
\newblock \emph{Critical Care Medicine}, 37\penalty0 (5):\penalty0 1634--1641,
  2009.

\bibitem[Ceritli et~al.(2022)Ceritli, Creagh, and Clifton]{ceritli2022mixture}
T.~Ceritli, A.~P. Creagh, and D.~A. Clifton.
\newblock Mixture of input-output hidden markov models for heterogeneous
  disease progression modeling.
\newblock In \emph{Workshop on Healthcare AI and COVID-19}, pages 41--53. PMLR,
  2022.

\bibitem[Chen et~al.(2022)Chen, Krishnan, and Sontag]{chen2022clustering}
I.~Y. Chen, R.~G. Krishnan, and D.~Sontag.
\newblock Clustering interval-censored time-series for disease phenotyping.
\newblock In \emph{Proceedings of the AAAI Conference on Artificial
  Intelligence}, volume~36, pages 6211--6221, 2022.

\bibitem[Cho et~al.(2014)Cho, Van~Merri{\"e}nboer, Bahdanau, and
  Bengio]{cho2014properties}
K.~Cho, B.~Van~Merri{\"e}nboer, D.~Bahdanau, and Y.~Bengio.
\newblock On the properties of neural machine translation: Encoder-decoder
  approaches.
\newblock \emph{arXiv preprint arXiv:1409.1259}, 2014.

\bibitem[Coley et~al.(2011)Coley, Andrieu, Jaros, Weiner, Cedarbaum, and
  Vellas]{coley2011suitability}
N.~Coley, S.~Andrieu, M.~Jaros, M.~Weiner, J.~Cedarbaum, and B.~Vellas.
\newblock Suitability of the clinical dementia rating-sum of boxes as a single
  primary endpoint for alzheimer’s disease trials.
\newblock \emph{Alzheimer's \& Dementia}, 7\penalty0 (6):\penalty0 602--610,
  2011.

\bibitem[Delor et~al.(2013)Delor, Charoin, Gieschke, Retout, Jacqmin, and
  Initiative]{delor2013modeling}
I.~Delor, J.-E. Charoin, R.~Gieschke, S.~Retout, P.~Jacqmin, and A.~D.~N.
  Initiative.
\newblock Modeling alzheimer's disease progression using disease onset time and
  disease trajectory concepts applied to cdr-sob scores from adni.
\newblock \emph{CPT: Pharmacometrics \& Systems Pharmacology}, 2\penalty0
  (10):\penalty0 1--10, 2013.

\bibitem[Denny et~al.(2013)Denny, Bastarache, Ritchie, Carroll, Zink, Mosley,
  Field, Pulley, Ramirez, Bowton, et~al.]{denny2013systematic}
J.~C. Denny, L.~Bastarache, M.~D. Ritchie, R.~J. Carroll, R.~Zink, J.~D.
  Mosley, J.~R. Field, J.~M. Pulley, A.~H. Ramirez, E.~Bowton, et~al.
\newblock Systematic comparison of phenome-wide association study of electronic
  medical record data and genome-wide association study data.
\newblock \emph{Nature Biotechnology}, 31\penalty0 (12):\penalty0 1102--1111,
  2013.

\bibitem[Giannoula et~al.(2018)Giannoula, Gutierrez-Sacrist\'{i}an, Bravo,
  Sanz, and Furlong]{agiannoula2018clustering}
A.~Giannoula, A.~Gutierrez-Sacrist\'{i}an, A.~Bravo, F.~Sanz, and L.~I.
  Furlong.
\newblock Identifying temporal patterns in patient disease trajectories using
  dynamic time warping: A population-based study.
\newblock \emph{Scientific Reports}, 8\penalty0 (4216), 2018.

\bibitem[Goldberger et~al.(2000)Goldberger, Amaral, Glass, Hausdorff, Ivanov,
  Mark, Mietus, Moody, Peng, and Stanley]{goldberger2000physiobank}
A.~L. Goldberger, L.~A. Amaral, L.~Glass, J.~M. Hausdorff, P.~C. Ivanov, R.~G.
  Mark, J.~E. Mietus, G.~B. Moody, C.-K. Peng, and H.~E. Stanley.
\newblock Physiobank, physiotoolkit, and physionet: components of a new
  research resource for complex physiologic signals.
\newblock \emph{Circulation}, 101\penalty0 (23):\penalty0 e215--e220, 2000.

\bibitem[Haas et~al.(2017)Haas, Van~Dillen, de~Lange, Van~Dijk, and
  Hamaker]{haas2017outcome}
L.~E. Haas, L.~Van~Dillen, D.~de~Lange, D.~Van~Dijk, and M.~Hamaker.
\newblock Outcome of very old patients admitted to the icu for sepsis: a
  systematic review.
\newblock \emph{European Geriatric Medicine}, 8\penalty0 (5-6):\penalty0
  446--453, 2017.

\bibitem[Hayashi et~al.(2005)Hayashi, Mizuhara, and
  Suematsu]{hayashi2005embedding}
A.~Hayashi, Y.~Mizuhara, and N.~Suematsu.
\newblock Embedding time series data for classification.
\newblock In \emph{Machine Learning and Data Mining in Pattern Recognition: 4th
  International Conference, MLDM 2005, Leipzig, Germany, July 9-11, 2005.
  Proceedings 4}, pages 356--365. Springer, 2005.

\bibitem[Ho et~al.(2014)Ho, Ghosh, and Sun]{jho2014clustering}
J.~C. Ho, J.~Ghosh, and J.~Sun.
\newblock Marble: High-throughput phenotyping from electronic health records
  via sparse nonnegative tensor factorization.
\newblock In \emph{Proceedings of the ACM SIGKDD Conference on Knowledge
  Discovery and Data Mining}, 2014.

\bibitem[Hripcsak and Albers(2013)]{hripcsak2013phenotype}
G.~Hripcsak and D.~J. Albers.
\newblock Next-generation phenotyping of electronic health records.
\newblock \emph{Journal of the American Medical Informatics Association},
  20\penalty0 (1):\penalty0 117--121, 2013.

\bibitem[Kim et~al.(2020)Kim, Woo, Kim, Jang, Kim, Cho, Kim, Kim, Shin, Kim,
  et~al.]{kim2020disease}
K.~W. Kim, S.~Y. Woo, S.~Kim, H.~Jang, Y.~Kim, S.~H. Cho, S.~E. Kim, S.~J. Kim,
  B.-S. Shin, H.~J. Kim, et~al.
\newblock Disease progression modeling of alzheimer’s disease according to
  education level.
\newblock \emph{Scientific Reports}, 10\penalty0 (1):\penalty0 1--9, 2020.

\bibitem[Lee and van~der Schaar(2020)]{lee2020temporal}
C.~Lee and M.~van~der Schaar.
\newblock Temporal phenotyping using deep predictive clustering of disease
  progression.
\newblock In \emph{International Conference on Machine Learning}, pages
  5767--5777. PMLR, 2020.

\bibitem[Lee et~al.(2020)Lee, Rashbass, and Van~der Schaar]{lee2020outcome}
C.~Lee, J.~Rashbass, and M.~Van~der Schaar.
\newblock Outcome-oriented deep temporal phenotyping of disease progression.
\newblock \emph{IEEE Transactions on Biomedical Engineering}, 68\penalty0
  (8):\penalty0 2423--2434, 2020.

\bibitem[Lee et~al.(2022)Lee, Light, Saveliev, van~der Schaar, and
  Gnanapragasam]{lee2022developing}
C.~Lee, A.~Light, E.~S. Saveliev, M.~van~der Schaar, and V.~J. Gnanapragasam.
\newblock Developing machine learning algorithms for dynamic estimation of
  progression during active surveillance for prostate cancer.
\newblock \emph{npj Digital Medicine}, 5\penalty0 (1):\penalty0 110, 2022.

\bibitem[Lei et~al.(2019)Lei, Yi, Vaculin, Wu, and Dhillon]{lei2019similarity}
Q.~Lei, J.~Yi, R.~Vaculin, L.~Wu, and I.~S. Dhillon.
\newblock Similarity preserving representation learning for time series
  clustering.
\newblock In \emph{{Proceedings of the International Joint Conference on
  Artificial Intelligence}}, 2019.

\bibitem[Leitgeb et~al.(2013)Leitgeb, Mauritz, Brazinova, Majdan, Janciak,
  Wilbacher, and Rusnak]{leitgeb2013glasgow}
J.~Leitgeb, W.~Mauritz, A.~Brazinova, M.~Majdan, I.~Janciak, I.~Wilbacher, and
  M.~Rusnak.
\newblock Glasgow coma scale score at intensive care unit discharge predicts
  the 1-year outcome of patients with severe traumatic brain injury.
\newblock \emph{European Journal of Trauma and Emergency Surgery}, 39\penalty0
  (3):\penalty0 285--292, 2013.

\bibitem[Ma et~al.(2019)Ma, Zheng, Li, and Cottrell]{ma2019learning}
Q.~Ma, J.~Zheng, S.~Li, and G.~W. Cottrell.
\newblock Learning representations for time series clustering.
\newblock \emph{Advances in Neural Information Processing Systems}, 32, 2019.

\bibitem[O’Bryant et~al.(2008)O’Bryant, Waring, Cullum, Hall, Lacritz,
  Massman, Lupo, Reisch, Doody, Consortium, et~al.]{o2008staging}
S.~E. O’Bryant, S.~C. Waring, C.~M. Cullum, J.~Hall, L.~Lacritz, P.~J.
  Massman, P.~J. Lupo, J.~S. Reisch, R.~Doody, T.~A.~R. Consortium, et~al.
\newblock Staging dementia using clinical dementia rating scale sum of boxes
  scores: a texas alzheimer's research consortium study.
\newblock \emph{Archives of Neurology}, 65\penalty0 (8):\penalty0 1091--1095,
  2008.

\bibitem[Ramos et~al.(2019)Ramos, Smith, McKone, Pilewski, Lucy, Hempstead,
  Tallarico, Faro, Rosenbluth, Gray, et~al.]{ramos2019lung}
K.~J. Ramos, P.~J. Smith, E.~F. McKone, J.~M. Pilewski, A.~Lucy, S.~E.
  Hempstead, E.~Tallarico, A.~Faro, D.~B. Rosenbluth, A.~L. Gray, et~al.
\newblock Lung transplant referral for individuals with cystic fibrosis: Cystic
  fibrosis foundation consensus guidelines.
\newblock \emph{Journal of Cystic Fibrosis}, 18\penalty0 (3):\penalty0
  321--333, 2019.

\bibitem[Rao et~al.(2022)Rao, Ganaraja, Murlimanju, Joy, Krishnamurthy, and
  Agrawal]{rao2022hippocampus}
Y.~L. Rao, B.~Ganaraja, B.~Murlimanju, T.~Joy, A.~Krishnamurthy, and
  A.~Agrawal.
\newblock Hippocampus and its involvement in alzheimer’s disease: a review.
\newblock \emph{3 Biotech}, 12\penalty0 (2):\penalty0 55, 2022.

\bibitem[Richesson et~al.(2016)Richesson, Sun, Pathak, Kho, and
  Denny]{richesson2016clinical}
R.~L. Richesson, J.~Sun, J.~Pathak, A.~N. Kho, and J.~C. Denny.
\newblock Clinical phenotyping in selected national networks: demonstrating the
  need for high-throughput, portable, and computational methods.
\newblock \emph{Artificial Intelligence in Medicine}, 71:\penalty0 57--61,
  2016.

\bibitem[Rousseeuw(1987)]{rousseeuw1987silhouettes}
P.~J. Rousseeuw.
\newblock Silhouettes: a graphical aid to the interpretation and validation of
  cluster analysis.
\newblock \emph{Journal of Computational and Applied Mathematics}, 20:\penalty0
  53--65, 1987.

\bibitem[Schulam et~al.(2015)Schulam, Wigley, and Saria]{schulam2015clustering}
P.~Schulam, F.~Wigley, and S.~Saria.
\newblock Clustering longitudinal clinical marker trajectories from electronic
  health data: Applications to phenotyping and endotype discovery.
\newblock In \emph{Proceedings of the AAAI Conference on Artificial
  Intelligence}, volume~29, 2015.

\bibitem[Steinley(2004)]{steinley2004properties}
D.~Steinley.
\newblock Properties of the hubert-arable adjusted rand index.
\newblock \emph{Psychological Methods}, 9\penalty0 (3):\penalty0 386, 2004.

\bibitem[Vinh et~al.(2009)Vinh, Epps, and Bailey]{vinh2009information}
N.~X. Vinh, J.~Epps, and J.~Bailey.
\newblock Information theoretic measures for clusterings comparison: is a
  correction for chance necessary?
\newblock In \emph{Proceedings of the 26th Annual International Conference on
  Machine Learning}, pages 1073--1080, 2009.

\bibitem[Xie et~al.(2016)Xie, Girshick, and Farhadi]{xie2016unsupervised}
J.~Xie, R.~Girshick, and A.~Farhadi.
\newblock Unsupervised deep embedding for clustering analysis.
\newblock In \emph{International Conference on Machine Learning}, pages
  478--487. PMLR, 2016.

\bibitem[Yamazaki et~al.(2019)Yamazaki, Zhao, Caulfield, Liu, and
  Bu]{yamazaki2019apolipoprotein}
Y.~Yamazaki, N.~Zhao, T.~R. Caulfield, C.-C. Liu, and G.~Bu.
\newblock Apolipoprotein e and alzheimer disease: pathobiology and targeting
  strategies.
\newblock \emph{Nature Reviews Neurology}, 15\penalty0 (9):\penalty0 501--518,
  2019.

\bibitem[Yang et~al.(2017)Yang, Fu, Sidiropoulos, and Hong]{yang2017towards}
B.~Yang, X.~Fu, N.~D. Sidiropoulos, and M.~Hong.
\newblock Towards k-means-friendly spaces: Simultaneous deep learning and
  clustering.
\newblock In \emph{International Conference on Machine Learning}, pages
  3861--3870. PMLR, 2017.

\bibitem[Zhang et~al.(2019)Zhang, Chou, Liang, Xiao, Zhao, Sarv, Henchcliffe,
  and Wang]{xzhang2019clustering}
X.~Zhang, J.~Chou, J.~Liang, C.~Xiao, Y.~Zhao, H.~Sarv, C.~Henchcliffe, and
  F.~Wang.
\newblock Data-driven subtyping of parkinson’s disease using longitudinal
  clinical records: A cohort study.
\newblock \emph{Scientific Reports}, 9\penalty0 (797), 2019.

\end{thebibliography}
\bibliographystyle{abbrvnat}

\newpage
\onecolumn
\appendix

\section*{Appendix}
\renewcommand{\thetable}{\thesection.\arabic{table}}
\setcounter{table}{0}
\renewcommand{\thefigure}{\thesection.\arabic{figure}}
\setcounter{figure}{0}
\renewcommand\thesubfigure{\alph{subfigure}}
\renewcommand{\thealgorithm}{\thesection.\arabic{algorithm}}
\setcounter{algorithm}{0}
The appendix is organized in the following structure.
\begin{enumerate}[label=\textbf{\Alph*}]
    \item Detailed discussion of the Laplace encoder.
    \item Proof of Proposition 1 and relevant discussions.
    \item The graph-constrained $K$-means algorithm
    \item Experiment setup.
    \item Hyperparameter Selection.
    \item Complete benchmark results.
    \item Additional analyses of results obtained on the two real-world datasets.
\end{enumerate}

A summary of major notations used in this paper is provided below.
{
\setlength{\nomlabelwidth}{2cm}
\nomenclature[01]{$\vx$}{Continuous-time disease trajectory of a patient}
\nomenclature[02]{$\vy$}{Label vector indicating clinical status of a patient}
\nomenclature[03]{$\vt$}{A vector of time stamps}
\nomenclature[04]{$\vz$}{A vector of latent variables}
\nomenclature[05]{$\vw$}{A vector of Laplace embedding}
\nomenclature[06]{$\mX$}{Discrete-time observation to disease trajectory $\vx(t)$}
\nomenclature[07]{$\Phi$}{A connected set of patient trajectories which represents a temporal pattern}
\nomenclature[08]{$g(\vx)$}{Vector-valued function that describes the conditional distribution $p(\vy\vert\vx)$}
\nomenclature[09]{$\mathrm{d}_y(\cdot,\cdot)$}{Distance metric of two label distributions}
\nomenclature[10]{$\Gamma(\vx^1\rightarrow\vx^2)$}{A translation from trajectory $\vx^1$ to $\vx^2$}
\nomenclature[11]{$\gamma(\vz^1\rightarrow\vz^2)$}{A translation from latent representation $\vz^1$ to $\vz^2$}
\nomenclature[12]{$\mathrm{d}_\Gamma(\vx^1,\vx^2)$}{Path-based similarity score between trajectories $\vx^1$ and $\vx^2$}
\nomenclature[13]{$\mathrm{d}_\gamma(\vz^1,\vz^2)$}{Proxy of path-based similarity score $\mathrm{d}_\Gamma(\vx^1,\vx^2)$ in latent space}
\nomenclature[14]{$\mS$}{A distance matrix of path-based similarity score between samples in a dataset}
\nomenclature[15]{$\gG_\delta$}{A graph generated from matrix $\mS$ with threshold $\delta$}
\nomenclature[16]{$K$}{Number of clusters}
\nomenclature[17]{$\mathcal{C}$}{A set of $K$ clusters}
\nomenclature[18]{$f_L$}{A Laplace encoder}
\nomenclature[19]{$f_E$}{A composite encoder with feature-wise Laplace encoders}
\nomenclature[20]{$f_P$}{A predictor for label distribution}
\renewcommand{\nomname}{NOMENCLATURE}
\printnomenclature
}

\textbf{Code Availability.}\quad
The source code of {\ours} can be found in the two GitHub repositories listed below:
\begin{itemize}
    \item The van der Schaar lab repo: \url{https://github.com/vanderschaarlab/tphenotype}
    \item The author's personal repo: \url{https://github.com/yvchao/tphenotype}
\end{itemize}

\renewcommand{\thesection}{A}
\setcounter{table}{0}
\setcounter{figure}{0}
\setcounter{algorithm}{0}
\section{ANALYSIS OF THE LAPLACE ENCODER}
\subsection{Implementation Details}
The proposed Laplace encoder is implemented with a RNN-based neural network $f_L$ parameterized by $\theta_L$.
As shown in \Figref{fig:encoder}, given discrete time-series of a one-dimension trajectory $x(t)$, the Laplace encoder first generates a summary of time-series $x(\vt)$ via the RNN.
With the summary as input, the MLP outputs a representation ${\vw}\in \mathbb{C}^{n(d+1)}$.
Elements in ${\vw}$ can be divided into two groups: poles and coefficients, which are further used to construct a function in the frequency domain, $F_\vw(s)$, as defined in  (\ref{eq:laplace-transform}). 
Changing the order of poles (and associated coefficients) in ${\vw}$ has no effect on $F_{{\vw}}(s)$ since it is permutation-invariant to the poles in ${\vw}$.
As discussed in the main manuscript and in the next paragraph, we impose a lexical order on poles in the MLP output ${\vw}$ to make it a unique representation of $F_{{\vw}}(s)$.
The trajectory $x(t)$ can be reconstructed as $\hat{x}(t)$ through the inverse Laplace transform (\ref{eq:inverse-laplace-transform}) on $F_\vw (s)$. Here, the reconstruction $\hat{x}(t)$ is a function and its value can be evaluated everywhere in $t\in [0,1]$. This allows us to compare input time-series with variable-length and irregularly-sampled observations in a unified latent space. 
For the sake of convenience, we denote with $\mathcal{L}^{-1}(\vw) =\hat{x}(t)= \mathcal{L}^{-1} [F_\vw (s)](t)$ the transform that maps embedding $\vw$ to its time-domain reconstruction $\hat{x}(t)$.

\begin{figure}[!h]
	\begin{center}
		\centering
		{\includegraphics[width=0.95\textwidth]{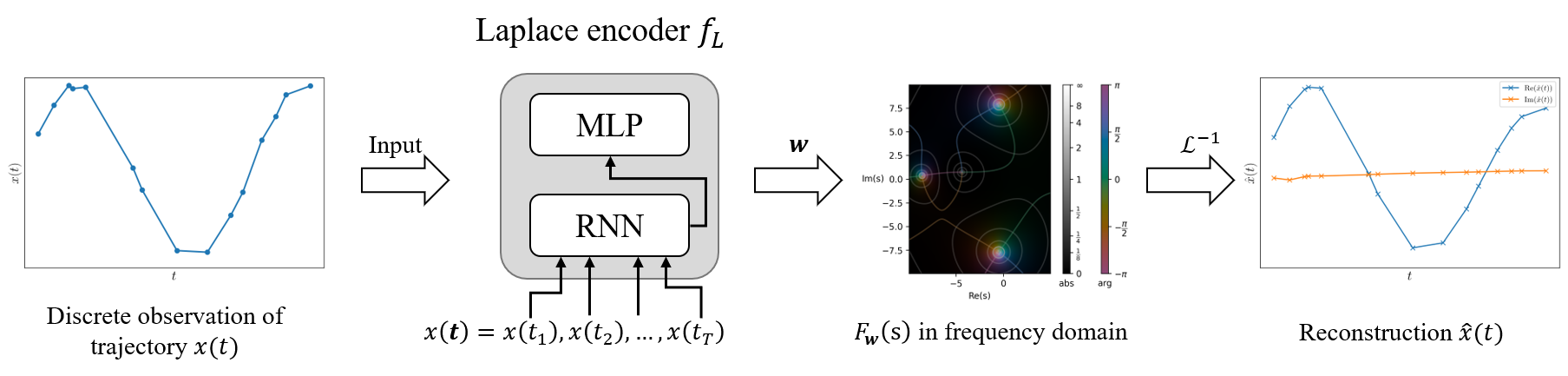}}
		\caption{{Laplace Encoder.}
		\small} 
		\label{fig:encoder}
	\end{center}
\end{figure}

\textbf{Robust Lexical Order of Poles.}\quad
Due to the summation in (\ref{eq:laplace-transform}), $F_\vw(s)$ is permutation-equivariant with respect to the poles in $\vw$. 
Thus, we impose a lexical order ($p_m \leq p_{m+1}$ for $m=1,\dots, n-1$) on the poles to obtain a unique Laplace embedding $\vw$ as discussed in the manuscript. 
To guarantee this property, we transform the unordered representation (output of the MLP in Laplace encoder) into the final unique Laplace embedding $\vw$ by sorting the poles (together with their associated coefficients) in a lexical order. 
To achieve a stable ordering that is robust to inevitable noise in ${\vw}$, we encourage any pair of two poles to be sufficiently different to avoid abrupt changes in their order.  
Hence, given two poles $p_m,p_l$, we say $p_{m} \leq p_{l}$ if and only if $(\mathrm{Re}(p_m) < \mathrm{Re}(p_{l}) ) \wedge (|\mathrm{Re}(p_m) - \mathrm{Re}(p_{l})| > \delta_{pole})$ or $(|\mathrm{Re}(p_m) - \mathrm{Re}(p_{l})| \leq \delta_{pole}) \wedge  (\mathrm{Im}(p_m)\leq \mathrm{Im}(p_{l}))$, and $p_{m} > p_{l}$ otherwise, where $\delta_{pole} \geq 0$ is a threshold that controls the robustness of the lexical order. The best threshold $\delta_{pole}$ is search as a hyperparameter in our experiment.

\textbf{Ranges of Poles and Coefficients.}\quad
Each pole $p_m$ in embedding $\vw$ is located on the complex plane $\mathbb{C}$.
The real part $\mathrm{Re}(p_m)$ indicates the increase or decay speed of the corresponding component ($e^{\mathrm{Re}(p_m) t}$) in the time-domain reconstruction $\hat{x}(t)=\mathcal{L}^{-1}(\vw)$.
Too large or small value of $\mathrm{Re}(p_m)$ leads to unrealistic signals.
In the meantime, The imaginary part $\mathrm{Im}(p_m)$ represents the frequency of oscillations in the related component ($\cos(\mathrm{Im}(p_m) t)+j \sin(\mathrm{Im}(p_m) t), j^2=-1$) in reconstruction $\hat{x}(t)$.
Very high-frequency oscillation in the input time-series $x(\vt)$ are usually caused by random noise and should be discarded in reconstruction $\hat{x}(t)$.
In our experiment, we limit the range of poles to the area of $\{p~\big|~|\mathrm{Re}(p)|\leq r_{max}, |\mathrm{Im}(p)|\leq freq_{max}\}$, where $r_{max}$ limits that maximum increase or decrease speed of signals in reconstruction $\hat{x}$, $freq_{max}$ is the maximum allowed frequency such that high-frequency signals above $freq_{max}$ are considered as a noise component in time-series $x(\vt)$ and, thus, discarded when constructing $F_\vw(s)$.  
In our experiments, we set $r_{max}=10$ and $freq_{max}=20Hz$.
Similarly, the coefficient $c_{m,l}$ in embedding $\vw$ is limited to a square area of $\{c~\big|~|\mathrm{Re}(c)|\leq c_{max}, |\mathrm{Im}(c)|\leq c_{max}\}$.
We set $c_{max}=5$ which is sufficient for normalized time-series (via min-max or normal scaling). 
The range of poles and coefficients in $\vw$ can be adjusted accordingly based on needs in practical application scenarios.
When fed into the predictor network $f_P$, the poles and coefficients in embedding $\vw$ are normalized by the corresponding maximum allowed values to facilitate the learning process. 

\textbf{Embedding of Static Features.}\quad
In order to improve computation efficiency, when the $d$-th feature dimension $x_d$ of trajectory $\vx$ is known to be constant over time, i.e., $x_d(t)\equiv x_d(0)$, instead of training a Laplace encoder, the static value $x_d(0)$ is directly used to represent $x_d(t)$, and the $d$-th component $\vw_d$ in latent variable $\vz$ is replaced by $x_d(0)$.

\textbf{Regularization Terms.}\quad
Apart from the lexical order imposed on the embedding $\vw$, we further introduce three regularization terms that encourage the Laplace encoder to provide a unique and consistent Laplace representation given an input time-series. 
These regularization terms are combined into the second term of $\mathcal{L}_\mathrm{unique}$ in (\ref{eq:loss-encoder}); we will describe each in turn.

The first regularizer, $l_\mathrm{sep}$, penalizes the case where two poles in embedding $\vw$ are nearly identical -- that is, $p_m$ and $p_l$ are considered as an identical pole when $|p_m-p_l|\leq \delta_{pole}$ -- based on the following hinge loss:
\begin{equation}
    l_\mathrm{sep}(\hat{x}(\vt)) = \sum_{m\neq l}  \max(0, \delta_{pole} - |p_m-p_l|).
\end{equation}
Here, $p_m$ and $p_l$ are two poles in the associated embedding $\vw$ given the input time-series $x(\vt)$, i.e., $\vw=f_{L}(x(\vt))$, and the threshold $\delta_{pole}>0$ for robust pole sorting is reused here as a pole separation threshold. 

The second regularizer, $l_\mathrm{real}$, ensures that the reconstructed trajectory $\hat{x}(t)$ is real-valued on $[0,1]$ by suppressing the imaginary part of the reconstructed trajectory $\hat{x}(t)$ via the following loss:  
\begin{equation}
    l_\mathrm{real}(\hat{x}(\vt)) = \frac{1}{T} \| \mathrm{Im}(\hat{x}(\vt)) \|_2^2,
\end{equation}
where $\vt=[t_1,\dots,t_{T}]^\top$ includes time stamps randomly sampled over $t_j \in [0,1]$ for $j=1,\dots,T$.
Specifically, $t_j = \mathrm{clamp}(\frac{j}{T} + \frac{1}{2T}\varepsilon ,\min=0,\max=1), \varepsilon\sim \mathrm{Normal}(0,1)$.

The last regularizer, $l_\mathrm{distinct}$, encourages that no two distinct Laplace embeddings generate the same trajectory based on the following loss:
\begin{equation}
    l_\mathrm{distinct}(\hat{x}^i(\vt),\hat{x}^j(\vt)) =  \Vert \vw^i-\vw^j\Vert_2^2e ^{-\|\hat{x}^i(\vt)-\hat{x}^j(\vt)\|_2^2},
\end{equation}
where the radial basis similarity function $e ^{-\|\hat{x}^i(\vt)-\hat{x}^j(\vt)\|_2^2}$ is used to discover similar trajectories, $\vw^i$ and $\vw^j$ are embeddings of input time-series while $\hat{x}^i(\vt)$ and $\hat{x}^j(\vt)$ are their time-domain reconstructions.
Since the input time-series may be of different lengths and sampling intervals, we use the reconstructed trajectories for pair-wise comparison between time-series here.

Overall, we construct $\mathcal{L}_\mathrm{unique}(\theta_L)$ as a combination of the three regularization terms introduced above: 
\begin{equation}
    \mathcal{L}_\mathrm{unique}(\theta_L) = \sum_{d=1}^{\mathrm{dim}_x} \left(  \frac{1}{N} \sum_{i}  l_\mathrm{sep}(\hat{x}^i_{d}(\vt)) + \frac{\alpha_1}{\alpha} l_\mathrm{real}(\hat{x}^i_d(\vt)) + \frac{\alpha_2}{\alpha} \frac{1}{N(N-1)} \sum_{i\neq j} l_\mathrm{distinct}(\hat{x}^i_{d}(\vt),\hat{x}^j_{d}(\vt)) \right),
\end{equation}
where $\alpha$ is the coefficient for $\mathcal{L}_\mathrm{unique}(\theta_L)$ in  (\ref{eq:loss-encoder}), $\alpha_1$ and $\alpha_2$ are balancing coefficients that trade-off different uniqueness properties in the Laplace encoder. 
In the experiment, due to the high computational complexity, the last term $l_\mathrm{distinct}$ is only evaluated on a subset of 10 randomly selected time-series in each training batch.
In addition, since $l_\mathrm{distinct}$ relies on the reconstructed time-series $\hat{x}(t)$ which may be inaccurate in the beginning of training,  we fix $\alpha_2$ to $0.01$ such that it majorly takes effect after the reconstruction error is small enough.
 
\subsection{Quantitative Analysis}
\textbf{Comparison with Regular Auto-encoder.}\quad
We provide a toy example to demonstrate the advantage of our proposed Laplace encoder over regular auto-encoders in time-series reconstruction.
A Laplace encoder composed of a 1-layer GRU \citep{cho2014properties}  and a 1-layer MLP with 10 hidden units in each layer is considered in the following discussion.
Other parameters of the Laplace encoder are set as $n=4, d=1,\alpha=1.0, \alpha_1 = 0.1, \alpha_2=0.01, \delta_{pole}=1.0$.
A regular time-series auto-encoder is used for comparison.
The auto-encoder has a 1-layer GRU network as the encoder.
The decoder contains a 1-layer MLP on top of another 1-layer GRU.
Each layer in the auto-encoder includes 10 hidden units.
The auto-encoder maps the input time-series to a latent variable. 
Then, the latent variable is provided to the decoder network for reconstruction of the entire time-series.

Consider a toy dataset with $N=1000$ irregularly sampled time-series in $t\in[0,1]$. 
Each sample contains $T=15$ observations from one of the following four types of trajectories:
\begin{itemize}
    \item {Type 1: $x(t) = \cos(2\pi (t-\phi))$.}
    \item {Type 2: $x(t) = \cos(\pi (t-\phi))$.}
    \item {Type 3: $x(t) = \sin(\pi (t-\phi))$.}
    \item {Type 4: $x(t) = \sin(2\pi (t-\phi))$.}
\end{itemize}
Delay term $\phi\sim \mathrm{Exp}(\frac{1}{2})$. Gaussian noise sampled from $\mathrm{Normal(0,0.03^2)}$ is independently introduced to the observations at different time points.
The mean squared error (MSE) in time-series reconstruction of the considered Laplace encoder and auto-encoder network is evaluated over 5 random splits of the toy dataset with the train/validation/test ratio of 64/16/20.
Our proposed Laplace encoder achieves the best performance of $\mathrm{MSE} = 0.039\pm0.008$.
The auto-encoder has a much higher reconstruction error of $\mathrm{MSE} = 0.108\pm 0.019$.
Comparison of typical reconstruction outcomes of the Laplace encoder and the auto-encoder is illustrated in \Figref{fig:encoder-comparison}.

\begin{figure}[h]
	\begin{center}
	\begin{subfigure}[b]{0.4\textwidth}
	\centering
    \includegraphics[width=\textwidth]{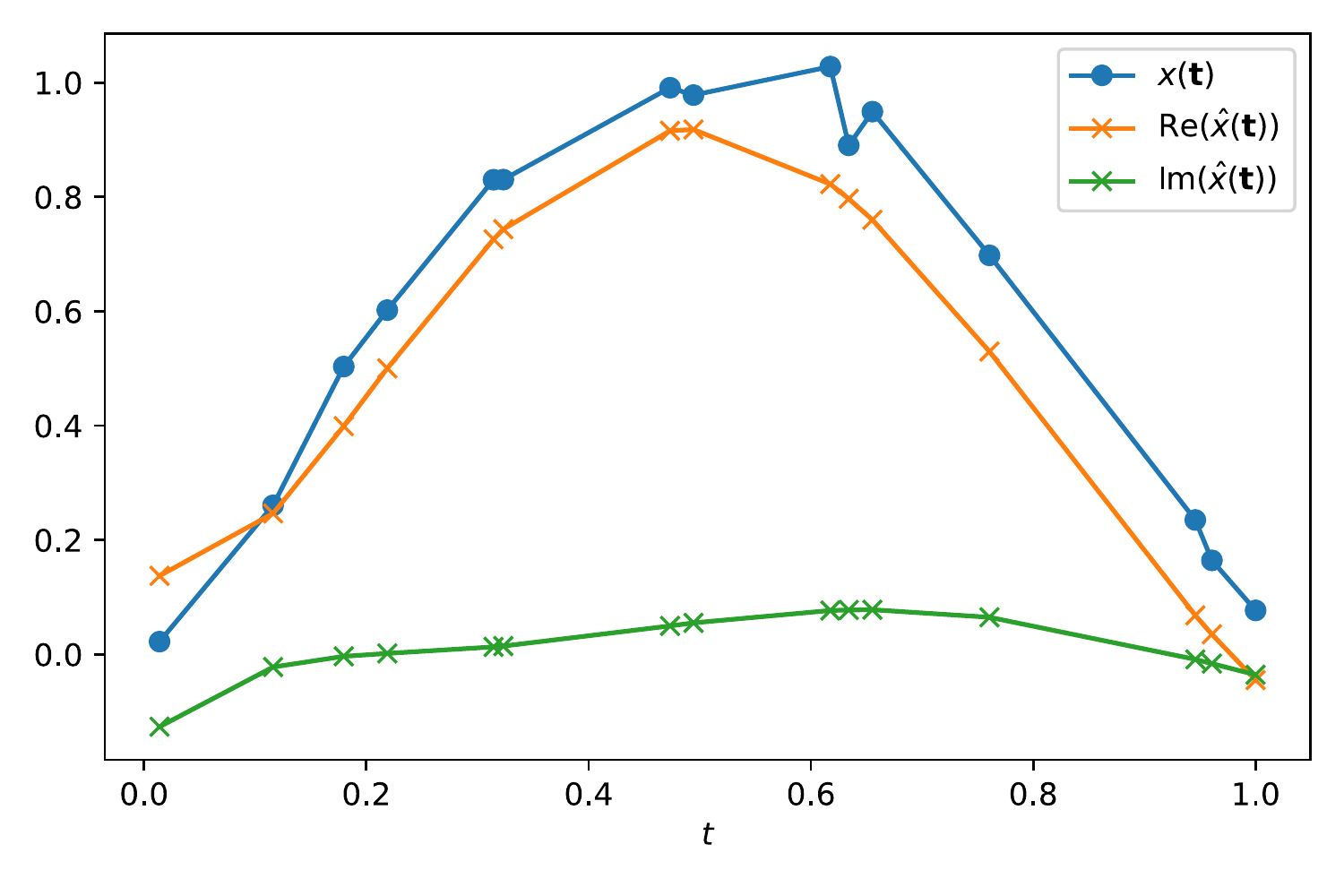}
    \caption{Trajectory reconstruction via Laplace encoder.}
    \label{fig:laplace-encoder-example}
    \end{subfigure}
    ~
    \begin{subfigure}[b]{0.4\textwidth}
	\centering
    \includegraphics[width=\textwidth]{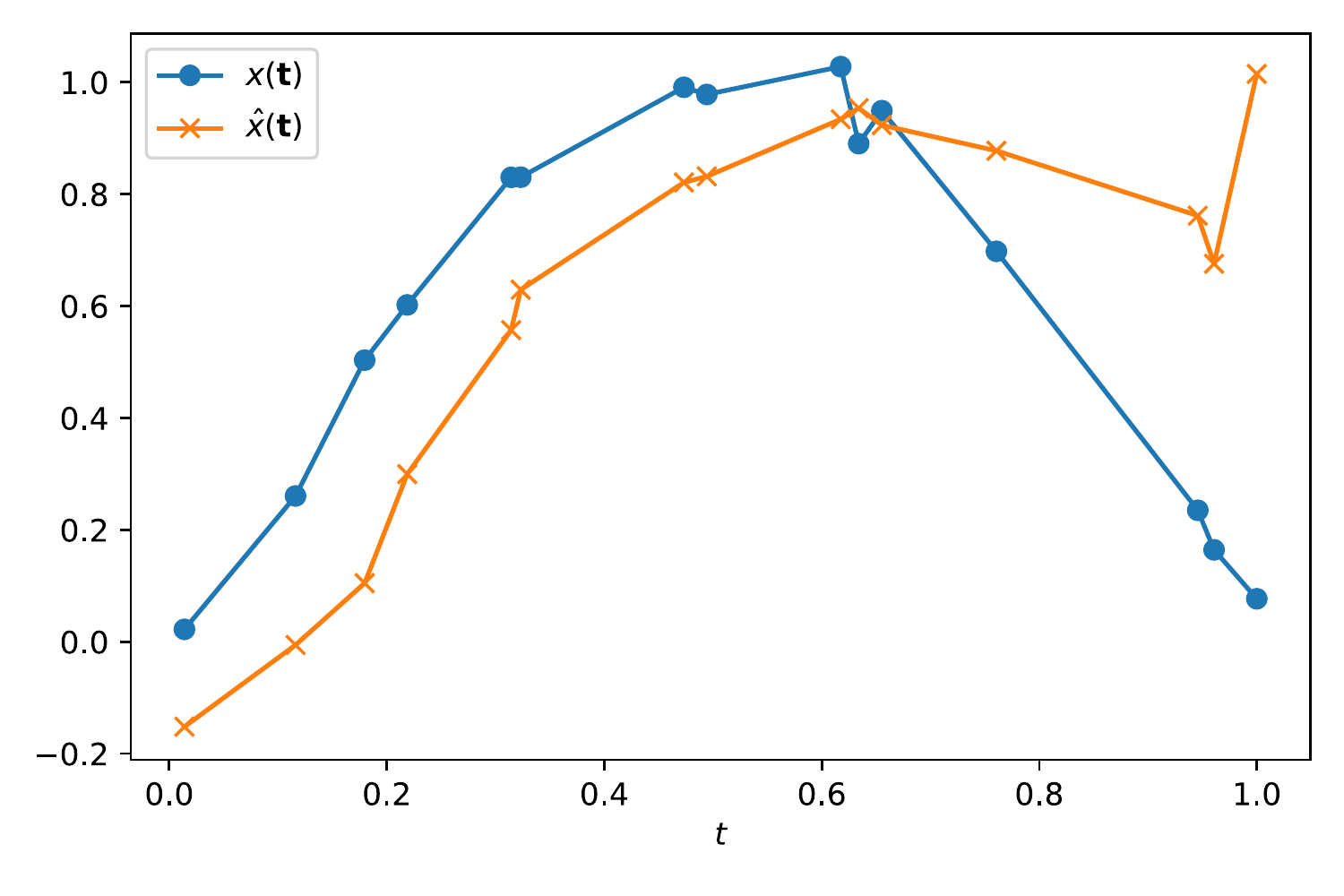}
    \caption{Trajectory reconstruction via Auto-encoder.}
    \label{fig:autoencoder-example}
    \end{subfigure}
    \caption{Comparison of Time-series Reconstruction Outcomes of Laplace Encoder and Auto-encoder.} 
		\label{fig:encoder-comparison}
	\end{center}
\end{figure}

\begin{figure}[!h]
	\begin{center}
	\begin{subfigure}[b]{0.3\textwidth}
	\centering
    \includegraphics[width=\textwidth]{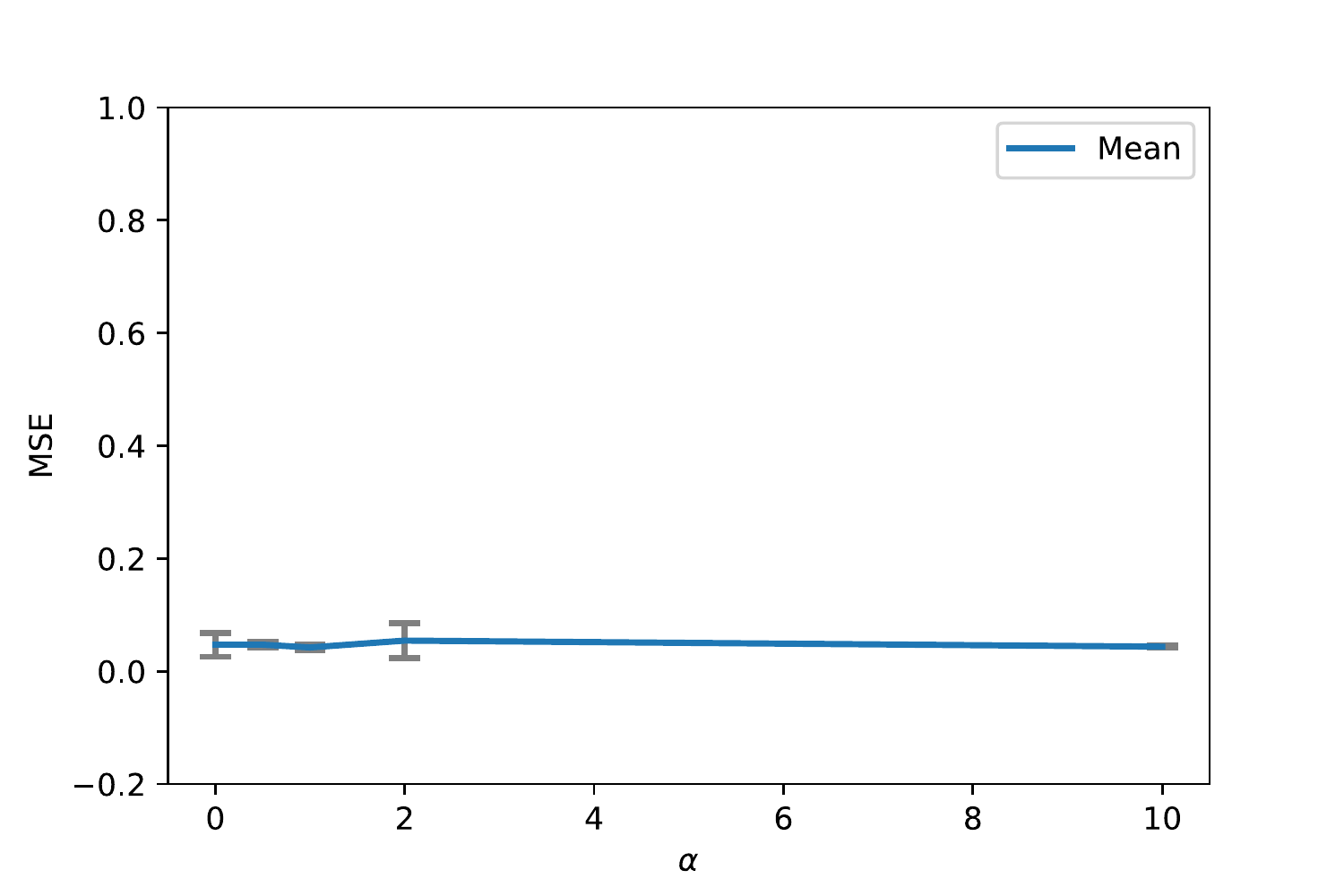}
    \caption{Impact of $\alpha$.}
    \end{subfigure}
    ~
    \begin{subfigure}[b]{0.3\textwidth}
	\centering
    \includegraphics[width=\textwidth]{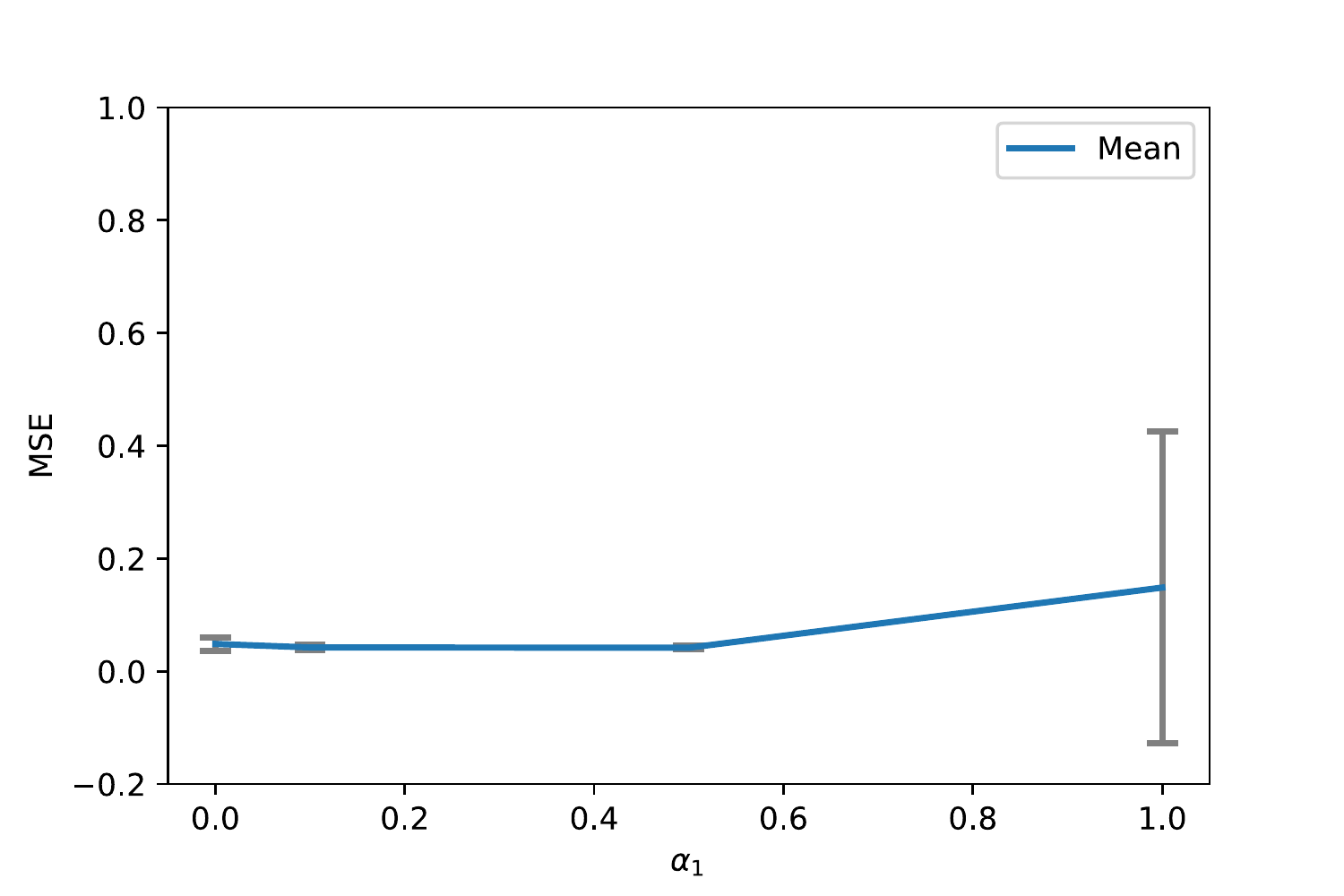}
    \caption{Impact of $\alpha_1$.}
    \end{subfigure}
    ~
    \begin{subfigure}[b]{0.3\textwidth}
	\centering
    \includegraphics[width=\textwidth]{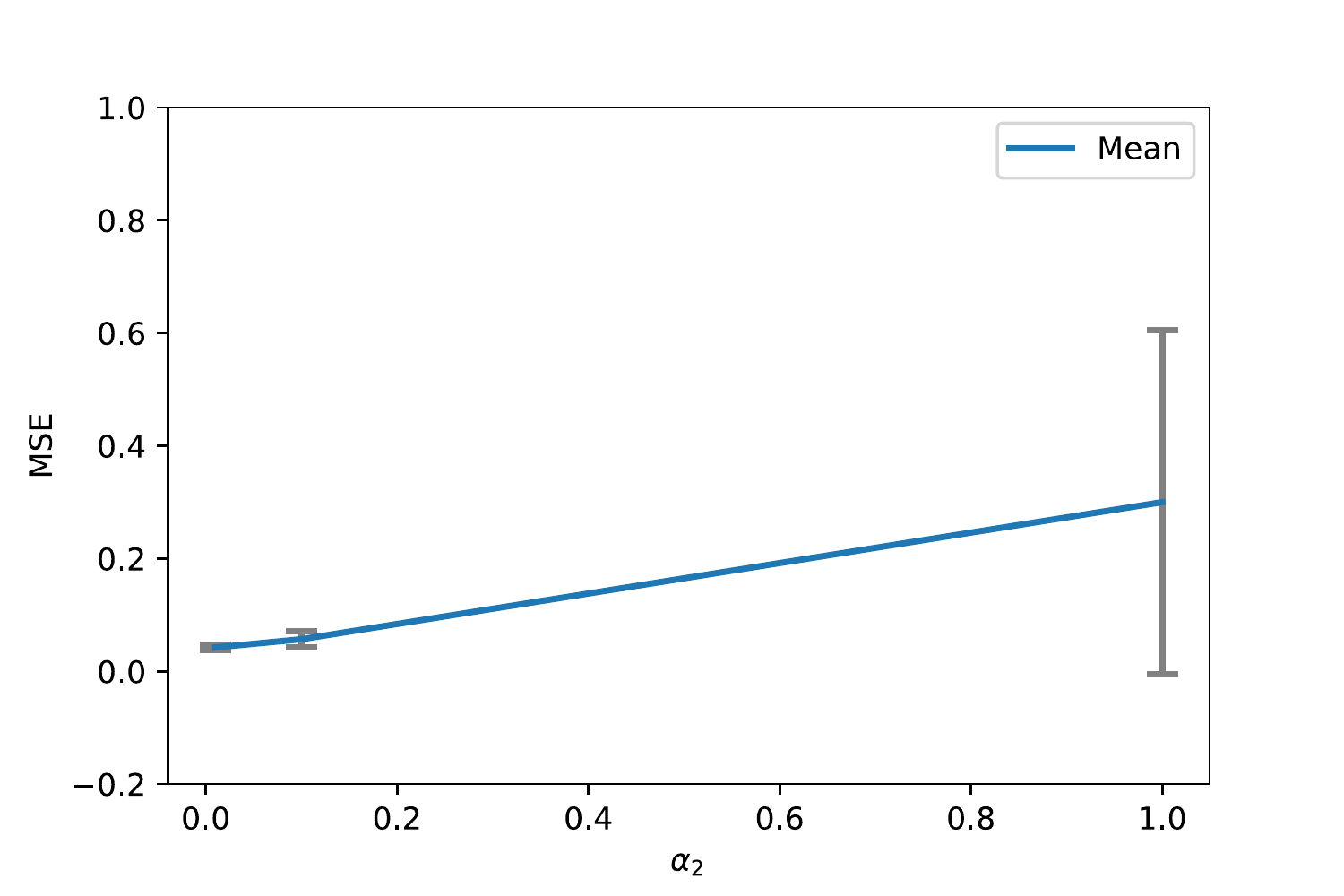}
    \caption{Impact of $\alpha_2$.}
    \label{fig:impact-alpha-2}
    \end{subfigure}
    \\
    \begin{subfigure}[b]{0.3\textwidth}
	\centering
    \includegraphics[width=\textwidth]{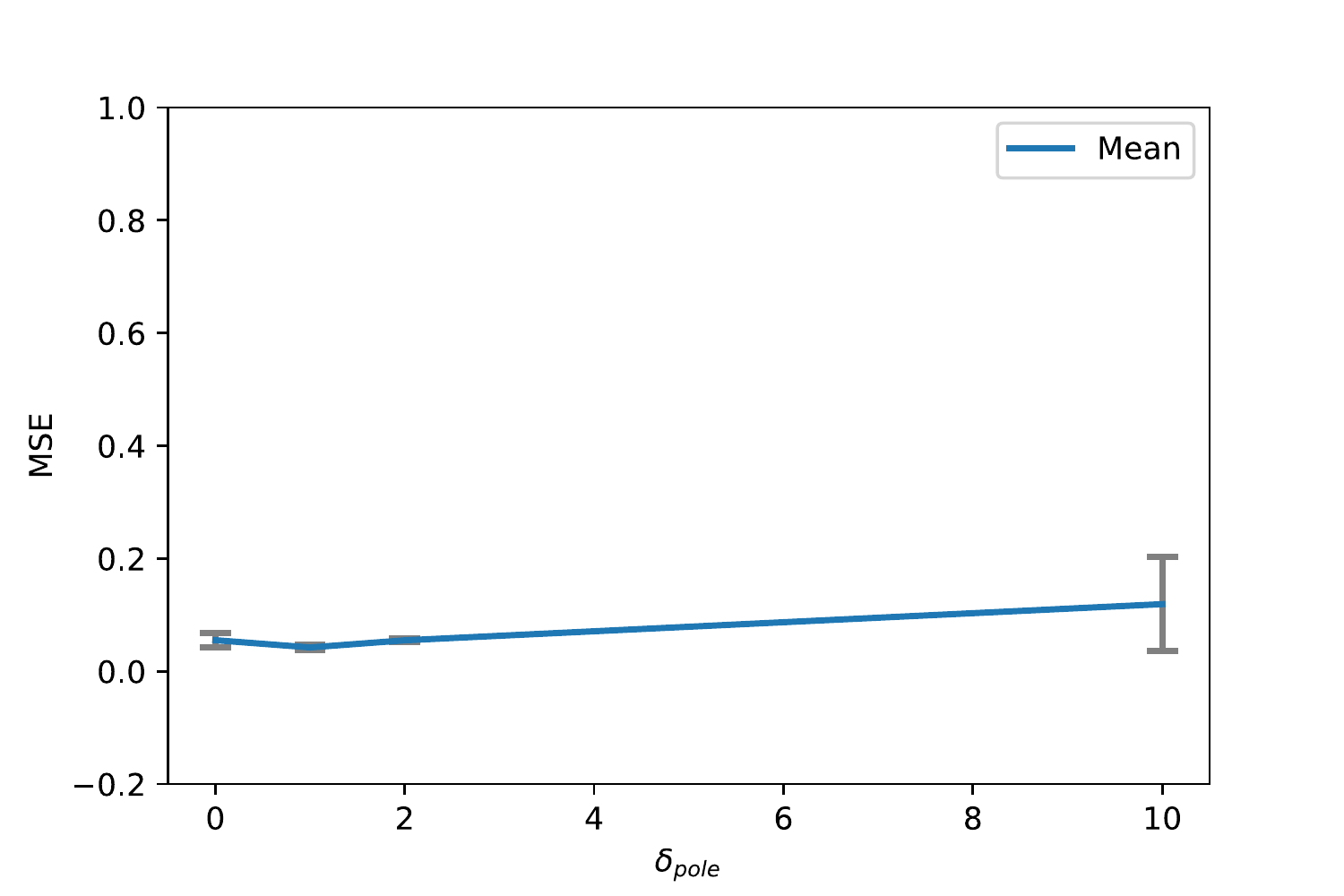}
    \caption{Impact of $\delta_{pole}$.}
    \end{subfigure}
    \begin{subfigure}[b]{0.3\textwidth}
	\centering
    \includegraphics[width=\textwidth]{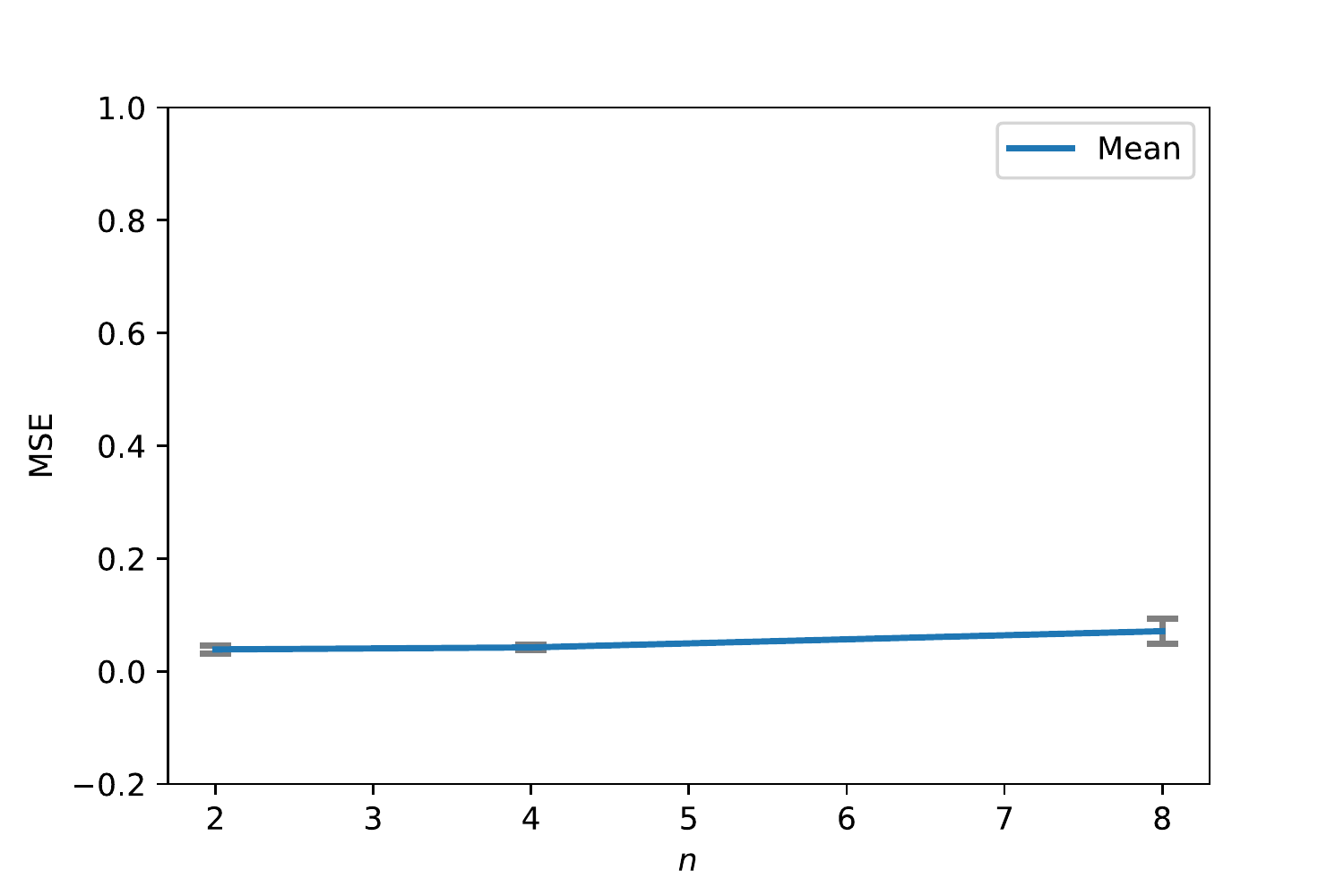}
    \caption{Impact of pole number $n$.}
    \end{subfigure}
    \begin{subfigure}[b]{0.3\textwidth}
	\centering
    \includegraphics[width=\textwidth]{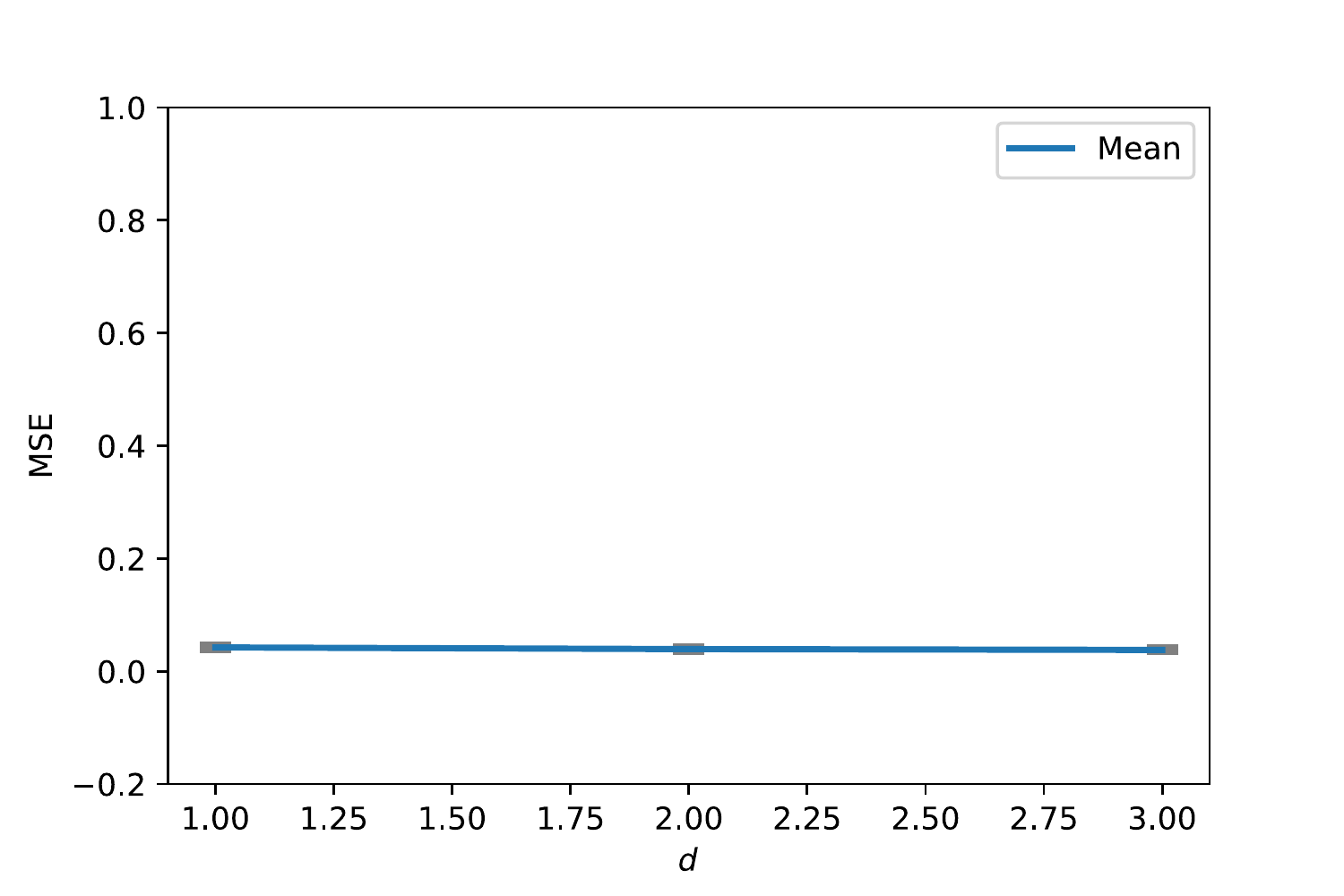}
    \caption{Impact of maximum pole degree $d$.}
    \end{subfigure}
    \caption{Sensitivity of Laplace Encoder with Respect to Different Hyperparameters. \small Error bars are calculated via evaluation on 3 random splits of the toy dataset.} 
		\label{fig:sensitivity-analysis}
	\end{center}
\end{figure}

\textbf{Sensitivity Analysis.}\quad
We further conduct a sensitivity analysis of the Laplace encoder $f_L$ under different hyperparameters on the toy dataset.
The default hyperparameters are set as $n=4, d=1,\alpha=1.0, \alpha_1 = 0.1, \alpha_2=0.01, \delta_{pole}=1.0$.
To evaluate the impact of individual hyperparameter on the Laplace encoder, in each test, we only alter the value of one hyperparameter and keep other hyperparameters the same as default setting.
The parameter sensitivity is measured via the reconstruction error (MSE), and the sensitivity test result is given in \Figref{fig:sensitivity-analysis}.

It can be found that our proposed Laplace encoder $f_L$ has relatively stable time-series reconstruction performance under different hyperparameters.
As mentioned earlier, the regularizer $l_\mathrm{ditinct}$ may generate wrong gradients in the beginning of training due to the large reconstruction error.
The increased MSE for larger $\alpha_2$ in \Figref{fig:impact-alpha-2} is within expectation, and we choose to set $\alpha_2$ to 0.01 such that it only takes effect when the reconstruction error is small enough.

In addition, the effect of pole separation threshold $\delta_{pole}$ on the Laplace embedding is illustrated in \Figref{fig:pole-separation}.
When $\delta_{pole}=0.0$, the order of poles in Laplace embedding $\vw$ can easily be affected by random noise in input time-series, which makes it difficult to ensure the uniqueness of $\vw$.
In contrast, setting $\delta_{pole}=1.0$ effectively improves the representations learned by the Laplace encoder, and different components in the Laplace transform $F_{\vw}(s)$ are clearly represented by distinct poles (marked with different colors).

\begin{figure}[h]
	\begin{center}
	\begin{subfigure}[b]{0.49\textwidth}
	\centering
    \includegraphics[width=\textwidth]{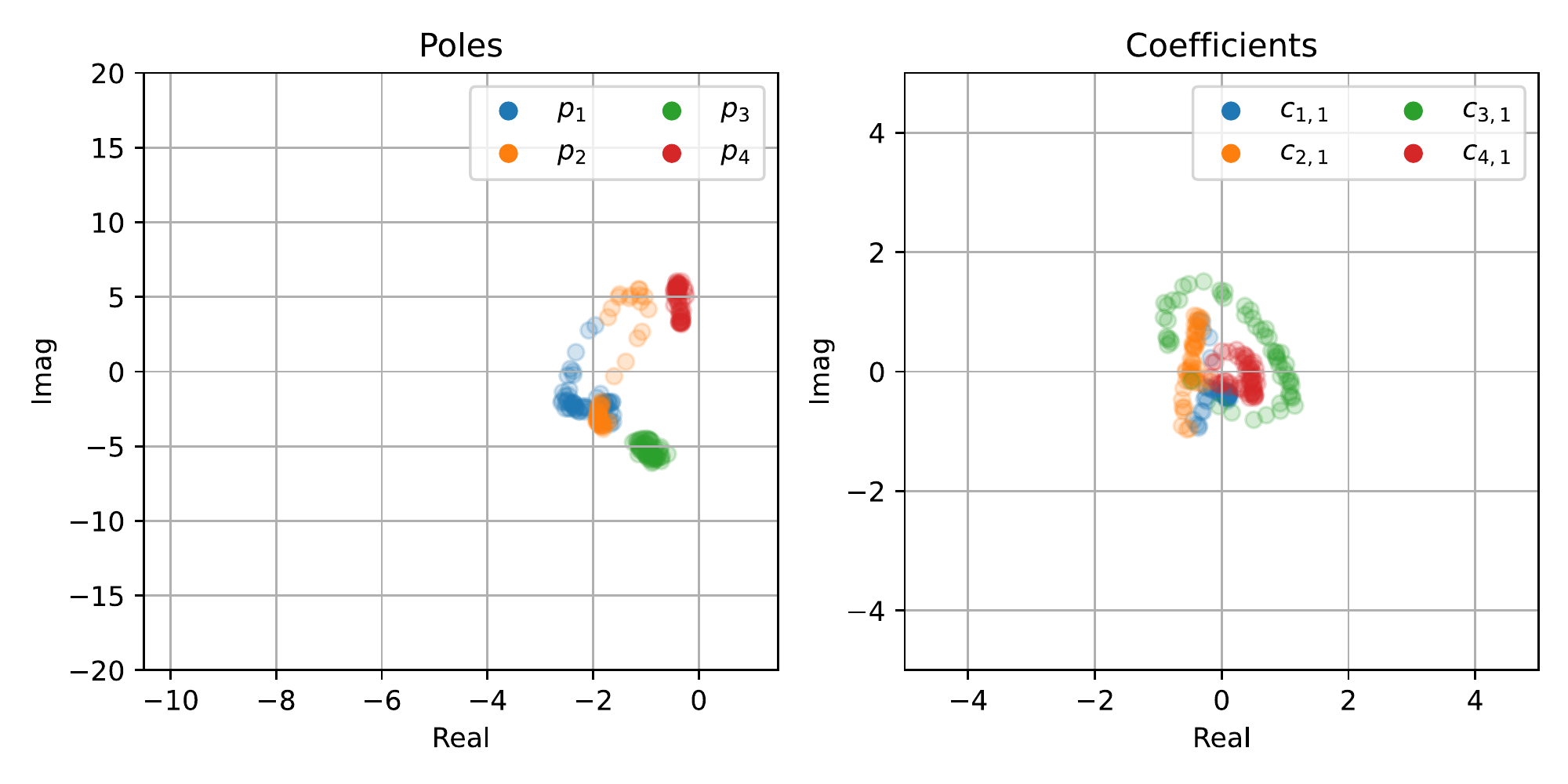}
    \caption{$\delta_{pole}=0.0$.}
    \end{subfigure}
    ~
    \begin{subfigure}[b]{0.49\textwidth}
	\centering
    \includegraphics[width=\textwidth]{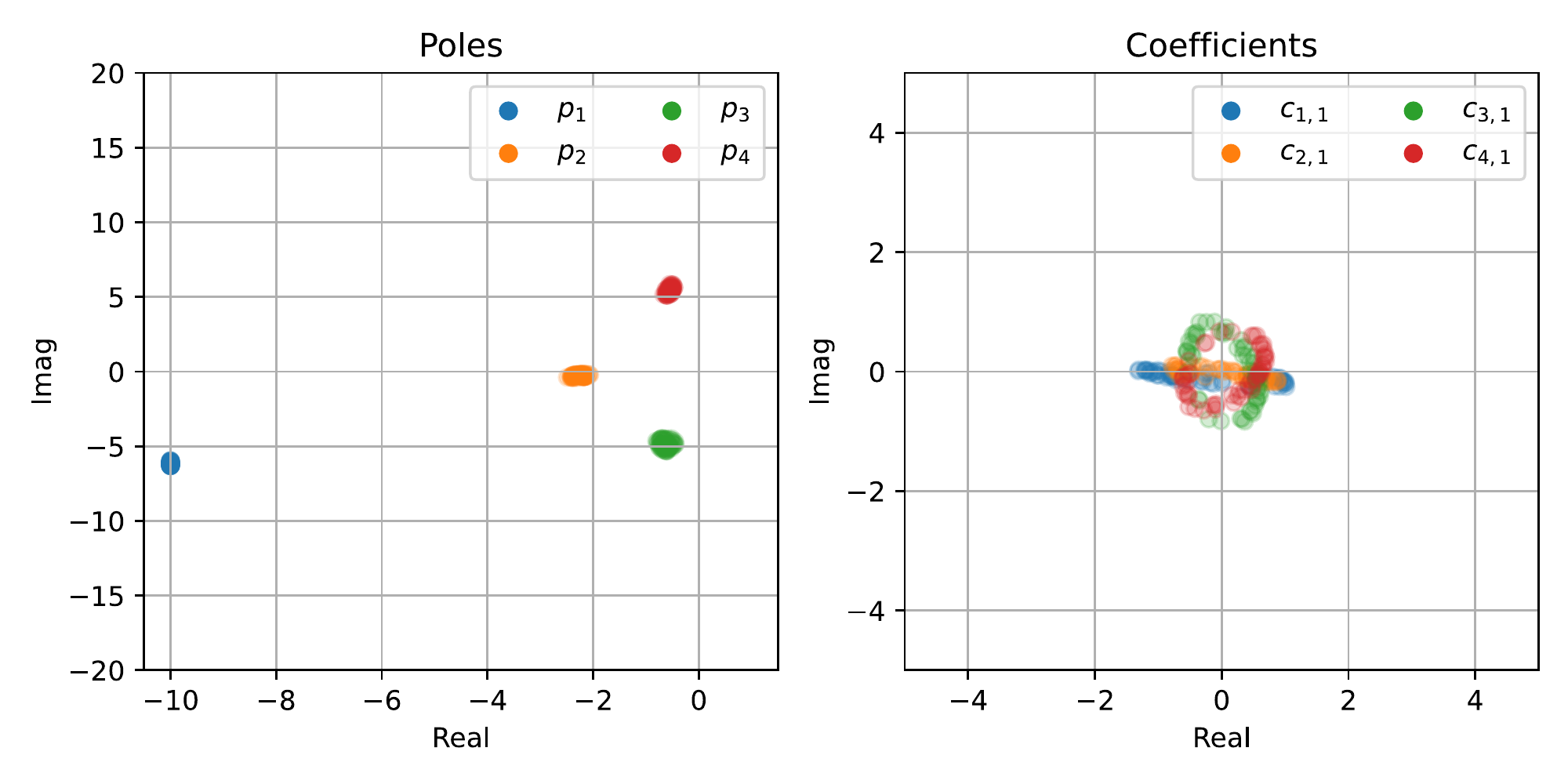}
    \caption{$\delta_{pole}=1.0$.}
    \end{subfigure}
    \caption{Distribution of Laplace Embedding $\vw$ under Different Thresholds of $\delta_{pole}$.
    \small
    The Laplace embeddings of trajectory $x(t)=\cos(2 \pi (t -\phi)), \phi\sim\mathrm{Exp}(\frac{1}{2})$ are plotted as poles and coefficients on the complex plane with different values of $\delta_{pole}$. 
    }
		\label{fig:pole-separation}
	\end{center}
\end{figure}

\textbf{Impact of Sampling Rate in Input Data.}\quad
The Nyquist Sampling Theorem states that a band-limited signal (maximum frequency of $B$) can be perfectly reconstructed from sequential observations with (average) sampling rate above $2B$. It provides a lower bound on the number of time-series observations required for our proposed Laplace encoder to work. 
Thus, we assume that the sampling rate in real-world datasets is sufficiently large so that important temporal patterns can be correctly identified. 
To validate the above statement, we conduct a synthetic experiment on time-series data generated by $x(t)=\sin(2\pi t+\varphi)$ where $\varphi \sim \mathrm{Exp}(\frac{1}{2})$ with different sampling rates. Figure \ref{fig:sampling-rate} demonstrates that the reconstruction error of the Laplace encoder converges to zero when the sampling rate is sufficiently large.

\begin{figure}[!h]
	\begin{center}
		\centering
		{\includegraphics[width=0.6\textwidth]{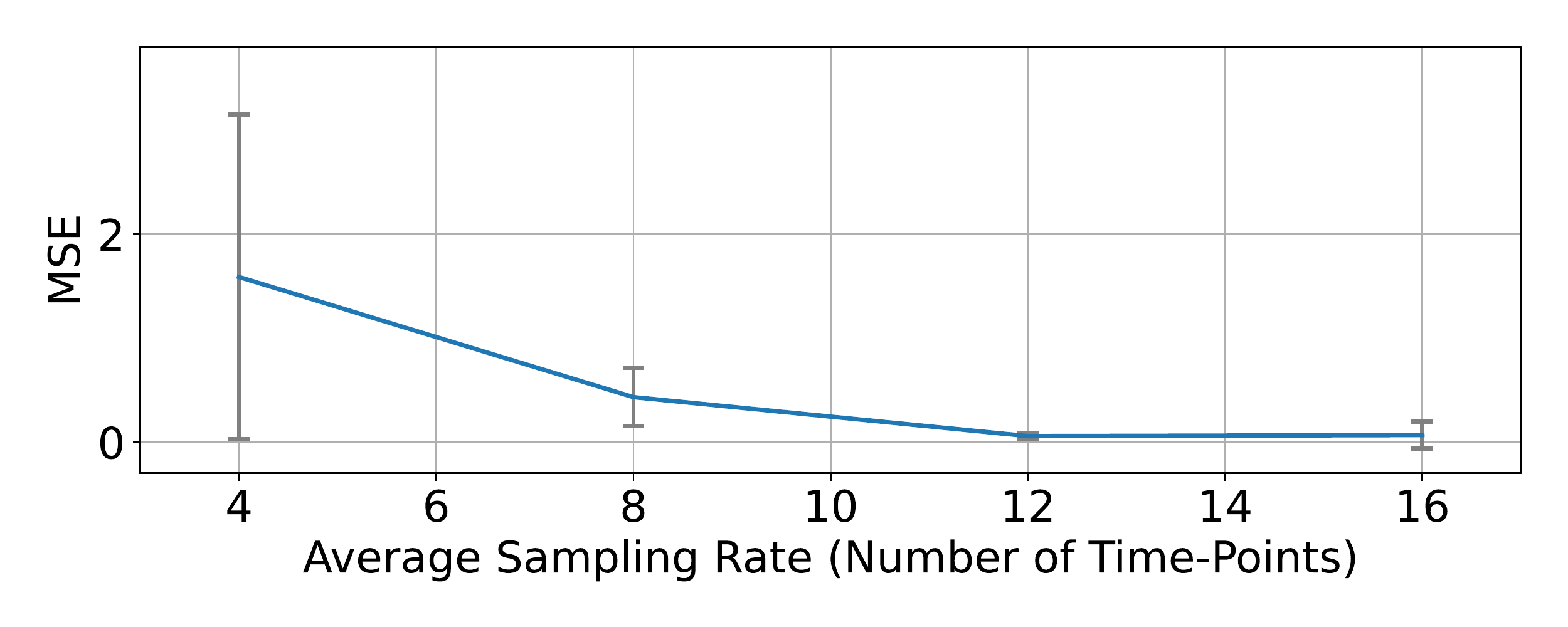}}
		\caption{{Impact of Sampling Rate on Time-series Reconstruction via Laplace Encoder.}
		\small} 
		\label{fig:sampling-rate}
	\end{center}
\end{figure}

\renewcommand{\thesection}{B}
\setcounter{table}{0}
\setcounter{figure}{0}
\setcounter{algorithm}{0}
\section{PROOF OF PROPOSITION 1}
Proposition \ref{proposition1} states that, given two Laplace embeddings $\vz^1$ and $\vz^2$ in latent space $\mathcal{Z}$, the distance between their corresponding time-domain trajectories $\hat{\vx}^1$ and $\hat{\vx}^2$ is upper-bounded by $ \psi\Vert\vz^1 - \vz^2\Vert_2^2$ with some scalar $\psi>0$. The proof of Proposition \ref{proposition1} can be derived as the following:
\vspace{-3mm}
\begin{proof} 
    Let us first consider the uni-variate case.
    Given two arbitrary Laplace embeddings $\vw^1, \vw^2 \in \mathbb{C}^{n(d+1)}$, their time-domain reconstructions can be obtained via inverse Laplace transform, i.e., $\hat{x}^i(t) = \mathcal{L}^{-1}(\vw^i) \triangleq \mathcal{L}^{-1}[F_{\vw^i}(s)](t), i=1,2$.
    According to (\ref{eq:laplace-transform}) and (\ref{eq:inverse-laplace-transform}), we have
    \begin{equation}
        \hat{x}^i(t) = \sum_{m=1}^n \sum_{l=1}^d \frac{c_{m,l}^it^{l-1}}{\Gamma(l)} e^{p_m^i t}, t\geq 0,
    \end{equation}
    where $\Gamma(l) = (l-1)!$ is the Gamma function, $\vw^i = [p_1^i,p_2^i,\ldots, c_{1,1}^i,\ldots, c_{n,d}^i]^\top, i=1,2$.
    
    \textbf{Difference in One Coefficient.}\quad
    Suppose $\vw^1$ and $\vw^2$ only differ at one coefficient $c_{m,l}$, which leads to the result $\Vert \vw^1 - \vw^2\Vert^2_2 = |c_{m,l}^1 - c_{m,l}^2|^2$.
    Then,
    \begin{equation}
    \begin{aligned}
        \Vert \hat{x}^1 - \hat{x}^2 \Vert_{\normltwo_{[0,1]}}^2 &= \int_0^1{|\hat{x}^1(t) - \hat{x}^2(t)|^2}\mathrm{d}t,\\
        & = \int_0^1{|c^1_{m,l} - c^2_{m,l}|^2 \left|\frac{t^{l-1}}{\Gamma(l)}e^{p_m t}\right|^2}\mathrm{d}t,\\
        & \leq |c^1_{m,l} - c^2_{m,l}|^2 \psi^c_{m,l} = \psi^c_{m,l} \Vert \vw^1 - \vw^2\Vert^2_2 ,
    \end{aligned}
    \end{equation}
    where $\psi^c_{m,l}$ is some suitable constant.
    
    \textbf{Difference in One Pole.}\quad
    Now, let us consider the case where $\vw^1$ and $\vw^2$ only differ at one pole $p_{m}$ which gives $\Vert \vw^1 - \vw^2\Vert^2_2 = |p_{m}^1 - p_{m}^2|^2$.
    Without loss of generality, we assume $p_m^2 - p_m^1 = r + j\theta $, where $r\leq 0, j^2=-1$. The following inequality can be established when $t\in [0,1]$:
    \begin{equation}
        \begin{aligned}
            |1 - e^{(p_m^2 - p_m^1) t}|^2 & = |1 - e^{r t} (\cos(\theta t) - j \sin(\theta t))|^2,\\
            & = (1-e^{rt})^2 + 2 e^{r t} (1-\cos (\theta t)),\\
            &\leq (1-e^{rt})^2 + e^{r t} \theta^2 t^2,&& \text{(via $e^{rt}> 0$ and $1-\cos(x) \leq \frac{x^2}{2}$)}\\
            & \leq r^2t^2 + e^{rt} \theta^2 t^2, &&\text{(via $r\leq 0$ and $0\leq 1-e^{rt} \leq (-r) t)$}\\
            & \leq (r^2+\theta^2) t^2,\\
            &= |p_m^1 - p_m^2|^2 t^2.
        \end{aligned}
    \end{equation}
    Hence, we have
    \begin{equation}
    \begin{aligned}
        \Vert \hat{x}^1 - \hat{x}^2 \Vert_{\normltwo_{[0,1]}}^2 &= \int_0^1{|\hat{x}^1(t) - \hat{x}^2(t)|^2}\mathrm{d}t,\\
        & = \int_0^1{\left|\frac{c_{i,l}t^{l-1}}{\Gamma(l)} \right|^2}|e^{p_m^1 t}|^2 |1 - e^{(p_i^2 - p_i^1) t}|^2\mathrm{d}t,\\
        & \leq  \int_0^1{\left|\frac{c_{i,l}t^{l-1}}{\Gamma(l)} \right|^2}|e^{p_m^1 t}|^2 |p_m^1 - p_m^2|^2 t^2 \mathrm{d}t,\\
        & \leq |p_m^1 - p_m^2|^2  \psi^p_m = \psi^p_m \Vert \vw^1 - \vw^2\Vert^2_2 ,
    \end{aligned}
    \end{equation}
    where $\psi^p_{m}$ is some suitable constant.

    {
    \textbf{General Cases.}\quad
    Now, we define an operator $S_{i}(\vw^1,\vw^2)$ that generates a new composite vector $\Bar{\vw}_i$ from $\vw^1$ and $\vw^2$.
    The first $i$ elements of the composite vector $\Bar{\vw}_i$ are taken from $\vw^2$ while the latter $n(d+1)-i$ elements of $\Bar{\vw}_i$ are obtained from $\vw^1$.
    For instance, we have
    $S_{0}(\vw^1,\vw^2)=\vw^1$,
    $S_{1}(\vw^1,\vw^2) =  [p_1^2,p_2^1,p_3^1,\ldots, c_{1,1}^1,\ldots,c_{n,d}^1]^\top$, $S_{2}(\vw^1,\vw^2) =  [p_1^2,p_2^2,p_3^1,p_4^1\ldots, c_{1,1}^1,\ldots, c_{n,d}^1]^\top$, $\ldots$, and $S_{n(d+1)}(\vw^1,\vw^2)=\vw^2$. It is easy to see that $S_{i}(\vw^1,\vw^2)$ and $S_{i+1}(\vw^1,\vw^2)$ only differ at one pole or one coefficient, and $\Vert S_{i}(\vw^1,\vw^2)-S_{i+1}(\vw^1,\vw^2)\Vert_2^2 = |p_{i+1}^1-p_{i+1}^2|^2$ when $0\leq i\leq n-1$ and $\Vert S_{i}(\vw^1,\vw^2)-S_{i+1}(\vw^1,\vw^2)\Vert_2^2 = |c_{m ,l}^1-c_{m,l}^2|^2$ otherwise, where $m=\floor{\frac{i-n}{d}+1}, l = i-n - (m-1) d+1$. 
    Each composite vector $\Bar{\vw}_i = S_{i}(\vw^1,\vw^2)$ yields a time-domain trajectory $\mathcal{L}^{-1}(\Bar{\vw}_i)$ via inverse Laplace transform of $F_{\Bar{\vw}_i}(s)$.
    
    Note that $\mathcal{L}^{-1}(\Bar{\vw}_0)=\mathcal{L}^{-1}(\vw^1)= \hat{x}^1$ and  $\mathcal{L}^{-1}(\Bar{\vw}_{n(d+1)})=\mathcal{L}^{-1}(\vw^2)= \hat{x}^2$. Based on the triangular inequality,
    \begin{equation}
    \begin{aligned}
         \Vert \hat{x}^1 - \hat{x}^2 \Vert_{\normltwo_{[0,1]}}^2 &= \Vert \sum_{i=0}^{n(d+1)-1} \mathcal{L}^{-1}(S_{i}(\vw^1,\vw^2)) - \mathcal{L}^{-1}(S_{i+1}(\vw^1,\vw^2)) \Vert_{\normltwo_{[0,1]}}^2\\
         &\leq  \sum_{i=0}^{n(d+1)-1} \Vert \mathcal{L}^{-1}(S_{i}(\vw^1,\vw^2)) - \mathcal{L}^{-1}(S_{i+1}(\vw^1,\vw^2)) \Vert_{\normltwo_{[0,1]}}^2,\\
         & \leq  \sum_{i=0}^{n(d+1)-1} \psi \Vert S_{i}(\vw^1,\vw^2) - S_{i+1}(\vw^1,\vw^2) \Vert_2^2,\\
         & =  \sum_{m=1}^{n} \psi |p_m^1-p_m^2|^2+\sum_{m=1}^{n} \sum_{l=1}^{d} \psi |c_{m,l}^1 - c_{m,l}^2|^2,\\
         & = \psi \Vert \vw^1 - \vw^2\Vert_2^2,
    \end{aligned}\label{eq:upper-bound}
    \end{equation}
    where we take $\psi = \max_{m,l} (\psi^p_m, \psi^c_{m,l})$.
    }
    
    Finally, for the multivariate case,
    let us consider two latent embeddings $\vz^1$ and $\vz^2$ as well as their associated time-domain reconstructions $\hat{\vx}^1$ and $\hat{\vx}^2$.
    We define the distance between trajectories $\hat{\vx}^1$ and $\hat{\vx}^2$ as
    \begin{equation}
        \Vert \hat{\vx}^1  - \hat{\vx}^2 \Vert_{\normltwo_{[0,1]}}^2 \triangleq \sum_{d=1}^{\mathrm{dim}_x} \int_0^1 |\hat{x}^1_d(t) - \hat{x}^2_d(t)|^2 \mathrm{d}t,
    \end{equation}
    where $\hat{x}_d^m$ is the $d$-th dimension of trajectory, $\hat{\vx}^m = \mathcal{L}^{-1}(\vw_d^m) = \mathcal{L}^{-1}[F_{\vw_d^m}(s)]$ for $m=1,2$, $\vw_d^m$ is the $d$-th component of $\vz^m$.
    According to (\ref{eq:upper-bound}), we have the following bound for each dimension $d$.
    \begin{equation}
         \int_0^1 |\hat{x}^1_d(t) - \hat{x}^2_d(t)|^2 \mathrm{d}t = \Vert \hat{\vx}_d^1  - \hat{\vx}_d^2 \Vert_{\normltwo_{[0,1]}}^2\leq \psi_d \Vert \vw_d^1-\vw_d^2\Vert_2^2,
    \end{equation}
    where $\psi_d>0$ is some suitable scalar.
    Since $\Vert \vz^1 - \vz^2_2 \Vert_2^2 =\sum_{d=1}^{\mathrm{dim}_x} \Vert \vw_d^1 - \vw_d^2 \Vert_2^2 $, the distance between the two reconstructed trajectories $\hat{\vx}^1$ and $\hat{\vx}^2$ can be upper-bounded as follows     with some suitable $\psi>0$.
    \begin{equation}
        \Vert \hat{\vx}^1  - \hat{\vx}^2 \Vert_{\normltwo_{[0,1]}}^2 \leq \sum_{d=1}^{\mathrm{dim}_x} \psi_d \Vert \vw_d^1 - \vw_d^2 \Vert_2^2 \leq \psi \Vert \vz^1 - \vz^2 \Vert_2^2.
    \end{equation}
\end{proof}

\begin{corollary}\label{corollary1}
Given a continuous set $\Phi_z$ in latent space, the set $\Phi$, which consists of reconstructed trajectories of $\vz\in \Phi_z$, is also a continuous set in trajectory space $\mathcal{X}$.
\end{corollary}

\begin{proof}
Consider a trajectory $\hat{\vx}\in \Phi$ and its corresponding latent embedding $\vz \in \Phi_z$.
For any $\varepsilon>0$, due to the continuity of $\Phi_z$, there must exist another embedding $\vz'\in \Phi_z$ such that $\Vert \vz-\vz'\Vert_2^2 < \delta{\varepsilon}$, where $\delta>0$ is a scalar.
Let us denote the time-domain reconstruction of $\vz'$ as $\hat{\vx}'\in \Phi$.
According to Proposition \ref{proposition1}, $\Vert \hat{\vx}  - \hat{\vx}' \Vert_{\normltwo_{[0,1]}}^2 \leq \psi \Vert \vz-\vz'\Vert_2^2$ holds for some $\psi>0$. Setting $\delta = \frac{1}{\psi}$ leads to the inequality $\Vert \hat{\vx}  - \hat{\vx}' \Vert_{\normltwo_{[0,1]}}^2 \leq\varepsilon$ which  indicates the continuity of set $\Phi$.
\end{proof}

\textbf{Equivalent Translation in the Latent Space.} \quad
Consider two trajectories $\vx^1, \vx^2\in \mathcal{X}$ with the corresponding latent embeddings $\vz^1$ and $\vz^2$ in the latent space.
We construct a set $P_z = \{ \Tilde{\gamma}(\vz^1\rightarrow\vz^2)\}$ of all possible continuous path $\Tilde{\gamma}$ in the latent space that connects $\vz^1$ and $\vz^2$.
Let $g_E:\mathcal{Z} \rightarrow \mathcal{X}$ be a function that maps latent embedding $\vz$ back to its time-domain reconstruction $\hat{\vx}$ in the trajectory space. 
Then, given a translation $\Gamma(\vx^1\rightarrow\vx^2)$ in the trajectory space, we can define the (approximately) equivalent translation in the latent space as 
\begin{equation}
    \gamma(\vz^1\rightarrow\vz^2) \triangleq \argmin_{\Tilde{\gamma}\in P_z} {\min_{\vz\in \Tilde{\gamma}} \max_{\vx\in \Gamma}\Vert \vx - g_E(\vz) \Vert_{\normltwo_{[0,1]}}^2},
\end{equation}
where ${\min_{\vz\in \Tilde{\gamma}} \max_{\vx\in \Gamma}\Vert \vx - g_E(\vz) \Vert_{\normltwo_{[0,1]}}^2}$ measures the minimum distance between translation, i.e., $\Gamma$, and the time-domain reconstruction of latent path $\Tilde{\gamma}$, i.e., $\Tilde{\Gamma}= \{g_E(\vz)~|~\vz\in \Tilde{\gamma}\}$. 
In general, $\gamma$ is the closet projection of $\Gamma$ within the latent space $\mathcal{Z}$, and the equivalence of trajectory translation is approximate.
If every trajectory $\vx\in \Gamma$ has a rational Laplace transform with no more than $n$ poles and maximum degree of $d$ as described in  (\ref{eq:laplace-transform}), the equivalence becomes strict.
Without loss of generality, let us consider the uni-variate case.
Given a translation $\Gamma$, we assume each $x\in \Gamma$ can be exactly described by the Laplace transform $F_\vw(s)$ in (\ref{eq:laplace-transform}), where $\vw = f_L(x(\vt))$, $\vt$ is a vector of some suitable sampling time stamps.
For any two trajectories $x,x'\in \Gamma$ that satisfy $|x(t) - x'(t)|\leq \delta$ almost everywhere in $t\in[0,1]$, we have
\begin{equation}
\begin{aligned}
    |F_{\vw}(s) - F_{\vw'}(s)|^2 &= \left|\int_0^{\infty} {(x(t)-x'(t))e^{-st} }\mathrm{d}t \right|^2,\\
    &\leq  \int_0^{\infty} {|x(t)-x'(t)|^2 |e^{-st}|^2 }\mathrm{d}t,\\
    & \leq \delta^2  \int_0^{\infty} {|e^{-st}|^2 }\mathrm{d}t,\\
    &=\frac{\delta^2}{2\mathrm{Re}(s)},
\end{aligned}
\end{equation}
holds for $\mathrm{Re}(s)>0$.
When $\delta\rightarrow 0$, we have $x'\rightarrow x$ and $F_{\vw'}\rightarrow F_{\vw}$.
Note that $F_{\vw} - F_{\vw'}$ is rational and can be determined with a sufficient number of observations in its region of convergence, e.g., $\mathrm{Re}(s)>0$.
The equivalence in Laplace transform, i.e., $|F_{\vw}(s)- F_{\vw}(s)|^2 \equiv 0$, implies that $\vw' = \vw$.\footnote{When $F_{\vw}(s)$ and $F_{\vw'}(s)$ have less than $n$ poles, $\vw$ and $\vw'$ may take value from multiple alternative embeddings. However, we can always select the combination such that $\vw' = \vw$.}
Thus, $x'\rightarrow x$ also leads to $\vw'\rightarrow\vw$, which means that the collection of Laplace embeddings $\{\vw|\vw=f_L(x(\vt)),x\in \Gamma\}$ is in fact a continuous path $\gamma$ in the latent space.
Thereby, path $\gamma$ is a latent translation that exactly yields the trajectory translation $\Gamma$. Similar results can be easily extended to the multi-variate trajectory setting.

\textbf{Justification for Latent Path-based Test.}\quad
The path-based connectivity test $\mathrm{d}_\Gamma(\vx^1,\vx^2)$ is defined based on the oracle model $g(\vx)$ of conditional distribution $p(\vy|\vx)$.
In our proposed method {\ours}, a predictor is built upon the Laplace embedding, i.e., $f(\mX) = f_P\circ f_E (\mX)$, to approximate the oracle conditional distribution such that $f(\mX)\approx g(\vx)$ given time-series $\mX$ sampled from $\vx$.
Thus, we have $\mathrm{d}_\Gamma(\vx^1,\vx^2) \approx \max_{\vx\in \Gamma, i=1,2} \mathrm{d}_y(f(\vx(\vt)),f(\mX^i))$, where $\vt$ is a vector of some suitable observation time stamps.
Further, note that translation $\Gamma$ in trajectory space can be approximated by $\hat{\Gamma}$ as time-domain reconstruction of latent translation $\gamma(\vz^1\rightarrow\vz^2)$ in $\mathcal{Z}$, where $\vz^i = f_E(\mX^i)$ for $i=1,2$. Then, we have 
\begin{equation}
\max_{\vx\in \Gamma, i=1,2} \mathrm{d}_y(f(\vx(\vt)),f(\mX^i))\approx \max_{\hat{\vx}\in \hat{\Gamma}, i=1,2} \mathrm{d}_y(f(\hat{\vx}(\vt)),f(\mX^i))  \approx \max_{\vz\in \gamma, i=1,2} \mathrm{d}_y(f_P(\vz),f_P(\vz^i)),
\end{equation}
which leads to the latent path-based test in (\ref{eq:latent-path-based-test}).

\renewcommand{\thesection}{C}
\setcounter{table}{0}
\setcounter{figure}{0}
\setcounter{algorithm}{0}
\section{GRAPH-CONSTRAINED \texorpdfstring{$K$-MEANS}{} ALGORITHM IN T-PHENOTYPE}
The graph-constrained $K$-means iteration in Algorithm \ref{alg:clustering} is provided in Algorithm \ref{alg:kmeans-on-graph}.
After each run via \var{GK-means}, the objective function $J$ in (\ref{eq:problem}) is re-evaluated. 
The main algorithm of {\ours} stops after 5 iterations with no improvement in objective $J$ under maximum of 1,000 iterations. 
Alternatively, {\ours} stops when the improvement is below certain tolerance $\mathrm{tol}=10^{-7}$, i.e., $|\Delta J| \leq \mathrm{tol}$.
\begin{algorithm}
\caption{\var{GK-means} (Single $K$-means iteration over similarity graph $\gG_\delta$)}\label{alg:kmeans-on-graph}
  \begin{algorithmic}[1]
  \Require{$J, e_1,e_2,\ldots, e_K,\gG_\delta$} \Comment{$J$ objective, $e_k$ cluster seed, $\gG_\delta$: similarity graph}
\Ensure{$\mathcal{C} = \{C_1,C_2,\ldots, C_K\}$}
    \For{$k=1,2,\ldots,K$}
        \State $\vv_k, \mX^{(k)} \gets e_k $
        \State $C_k \gets \{\mX^{(k)}\}$ \Comment{Initialize cluster $C_k$ with seed $e_k$}
    \EndFor
    \State $D_\mathrm{free} \gets \{\mX| \mX\not \in C_k, \,\forall C_k\in \mathcal{C}\}$ \Comment{Get the set of unclustered samples}
    \While {$|D_\mathrm{free}|>0$}
        \For{$\mX \in D_\mathrm{free}$}
            \State $C^\ast \gets \argmin_{C_k\in\mathcal{C}, \mX \xleftrightarrow{\gG_\delta} C_k} {\mathrm{d}_y(f(\mX),\vv_k)}$ \Comment{Find the best cluster assignment}
            \State $C^\ast \gets C^\ast \cup \{\mX\}$
            \State $D_\mathrm{free} \gets D_\mathrm{free} \setminus \{\mX\}$
        \EndFor
        \For{$k=1,2,\ldots,K$}
            \State $\vv_k\gets \frac{1}{|C_k|}\sum_{\mX\in C_k} f(\mX)$ \Comment{Update cluster centroid}
        \EndFor
    \EndWhile
  \end{algorithmic}
\end{algorithm}

\renewcommand{\thesection}{D}
\setcounter{table}{0}
\setcounter{figure}{0}
\setcounter{algorithm}{0}
\section{EXPERIMENT SETUP}

\subsection{Datasets and Statistics}
For the two real-world medical datasets, we want to capture recent temporal patterns and  associated target outcomes. Thus, we utilize a sliding window of size 6 years and 24 hours to extract sub-sequences containing temporal predictive patterns among most recent observations for ADNI and ICU datasets, respectively.
Statistics of major feature variables in the ADNI dataset and ICU dataset can be found in Table \ref{tab:ADNI-statistics} and Table \ref{tab:ICU-statistics}, respectively.

\begin{table}[!h]
\caption{Statistics of ADNI Dataset.
}
\label{tab:ADNI-statistics}
\begin{center}
\tiny
\begin{tabular}{lllllllll}
\toprule
\multicolumn{2}{l}{\textbf{\capstr{Static Covariates}}} & \textbf{TYPE} & \textbf{MEAN} & \textbf{MIN/MAX (MODE)}&  & \textbf{TYPE} &\textbf{MEAN} &\textbf{MIN/MAX (MODE)}\\
\midrule
\multirow{2}{*}{Demographic}& Race &Cat. &0.93 &White& Ethnicity& Cat. &0.97& Not Hisp/Latino\\
&Education &Cat. &16.13 &16 &Marital Status& Cat. &0.75 &Married\\
\midrule
Genetic& APOE $\varepsilon4$& Cat. &0.44& 0 \\
\midrule
\\
\multicolumn{2}{l}{\textbf{\capstr{Time-varying Covariates}}} & \textbf{TYPE} & \textbf{MEAN} & \textbf{MIN/MAX (MODE)}&  & \textbf{TYPE} &\textbf{MEAN} &\textbf{MIN/MAX (MODE)}\\
\midrule
\multirow{1}{*}{Demographic}& Age &Cont. &73.62 &55/91.4\\
\midrule
\multirow{4}{*}{Biomarker}& Entorhinal& Cont. & 3.6E+3& 1.0E+3/6.7E+3& Mid Temp& Cont. & 2.0E+4 & 8.9E+3/3.2E+4\\
& Fusiform& Cont. & 1.7E+4& 9.0E+3/2.9E+4& Ventricles& Cont.& 4.1E+4& 5.7E+3/1.6E+5\\
&Hippocampus&Cont.&6.9E+4&2.8E+3/1.1E+4& Whole Brain& 1.0E+6& 6.5E+5/1.5E+6\\
&Intracranial& Cont. & 1.5E+6& 2.9E+2/2.1E+6\\
\midrule
\multirow{4}{*}{Cognitive}& CDRSB& Cont.&1.21&0.0/17.0& Mini Mental State& Cont. &27.84& 2.0/30.0\\
& ADAS-11& Cont. & 8.58& 0.0/70.0 & ADAS-13& Cont. &13.60& 0.0/85.0\\
& RAVLT Immediate& Cont. & 38.26& 0.0/75.0 & RAVLT Learning & Cont.& 4.65&-5.0/14.0\\
& RAVLT Forgetting & Cont.& 4.19&-12.0/15.0& RAVLT Percent& Cont.& 51.68&-500.0/100.0\\
\bottomrule
\end{tabular}
\end{center}
\end{table}

\begin{table}[!h]
\caption{Statistics of ICU Dataset.
}
\label{tab:ICU-statistics}
\begin{center}
\tiny
\begin{tabular}{lllllllll}
\toprule
\multicolumn{2}{l}{\textbf{Static Covariates}} & \textbf{TYPE} & \textbf{MEAN} & \textbf{MIN/MAX (MODE)}&  & \textbf{TYPE} &\textbf{MEAN} &\textbf{MIN/MAX (MODE)}\\
\midrule
\multirow{1}{*}{Demographic}& Age &Cont. &67.25 &15.0/90.0& Gender& Cat. &0.56& Male\\
Admission &ICU Type& Cat. &2.76& Medical ICU\\
\midrule
\\
\multicolumn{2}{l}{\textbf{\capstr{Time-varying Covariates}}} & \textbf{TYPE} & \textbf{MEAN} & \textbf{MIN/MAX (MODE)}&  & \textbf{TYPE} &\textbf{MEAN} &\textbf{MIN/MAX (MODE)}\\
\midrule
\multirow{10}{*}{Blood Test}
& Albumin&Cont. &2.92& 1.0/5.3& ALP&Cont.&1.2E+2& 1.2E+1/2.2E+3\\
& ALT & Cont.& 3.9E+3& 1.0/1.2E+4&
AST& Cont.& 5.1E+2& 4.0/1.8E+4\\
&Bilirubin & Cont.& 2.91& 0.1/47.7& BUN&Cont.&27.41&0.0/197.0\\
&Cholesterol &Cont.&156.52& 28.0/330.0& Creatinine & Cont.& 1.50& 0.1/22.1\\
& Glucose &Cont.& 1.4E+3& 1.0E+1/1.1E+3& Lactate&Cont.& 2.88& 0.3/29.3\\
& HCO3& Cont.& 23.12& 5.0/50.0& pH & Cont.& 7.49& 1.0/735.0\\
&K &Cont. & 4.14& 1.8/22.9& Mg&Cont.& 2.03&0.6/9.9\\
& Na&Cont. &139.07& 98.0/177.0& 
HCT& Cont. & 30.69& 9.0/61.8\\
& TroponinI&Cont.& 7.15& 0.3/49.2& TroponinT& Cont. &1.20& 0.01/24.91\\
&Platelets &Cont. &1.9E+2& 6.0/1.0E+3& White Blood Cell & Cont. & 12.67& 0.1/187.5\\
\midrule
\multirow{6}{*}{Monitoring}
& Heart Rate& Cont. & 86.80& 0.0/199.5& Respiratory Rate& Cont.& 19.64&  0.0/98.0\\
& SysABP & Cont. &119.57& 0.0/273.0& NISysABP & Cont.&119.20&0.0/247.0\\
& DiasABP& Cont. & 59.54& 0.0/268.0& NIDiasABO& Cont. & 58.18& 0.0/180\\
& MAP& Cont. &80.23& 0.0/295.0& NIMAP& Cont. & 77.13& 0.0/194.0\\
&GCS& Cont.& 11.41& 3.0/15.0& Temperature& Cont.& 37.07& -17.8/42.1\\
& Urine& Cont.& 12E+2& 0.0/1.1E+5\\
\midrule
\multirow{2}{*}{Oxygen}
&FiO2& Cont. & 0.54& 0.21/1.0& PaCO2 & Cont. & 40.41& 11.0/100.0\\
& PaO2&Cont.& 147.82& 0.0/500.0&SaO2& Cont.& 96.65& 26.0/100.0\\
\bottomrule
\end{tabular}
\end{center}
\end{table}

\subsection{Baselines}
We compare the performance of {\ours} with the following five benchmarks ranging from traditional method to state-of-the-art deep learning-based methods, where each clustering method reflects a different notion of temporal phenotypes:

\textbf{$K$-means with Warping-based Distance.}\quad
The technique of dynamic time warping (DTW) provides one way to measure time-series similarity regardless of the observation interval.
Time-series with similar temporal patterns usually leads to smaller DTW distances.
We apply conventional $K$-means with the DTW-based similarity measure to discover clusters representing different temporal patterns. We denote this approach as KM-DTW.

\textbf{Deep Temporal $K$-means.}\quad
Embedding (i.e., hidden representations) from RNNs can provide meaningful information to measure the similarity between time-series.
With the encoder-predictor (E2P) structure introduced in \citep{lee2020temporal}, we include the baseline of KM-E2P that performs clustering in a representation space via $K$-means.
We denote the baseline as KM-E2P(z) when The representation space is formed by the latent embeddings from an encoder network.
The discovered cluster will capture both similarities in input time-series and the output label prediction due to the E2P structure.
When the representation space is selected to be the output (label prediction) of the predictor network, we refer to the method as KM-E2P(y).
In this case, the discovered clusters are aligned to major modes in the label distribution and are not necessarily associated with certain temporal patterns in trajectory space. 

\textbf{$K$-means with Laplace Encoder.}\quad
Similar to the baseline of KM-DTW, the time-series embedding from Laplace encoder provides a unified representation of (potentially) irregularly sampled time-series. The Euclidean distance between Laplace embeddings can thus be used as a similarity measure for different patient trajectories. In practice, the longitudinal observations of patients are first converted to a latent space via the Laplace encoder. Then, $K$-means algorithm is performed over the latent representations to identify patient subgroups based on their similarity in temporal patterns. 

\textbf{Toward $K$-means Friendly Spaces using Sequence-to-sequence.}\quad
Sequence-to-sequence (SEQ2SEQ) learning paradigm allows the learning of a representation space that is easier to perform clustering compared to the original time-series data. 
Such baseline reflects the recent trend of combining conventional clustering methods, e.g., $K$-means, with dimension reduction using deep learning technique \citep{xie2016unsupervised,baytas2017patient}.
With different temporal patterns encoded in a low-dimension representation space, $K$-means clustering is applied to discover clusters that represent various temporal feature interactions in input time-series data.
In the experiment, we use a modified version of DCN \citep{yang2017towards} as the SEQ2SEQ baseline.

\textbf{AC-TPC.}\quad
AC-TPC \citep{lee2020temporal} is one of the state-of-the-art temporal clustering approach that discovers outcome-oriented clusters.
AC-TPC learns a cluster assignment policy in the latent space based on  an encoder network.
The cluster assignment policy is trained with the actor-critic loss from reinforcement learning to find the optimal clusters that represent typical label distributions learned by a predictor network.
Similar to KM-E2P(y), there is no guarantee on the association between temporal patterns and clusters discovered by AC-TPC.

\subsection{Training Procedure of \ours}
To fit the model of {\ours} on a dataset, the Laplace encoder for each trajectory dimension is firstly pre-trained based on (\ref{eq:loss-encoder}) calculated at each time step.
Then, we fit the predictor $f_L$ with observed patient outcomes $y$.
Finally, the temporal clusters are discovered via graph-constrained $K$-means algorithm \ref{alg:kmeans-on-graph} based on the output from the predictor.
Latent embeddings from the Laplace encoder has a clear mathematical meaning. Thus, we freeze the pre-trained Laplace encoder to be isolated from gradients due to outcome predictions. 
Joint optimization of the encoder and predictor may lead to slower convergence and lower performance as shown in Table \ref{tab:benchmark-synth} and Table \ref{tab:benchmark-medical} with ``{\ours} (J)'' as the ablation study.

\subsection{Performance Metrics}

\textbf{Prediction Performance.}\quad
Area under the curve of receiving-operator characteristic (AUROC) and area under the curve of precision-recall (AUPRC) are used to assess the prognostic value of the discovered clusters on predicting the target label $\vy$. 
For non-binary (category larger than 2) labels, these scores are calculated individually for each category and averaged over the entire categories. 

\textbf{Clustering Performance.}\quad
For synthetic data, we evaluate the clustering performance in terms of the purity score \citep{lee2020temporal}, adjusted Rand index (RAND) \citep{steinley2004properties}, and normalized mutual information (NMI) \citep{vinh2009information} as the ground-truth cluster label is available. 
For the real-world dataset, there is no ground-truth of cluster label. 
In such a case, the Silhouette coefficient \citep{rousseeuw1987silhouettes} is commonly used as a measure of cluster consistency by assessing the homogeneity within each cluster and heterogeneity across different clusters.
More specifically, the traditional Silhouette index assumes convex clusters and 
uses the average intra-cluster distance ($a$) and inter-cluster distance ($b)$ to evaluate the consistency between cluster assignment and pattern distribution as $s = \frac{|b-a|}{\max (a,b)}$. 
Averaging $s$ over all samples gives the Silhouette index $S$.

In this paper, the clusters are identified via predictive temporal patterns and are not necessarily in convex shapes.
To better reflect our new notion of clusters, we instead use an $m$-nearest neighbor version of Silhouette index, i.e., $S^m$.
Specifically, suppose there are $K$ clusters $\mathcal{C}=\{C_1,C_2,\ldots, C_K\}$.
Given a time-series $\mX$ in cluster $C_k$, we only consider its $m$ nearest samples in the corresponding cluster when calculating intra- and inter-cluster distances $a^m$ and $b^m$ as given below:
\begin{equation}
    a^m = \frac{1}{|N_m(\mX, C_k)|}\sum_{\mX' \in N_m(\mX, C_k)}{\Vert\mX - \mX' \Vert_2^2},~~~~b^m= \min_{i\neq k}  \frac{1}{|N_m(\mX, C_i)|}\sum_{\mX' \in N_m(\mX, C_i)}{\Vert\mX - \mX' \Vert_2^2},
\end{equation}
where $N_m(\mX,C_k)$ indicates the set of $m$ nearest neighbors of $\mX$ in cluster $C_k$. Then, the clustering consistency in our variant Silhouette index is calculated as $s^m = \frac{|b^m-a^m|}{\max (a^m,b^m)}$. The average score $S^m$ of all samples is used to measure the overall clustering consistency.
Note that when $m\geq \max_{C_k\in \mathcal{C}} |C_k| $, the variant $S^m$ is identical to the original Silhouette index, i.e., $S^m=S$.

Focusing on $m$ closest samples allows us to effectively evaluate pattern consistency in non-convex and irregularly shaped clusters. 
Nonetheless, when multiple temporal patterns are put into the same cluster, $S^m$ may still generate a high score due to the focus on local similarity.
To address this issue, we use another connectivity-based metric $P^m$ to evaluate the purity of a cluster in terms of temporal patterns.
Consider a cluster $C_k$, a connectivity graph over time-series in $C_k$ can be derived via $m$-nearest neighbor discovery. We use the count $p_k$ of connected subgraphs to estimate the number of temporal pattern included in cluster $C_k$ and calculate the temporal pattern purity via $P^m = \frac{1}{K}\sum_{C_k\in \mathcal{C}} \frac{1}{p_k}$.
It is clear that $P^m=1$ when $m$ is sufficiently large and each cluster only contains a single temporal pattern, and $P^m=\frac{1}{K}\sum_{C_k\in \mathcal{C}} \frac{1}{|C_k|}$ when $m=0$.

To get an overall assessment of cluster consistency, we normalize $S^m$ into $[0,1]$ and calculate the summary metric \textit{AUSIL} as the area under the curve of $S^m$ verses $P^m$ for $m=1,2,\ldots,M, M\in \sN$.
For the evaluation of phenotype discovery, we combine the prediction accuracy (AUROC and AUPRC) and cluster consistency (AUSIL) into two composite metrics $H_\mathrm{ROC}$ and $H_\mathrm{PRC}$.
Similar to the F1-score in classification, these composite metrics are defined respectively as 
\begin{equation}
H_\mathrm{ROC} \triangleq 2 \frac{\mathrm{AUROC}\cdot \mathrm{AUSIL} }{\mathrm{AUROC}+\mathrm{AUSIL}}, ~~~~~    H_\mathrm{PRC} \triangleq 2 \frac{\mathrm{AUPRC}\cdot \mathrm{AUSIL} }{\mathrm{AUPRC}+\mathrm{AUSIL}}.
\end{equation}

\renewcommand{\thesection}{E}
\setcounter{table}{0}
\setcounter{figure}{0}
\setcounter{algorithm}{0}
\section{HYPERPARAMETER SELECTION}
In the experiment, {\ours}, KM-E2P($y$), KM-E2P($z$) are implemented with PyTorch and are trained with learning rate of 0.1 in 50 epochs.
AdamW optimizer is used to tune the network parameters.
The $K$-means clustering in KM-E2P($y$), KM-E2P($z$) and KM-DTW is performed with $K$-means++ initialization based on implementation in PyClustering.\footnote{\url{https://pyclustering.github.io/}} The baselines of AC-TPC and SEQ2SEQ are implemented in TensorFlow.
They are trained with Adam optimizer with training epochs set to 200 due to different learning rates in their implementation.

We perform hyperparameter selection on each dataset via 3-fold cross-validation.
For {\ours}, the best hyperparameters of the Laplace encoders are searched to minimize the average reconstruction error over all temporal dimensions.
For each real-world dataset, the best number of clusters $K$ is searched via maximizing the composite metric $H_\mathrm{PRC}$ of {\ours}.
The selected best cluster number $K$ is used for all baselines on the same dataset.
For baselines of KM-E2P($y$) and  KM-E2P($z$), the hyperparameters for each dataset are search to maximize $H_\mathrm{PRC}$ (or purity score on the synthetic dataset) given the selected cluster number $K$.
The hyperparameters of AC-TPC and SEQ2SEQ are set to be the same with the original implementation in \citep{lee2020temporal} (dropout layers are disabled to ensure reproducibility). 
The hyperparameter space considered in our experiment is discussed as follows.

\subsection{Hyperparameter Selection of {\ours}}

\textbf{Laplace Encoder.}\quad
In the experiment, each Laplace encoder $f_L$ in {\ours} contains a 1-layer GRU and a 1-layer MLP with 10 hidden units in each layer. 
Given a time-series input, each Laplace encoder generates an embedding with $n=4$ poles and maximum degree of $d=1$.
As mentioned earlier, coefficient $\alpha_2$ for regularization term $l_\mathrm{distinct}$ is set to $0.01$ throughout the experiment.
The rest hyperparameters are searched in the parameter space as follows.
\begin{itemize}
    \item {Coefficient for pole separation loss $l_\mathrm{sep}$: $\alpha\in\{1.0,10.0\}$.}
    \item {Coefficient loss $l_\mathrm{real}$: $\alpha_1\in\{0.1,1.0\}$.}
    \item {Threshold for pole sorting  and the separation loss: $\delta_{pole}\in\{1.0,2.0\}$.}
\end{itemize}
To address the complex temporal patterns in the ICU dataset, the maximum degree of poles $d$ is also added to the search space, and the range of $d\in \{1,2\}$ is considered. 
The best hyperparameter for Laplace encoder on the three datasets are given as follows.
\begin{itemize}
    \item {Synthetic dataset: $\alpha=1.0, \alpha_1 = 0.1, \delta_{pole}=1.0$.}
    \item {ADNI dataset: $\alpha=1.0, \alpha_1 = 0.1, \delta_{pole}=2.0$.}
    \item {ICU dataset: $\alpha=1.0, \alpha_1 = 0.1, \delta_{pole}=2.0, d=2$.}
\end{itemize}

\textbf{Predictor.}\quad
The predictor $f_{P}$ is composed of a $3$-layer MLP with $10$ hidden units in each layer.

\textbf{Cluster Number $K$.}\quad
The best number of $K$ for each dataset is selected based on the optimal Laplace encoder and predictor structures selected above.
We use the ground truth cluster number $K=3$ for the synthetic dataset.
For the two real-world datasets, the cluster number is searched among $K\in\{2,3,4,5\}$ to maximize the composite clustering performance $H_\mathrm{PRC}$.
The optimal cluster number selection result is given below.
\begin{itemize}
    \item {Synthetic dataset: $K=3$ (we directly use the ground truth).}
    \item {ADNI dataset: $K=4$.}
    \item {ICU dataset: $K=3$.}
\end{itemize}

\subsection{Hyperparameter Selection of Baselines}
\textbf{KM-E2P($y$).}\quad
The KM-E2P($y$) model includes a 1-layer GRU network to extract temporal features from input time-series. A $2$-layer MLP is stacked on top of the GRU network to form an encoder.
Given the encoder output, another $2$-layer MLP is used to predict the categorical label $\vy$.
All layers in the GRU and MLP share the same number $h$ of hidden units.
Hyperparameters of $h\in\times\{5,10,20\}$ is searched in each dataset basedd on the corresponding $K$ determined above.
By maximizing the composite metric $H_\mathrm{PRC}$ or purity score, the hyperparameter selection result is obtained as follows.
\begin{itemize}
    \item {Synthetic dataset: $h=20$.}
    \item {ADNI dataset: $h=20$.}
    \item {ICU dataset: $h=20$.}
\end{itemize}

\textbf{KM-E2P($z$).}\quad
Similar to KM-E2P($y$), the KM-E2P($z$) model is composed of a encoder with 2-layer MLP on top  of a 1-layer GRU network to extract temporal features from input time-series. 
The encoder outputs a $r$-dimension latent vector, which is then used by a 2-layer MLP-based predictor for label prediction.
All layers in the GRU and MLP share the same number $h$ of hidden units.
Given the best cluster numbers of $K$ found by {\ours}, on each dataset, the optimal combination of $h$ and $r$ are search in the space of $(h,r)\in \{10,20\}\times\{5,10,20\}$ to maximize the composite metric $H_\mathrm{PRC}$ or purity score when ground truth cluster label is available.
The hyperparameter selection result is given as follows.
\begin{itemize}
    \item {Synthetic dataset: $h=10,r = 10$.}
    \item {ADNI dataset: $h=10,r=20$.}
    \item {ICU dataset: $h=20,r=10$.}
\end{itemize}

\textbf{KM-$\mathcal{L}$.}\quad
The baseline KM-$\mathcal{L}$ simply shares hyperparameters with {\ours} for its Laplace encoders on each dataset.

\renewcommand{\thesection}{F}
\setcounter{table}{0}
\setcounter{figure}{0}
\setcounter{algorithm}{0}
\section{COMPLETE BENCHMARK RESULT}
The complete benchmark result on synthetic dataset is shown in Table \ref{tab:benchmark-synth-full}.
{\ours} has significantly better clustering performance (purity score, adjusted Rand index, normalized mutual information) over all baselines on the synthetic data. 
In the meantime, the advantage of {\ours} over other baselines (except for AC-TPC) is clearly demonstrated via the proposed phenotype discovery performance metrics of $H_\mathrm{ROC}$ and $H_\mathrm{PRC}$.
An extra baseline of KM-Laplacian ($K$-means on graph Laplacian calculated via dynamic time warping) is included in Table \ref{tab:benchmark-synth-full} for reference. We note that this method has two major drawbacks: 1) there is not a stable and consistent representation space for cluster assignment for new samples; and 2) the distance matrix computation complexity in dynamic time warping could be extremely high, which makes this baseline infeasible for the two real-world datasets.

\begin{table}[!h]
\caption{Complete Benchmark Result on the Synthetic Dataset.
}
\label{tab:benchmark-synth-full}
\begin{center}
\scriptsize
\begin{tabular}{lccccccc}
\toprule
\textbf{METHOD}   &\textbf{AUROC}&\textbf{AUPRC}& \textbf{PURITY} & \textbf{RAND} & \textbf{NMI}&\textbf{$H_\mathrm{ROC}$}&\textbf{$H_\mathrm{PRC}$} \\
\midrule
KM-E2P(y)  &0.973$\pm$0.014 	&\textbf{0.962$\pm$0.019 }	&0.663$\pm$0.019 &	0.477$\pm$0.033 &	0.569$\pm$0.045 	&0.846$\pm$0.012 	&0.842$\pm$0.010\\
KM-E2P(z) &	0.963$\pm$0.012 	&0.948$\pm$0.011 &	0.677$\pm$0.029 &	0.418$\pm$0.024 	&0.485$\pm$0.047 	&	0.879$\pm$0.011 	&0.873$\pm$0.009 \\
KM-DTW & 0.722$\pm$0.033 	&0.649$\pm$0.028 &	0.469$\pm$0.017 	&0.068$\pm$0.021 &	0.077$\pm$0.022 	& 	0.787$\pm$0.020 	&0.742$\pm$0.019		 \\
KM-Laplacian& 0.736$\pm$0.024 & 	0.663$\pm$0.017 	&0.490$\pm$0.021 	&0.086$\pm$0.011 &	0.094$\pm$0.010 & 	0.797$\pm$0.016 &	0.752$\pm$0.013\\
KM-$\mathcal{L}$& 0.646$\pm$0.030 	&0.593$\pm$0.027 &	0.687$\pm$0.033 	&0.395$\pm$0.058 	&0.447$\pm$0.059 &		0.735$\pm$0.020 &	0.700$\pm$0.017\\
SEQ2SEQ &0.507$\pm$0.028 &	0.505$\pm$0.014 &	0.378$\pm$0.008 	&-0.003$\pm$0.003 	&0.005$\pm$0.003 & 	0.630$\pm$0.022 	&0.628$\pm$0.011\\ 
AC-TPC &0.966$\pm$0.012 	&0.952$\pm$0.017 	&0.659$\pm$0.020 &	0.487$\pm$0.035 &	0.596$\pm$0.043 &	\textbf{0.931$\pm$0.011 }	&\textbf{0.925$\pm$0.014}\\ \midrule
{\ours} (J)& 0.967$\pm$0.020 	&0.954$\pm$0.025 &	0.655$\pm$0.021 &	0.440$\pm$0.051 &	0.543$\pm$0.064 	&	0.845$\pm$0.064 &	0.840$\pm$0.064\\ 
{\ours}& \textbf{0.975$\pm$0.013} 	&0.960$\pm$0.024 &	\textbf{0.965$\pm$0.018}$^\ddagger$ 	&\textbf{0.902$\pm$0.048 }$^\ddagger$	&\textbf{0.875$\pm$0.050}$^\ddagger$&	0.927$\pm$0.010 	&0.920$\pm$0.014\\
\bottomrule
\end{tabular}
\end{center}
\small
Best performance is highlighted in \textbf{bold}. Symbol $^\ddagger$ indicates $p$-value $<0.01$
\end{table}

The complete benchmark result on two real-world datasets is provided in Table \ref{tab:benchmark-medical-full}.
{\ours} in general has the best (or second best) phenotype discovery performance ($H_\mathrm{ROC}$ and $H_\mathrm{PRC}$) while achieving high accuracy in outcome prediction (AUROC and AUPRC), which demonstrates the prognostic value of the phenotypes discovered by {\ours}.

\begin{table}[!h]
\caption{Complete Benchmark Result on Two Real-world Datasets.}
\label{tab:benchmark-medical-full}
\begin{center}
\small
\begin{tabular}{l lccccc}
\toprule
{} & \textbf{METHOD}   & \textbf{AUROC} & \textbf{AUPRC}& \textbf{AUSIL} & $H_\mathrm{ROC}$&$H_\mathrm{PRC}$\\
\midrule
\multirow{8}{*}{\rotatebox[origin=c]{90}{\textbf{ADNI}}}&
KM-E2P(y) &\textbf{0.893$\pm$0.005} &	\textbf{0.728$\pm$0.017} &	0.677$\pm$0.019 &	0.770$\pm$0.013 &	0.701$\pm$0.012\\
&KM-E2P(z) &0.884$\pm$0.012 &	0.711$\pm$0.020& 	0.672$\pm$0.028 &	0.763$\pm$0.018 	&0.690$\pm$0.013\\
&KM-DTW & 0.743$\pm$0.013 	&0.522$\pm$0.020& 	0.762$\pm$0.049 &	0.752$\pm$0.027 	&0.618$\pm$0.021\\
&KM-$\mathcal{L}$&0.697$\pm$0.029 &	0.465$\pm$0.019& 	\textbf{0.820$\pm$0.022}$^\ddagger$ &	0.753$\pm$0.019 	&0.593$\pm$0.018\\
&SEQ2SEQ &0.775$\pm$0.023 	&0.550$\pm$0.030 &	0.772$\pm$0.014 &	0.773$\pm$0.012 &	0.642$\pm$0.022\\
&AC-TPC &0.861$\pm$0.012 	&0.665$\pm$0.020 &	0.726$\pm$0.020 &	0.788$\pm$0.014 &	0.694$\pm$0.013\\ \cmidrule{2-7}
&{\ours} (J)&0.867$\pm$0.020 	&0.679$\pm$0.040& 	0.690$\pm$0.007 &	0.768$\pm$0.011 &	0.684$\pm$0.021\\
&{\ours} & 0.891$\pm$0.005 	&0.716$\pm$0.015& 	0.711$\pm$0.023 &	\textbf{0.791$\pm$0.013} &	\textbf{0.713$\pm$0.009}$^\ddagger$\\
\midrule
\multirow{8}{*}{\rotatebox[origin=c]{90}{\textbf{ICU}}}&
KM-E2P(y) & \textbf{0.697$\pm$0.014} 	&0.593$\pm$0.012 &	0.668$\pm$0.046 &	0.682$\pm$0.029 &	0.628$\pm$0.025\\
&KM-E2P(z) &0.677$\pm$0.030 &	0.579$\pm$0.018 &	0.698$\pm$0.042 &	0.686$\pm$0.031 &	0.633$\pm$0.024\\
&KM-DTW & 0.539$\pm$0.030 	&0.515$\pm$0.011 &	0.786$\pm$0.072 &	0.636$\pm$0.023 &	0.621$\pm$0.021\\
&KM-$\mathcal{L}$&0.577$\pm$0.019 	&0.532$\pm$0.009 	&\textbf{0.834$\pm$0.024} &	0.682$\pm$0.009 &	0.649$\pm$0.004\\
&SEQ2SEQ &0.592$\pm$0.024 &	0.539$\pm$0.012 &	0.830$\pm$0.016 &	0.690$\pm$0.011 &	\textbf{0.653$\pm$0.004}	\\
&AC-TPC &	0.660$\pm$0.008 &	0.573$\pm$0.003& 	0.735$\pm$0.024 &	0.695$\pm$0.014 &	0.644$\pm$0.011\\ \cmidrule{2-7}
&{\ours} (J)&0.697$\pm$0.025 &	\textbf{0.595$\pm$0.017} &	0.691$\pm$0.091 &	0.691$\pm$0.056 &	0.636$\pm$0.048\\
&{\ours} & 0.681$\pm$0.017 &	0.585$\pm$0.015 &	0.726$\pm$0.015 &	\textbf{0.703$\pm$0.007} &	0.648$\pm$0.008\\
\bottomrule 
\end{tabular}
\end{center}
\small
Best performance is highlighted in \textbf{bold}. Symbol $^\ddagger$ indicates $p$-value $<0.01$
\end{table}

\renewcommand{\thesection}{G}
\setcounter{table}{0}
\setcounter{figure}{0}
\setcounter{algorithm}{0}
\section{FURTHER ANALYSIS ON PHENOTYPE DISCOVERY}

\textbf{Comparison of Cluster Assignments on ADNI Dataset.} \quad
On the ADNI dataset, typical phenotypes from KM-E2P($y$), SEQ2SEQ, AC-TPC and {\ours} are compared in \Figref{fig:ADNI-comparison}.
Due to the model design, KM-E2P($y$) only focuses on the predicted outcome distribution when discovering phenotypes (as shown in \Figref{fig:ADNI-KME2Py}).
Compared to {\ours}, KM-E2P($y$) wrongly splits normal patients with the same temporal pattern (stable CDRSB trajectory) into two clusters under $K=4$.
Additionally,  KM-E2P($y$)  fails to discover the two subtypes of patients with high-risk of MCI as illustrated in phenotype 2 and 3 in \Figref{fig:ADNI-ours-appendix}.
While the SEQ2SEQ method is able to capture temporal patterns exhibit in patient trajectories, it is incapable to properly associate these temporal patterns with patient outcomes.
For instance, SEQ2SEQ wrongly splits high-risk patients with increasing CDRSB scores over time into two different subgroups with similar outcome distributions.

As discussed in the main manuscript, AC-TPC aims at discovering the minimum number of clusters that can sufficiently represent the outcome distribution.
Thus, it only identifies three phenotypes under $K=4$ and combines the two subtypes (Phenotype 2 and 3 in \Figref{fig:ADNI-ours-appendix}) of MCI patients into the same cluster.
In comparison, {\ours} discovers phenotypes based on both predicted outcome and the associated predictive temporal patterns.
The two subgroups of patients with expected diagnosis of MCI are correctly identified by {\ours}, which demonstrates the prognostic value of our method over the considered baselines.

\begin{figure*}[!htb]
	\begin{center}
	\centering
	\begin{subfigure}[b]{\textwidth}
	\centering
    \includegraphics[width=1.0\textwidth]{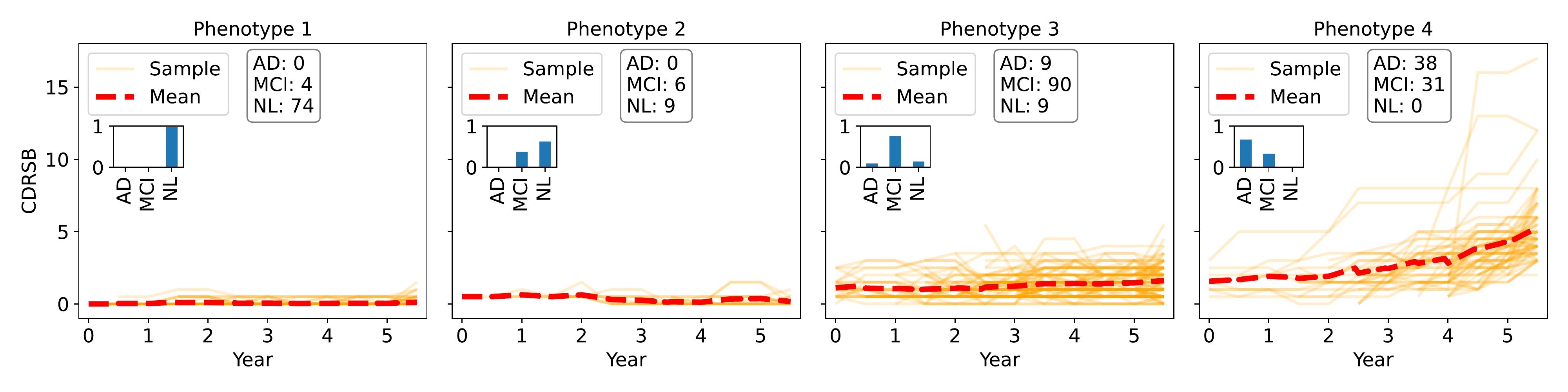}
    \caption{Four phenotypes from KM-E2P($y$).}
    \label{fig:ADNI-KME2Py}
    \end{subfigure}
    \\
    \begin{subfigure}[b]{\textwidth}
    \centering
        \includegraphics[width=1.0\textwidth]{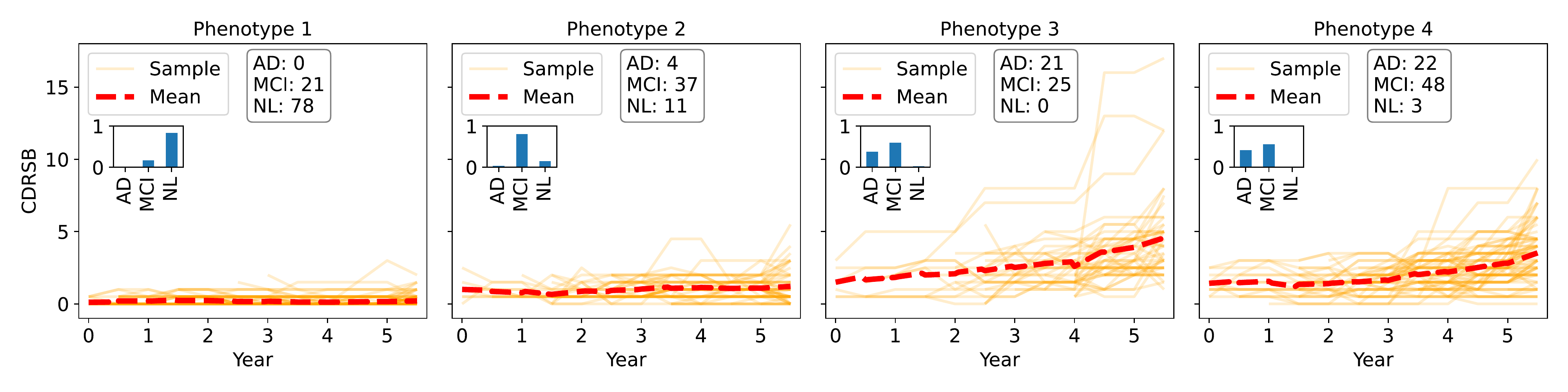}
        \caption{Four phenotypes from SEQ2SEQ.}
        \label{fig:ADNI-SEQ2SEQ}
    \end{subfigure}\\
	\begin{subfigure}[b]{\textwidth}
	\centering
    \includegraphics[width=0.75\textwidth]{figures/ADNI_AC_TPC.pdf}
    \caption{Three phenotypes from AC-TPC.}
    \label{fig:ADNI-AC-TPC-appendix}
    \end{subfigure}
    \\
    \begin{subfigure}[b]{\textwidth}
	\centering
    \includegraphics[width=1.0\textwidth]{figures/ADNI_TPhenotype.pdf}
    \caption{Four phenotypes from {\ours}.}
    \label{fig:ADNI-ours-appendix}
    \end{subfigure}
    \caption{{Comparison of Phenotypes Discovered on the ADNI Dataset.}
		\small} 
		\label{fig:ADNI-comparison}
	\end{center}
\end{figure*}

\textbf{Phenotypes on ICU Mortality.}\quad
On the ICU dataset, {\ours} is applied to identify phenotypes based on the patient's age, gender, GCS score and the fraction of PaCO2.
Three major phenotypes are discovered by {\ours}, and the GCS trajectories of test samples in each subgroup are illustrated in \Figref{fig:Physionet-phenotypes}.
Based on the stability of their GCS trajectory, patients in each phenotype are plotted separately in two subfigures.
The GCS score is predictive of patient mortality after ICU discharge \citep{leitgeb2013glasgow} and shows good discrimination accuracy on high- and low-risk patients admitted to ICU \citep{bastos1993glasgow}.
The predicted mortality rates in phenotypes 1, 2 and 3 are 15.3\%, 3.2\% and 32.4\%, respectively.
The GCS levels of patients in the three subgroups manifest a clear association to their corresponding mortality risks.
For instance, many patients of phenotype 3 had lower GCS score (below 10) than the two other subgroups.
In contrast, while having higher GCS levels, many patients in Phenotype 1 and 2 had an increase pattern (as shown in \Figref{fig:Physionet-nonstable}) in their recent GCS measurements, which potentially contributes to their decreased risks of death.
In the meantime, age is reported to be another risk factor for ICU mortality \citep{blot2009epidemiology,haas2017outcome}.
With the average patient age of 63.0 (IQR: 53.0 -- 76.0), 43.0 (IQR: 29.8 -- 55.3) and 70.6 (IQR: 62.0 -- 82.0) in the three identified subgroups, phenotype 1 and 2 are clearly separated.\footnote{Interquartile range (IQR) is the range defined by 25\% and 75\% quantiles of a variable.}

\begin{figure}[h]
	\begin{center}
	\begin{subfigure}[b]{\textwidth}
	\centering
    \includegraphics[width=0.9\textwidth]{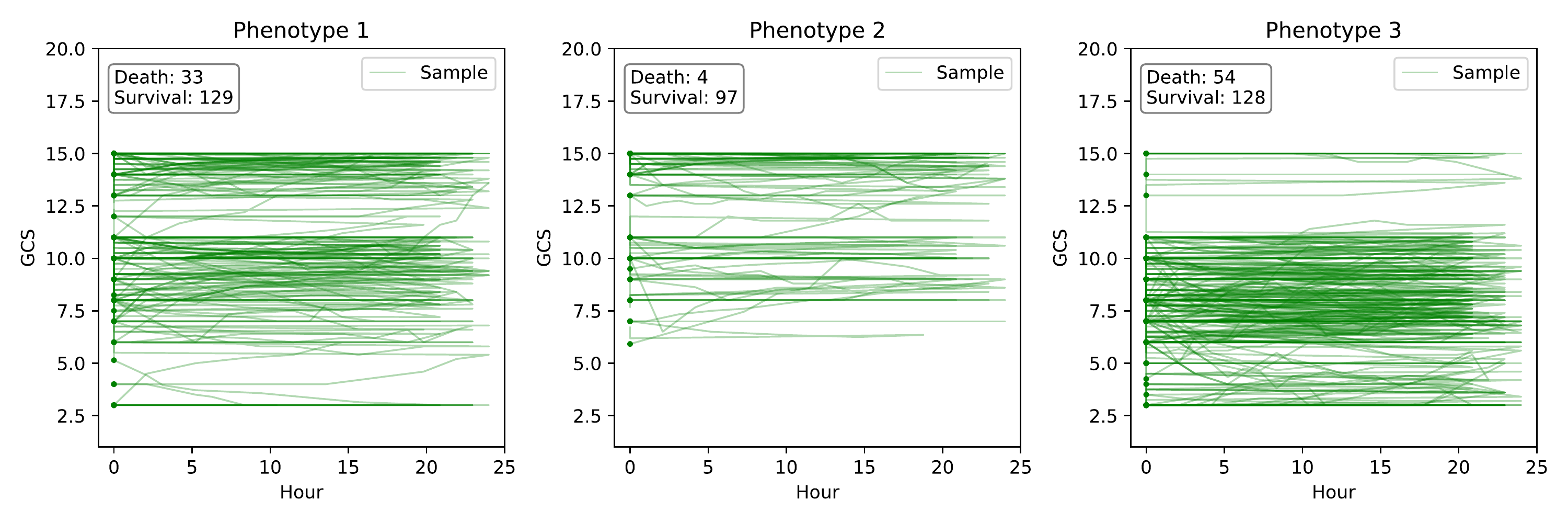}
    \caption{Patients with stable ($\mathrm{std}<1$) GCS trajectories.}
    \label{fig:Physionet-stable}
    \end{subfigure}
    \begin{subfigure}[b]{\textwidth}
	\centering
    \includegraphics[width=0.9\textwidth]{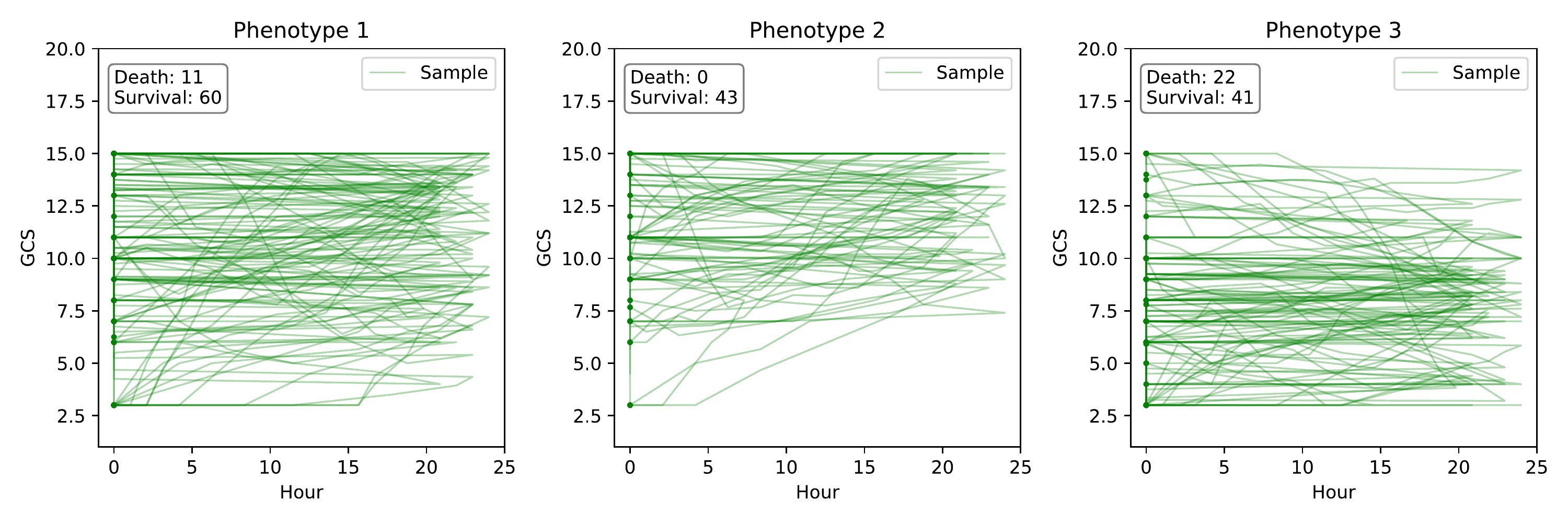}
    \caption{Patients with less stable ($\mathrm{std}\geq1$) GCS trajectories.}
    \label{fig:Physionet-nonstable}
    \end{subfigure}
    \caption{Three Phenotypes Discovered by {\ours} on ICU Dataset.
		\small
		The GCS trajectory of patients with different phenotypes are illustrated in the considered time period.
		All trajectories start at $t=0$ and are smoothed with a rolling window of size 5.} 
		\label{fig:Physionet-phenotypes}
	\end{center}
\end{figure}

\end{document}